%% file: main.tex
\documentclass[11pt]{article}
\usepackage[margin=1in]{geometry}
\usepackage[utf8]{inputenc} %
\usepackage[T1]{fontenc}    %
\usepackage{hyperref}       %
\usepackage{url}            %
\usepackage{booktabs}       %
\usepackage{amsfonts}       %
\usepackage{nicefrac}       %
\usepackage{microtype}      %

\usepackage{enumitem}
\usepackage{microtype}
\usepackage{booktabs} 
\usepackage{graphicx}
\usepackage[utf8]{inputenc}
\usepackage{amsmath,amsthm,amsfonts}
\usepackage{diagbox}
\usepackage[table]{xcolor}

\usepackage{algorithm}
\usepackage{algorithmic}
\usepackage{mathtools}
\usepackage{bbm}
\usepackage{subcaption}
\newcommand\numberthis{\addtocounter{equation}{1}\tag{\theequation}}

\newtheorem{lemma}{Lemma}
\newtheorem{corollary}{Corollary}
\newtheorem{theorem}{Theorem}
\newtheorem{conjecture}{Conjecture}
\newtheorem{proposition}{Proposition}

\newtheorem{assumption}{Assumption}
\newtheorem{definition}{Definition}

\newcommand{\mat}[1]{\mathbf{#1}}

\newcommand{\indep }{{\perp\!\!\!\perp}}

\title{Entropic Causal Inference: Identifiability and \\Finite Sample Results}

\author{%
  Spencer Compton \\
  MIT\\
  MIT-IBM Watson AI Lab\\
  \texttt{scompton@mit.edu} \\
  \and
  Murat Kocaoglu \\
  MIT-IBM Watson AI Lab\\
  IBM Research\\
  \texttt{murat@ibm.com} \\
  \and
  Kristjan Greenewald \\
  MIT-IBM Watson AI Lab\\
  IBM Research\\
  \texttt{kristjan.h.greenewald@ibm.com} \\
  \and
  Dmitriy Katz \\
  MIT-IBM Watson AI Lab\\
  IBM Research\\
  \texttt{dkatzrog@us.ibm.com} \\
}
\date{}
\begin{document}

\maketitle

\begin{abstract}
Entropic causal inference is a framework for inferring the causal direction between two categorical variables from observational data. The central assumption is that the amount of unobserved randomness in the system is not too large. This unobserved randomness is measured by the entropy of the exogenous variable in the underlying structural causal model, which governs the causal relation between the observed variables. \cite{Kocaoglu2017} conjectured that the causal direction is identifiable when the entropy of the exogenous variable is not too large. In this paper, we prove a variant of their conjecture. Namely, we show that for almost all causal models where the exogenous variable has entropy that does not scale with the number of states of the observed variables, the causal direction is identifiable from observational data. We also consider the minimum entropy coupling-based algorithmic approach presented by \cite{Kocaoglu2017}, and for the first time demonstrate algorithmic identifiability guarantees using a finite number of samples. We conduct extensive experiments to evaluate the robustness of the method to relaxing some of the assumptions in our theory and demonstrate that both the constant-entropy exogenous variable and the no latent confounder assumptions can be relaxed in practice. We also empirically characterize the number of observational samples needed for causal identification.
Finally, we apply the algorithm on T{\"u}bingen cause-effect pairs dataset.
\end{abstract}
\input{Introduction}
\input{RelatedWork}
\input{Identifiability}

\input{FiniteSamples}
\input{Experiments}

\input{Discussion}

\input{Conclusion}

\clearpage
\bibliographystyle{plain}
\bibliography{bibliography}
\clearpage
\input{Appendix}
\clearpage
\end{document}

%% file: Introduction.tex
\section{Introduction}
Understanding causal mechanisms is essential in many fields of science and engineering \cite{russo2010causality,susser2001glossary}. Distinguishing causes from effects allows us to obtain a causal model of the environment, which is critical for informed policy decisions \cite{Pearl2009}. Causal inference has been recently utilized in several machine learning applications, e.g., to explain the decisions of a classifier \cite{alvarez2017causal}, to design fair classifiers that mitigate dataset bias \cite{kilbertus2017avoiding,zhang2018fairness} and to construct classifiers that generalize \cite{subbaswamy2019preventing}. %

Consider a system that we observe through a set of random variables. For example, to monitor the state of a classroom, we might measure \emph{temperature, humidity} and \emph{atmospheric pressure} in the room. These measurements are random variables which come about due to the workings of the underlying system, the physical world. Changes in one are expected to cause changes in the other, e.g., decreasing the temperature might reduce the atmospheric pressure and increase humidity. As long as there are no feedback loops, we can represent the set of causal relations between these variables using a directed acyclic graph (DAG). This is called the \emph{causal graph} of the system. Pearl and others showed that knowing the causal graph enables us to answer many causal questions such as, \emph{``What will happen if I increase the temperature of the room?''}~\cite{Pearl2009}.

Therefore, for causal inference, knowing the underlying causal structure is crucial. Even though the causal structure can be learned from experimental data, in many tasks in machine learning, we only have access to a dataset and do not have the means to perform these experiments. In this case, observational data can be used for learning some causal relations. There are several algorithms in the literature for this task, which can be roughly divided into three classes: Constraint-based methods and score-based methods use conditional independence statements and likelihood function, respectively, to output (a member of) the equivalence class. An equivalence class of causal graphs are those that cannot be distinguished by the given data. The third class of algorithms impose additional assumptions about the underlying system or about the relations between the observed variables. Most of the literature focus on the special case of two observed variables $X,Y$ and to understand whether $X$ causes $Y$ or $Y$ causes $X$ under different assumptions. Constraint or score-based methods cannot answer this question simply because observed data is not sufficient without further %
assumptions. 
\begin{figure}
    \centering
    \begin{subfigure}[b]{0.35\textwidth}
    \includegraphics[width=\textwidth]{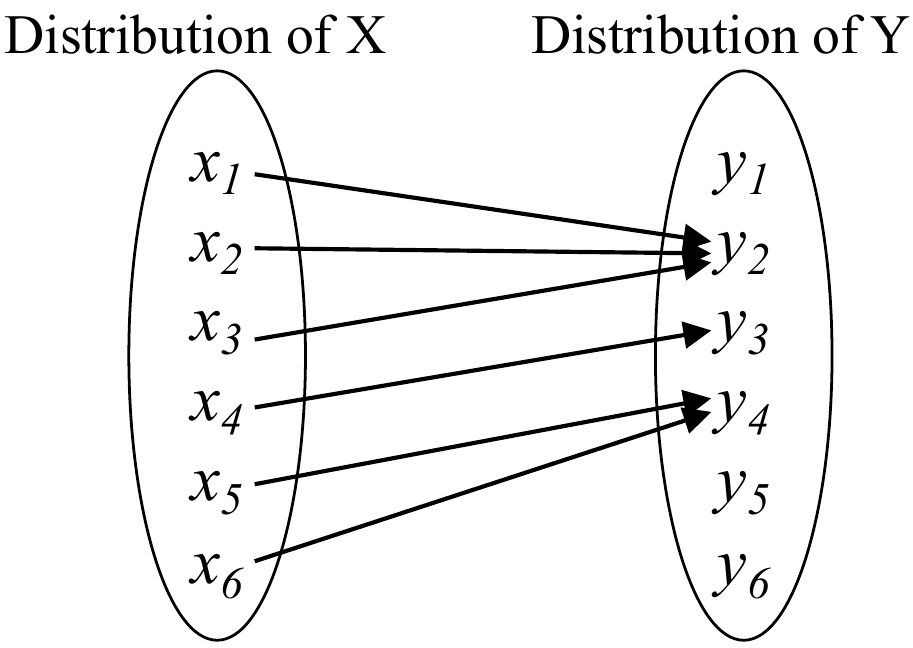}
    \caption{Deterministic relation.}
    \label{fig:intuition_a}
    \end{subfigure}
    \hspace{0.5in}
    \begin{subfigure}[b]{0.35\textwidth}
    \includegraphics[width=\textwidth]{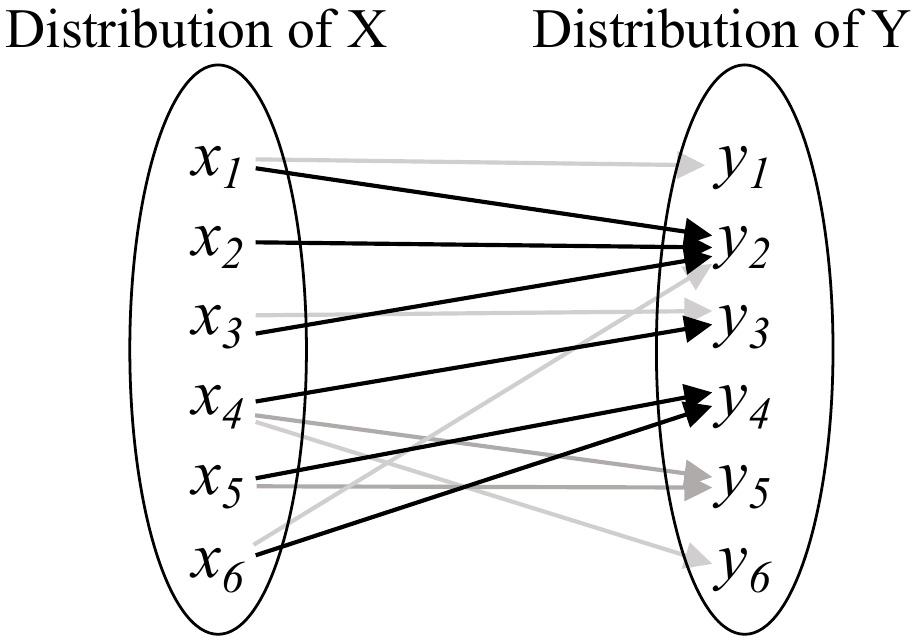}
    \caption{Relaxing determinism with noise.}
    \label{fig:intuition_b}
    \end{subfigure}
\caption{Intuition behind the entropic causality framework. \textbf{(a)} Most deterministic maps would be non-deterministic in the opposite direction, requiring non-zero additional randomness. \textbf{(b)} Entropic causality relaxes the deterministic map assumption to a map that needs low-entropy, and demonstrates that, most of the time, the reverse direction needs more entropy than the true direction.}
\label{fig:intuition}
\end{figure}
In this work, we focus on the special case of two categorical variables. Even though the literature is more established in the ordinal setting, few results exist when the observed variables are categorical. The main reason is that, for categorical data, numerical values of variables do not carry any meaning; whereas in continuous data one can use assumptions such as smoothness or additivity~\cite{Hoyer2008}. 

We first start with a strong assumption. Suppose that the system is \emph{deterministic}. This means that, even though observed variables contain randomness, the system has no additional randomness. When $X$ causes $Y$, this assumption implies that $Y=f(X)$ for some deterministic map $f(.)$. Consider the example in Figure \ref{fig:intuition}. Since there is no additional randomness, each value of $X$ is mapped to a single value of $Y$. What happens if we did not know the causal direction and tried to fit a function in the wrong direction as $X=g(Y)$. Unlike $f$, $g$ has to be one-to-many: $Y=2$ is mapped to three different value of $X$. Therefore, it is impossible to find a deterministic function in the wrong causal direction for this system. In fact, it is easy to show that most of the functions have this property: If $X,Y$ each has $n\geq 7$ states, all but $2^{-n}$ fraction of models can be identified. 

Although there might be systems where determinism holds such as in traditional computer software, this assumption in general is too strict. Then \emph{how much can we relax this assumption and still identify if $X$ causes $Y$ or $Y$ causes $X$?} In general, we can represent a system as $Y=f(X,E)$ where $E$ captures the additional randomness. To quantify this amount of relaxation, we use the entropy of the additional randomness in the structural equation, i.e., $H(E)$. For deterministic systems, $H(E)=0$. This question was posed as a conjecture in \cite{Kocaoglu2017}, within the entropic causal inference framework. 

We provide the first result in resolving this question. Specifically, we show that the causal direction is still identifiable for any $E$ with constant entropy. Our usage of ``constant'' is relative to the support size $n$ of the observed variables (note $0\leq H(X)\leq \log(n)$). This establishes a version of Kocaoglu's conjecture. 

A practical question is how much noise can the entropic causality framework handle: do we always need the additional randomness to not scale with $n$? Through experiments, we demonstrate that, in fact, we can relax this constraint much further. If $H(E)\approx \alpha\log(n)$, we show that in the wrong causal direction we need entropy of at least $\beta\log(n)$ for $\beta>\alpha$. This establishes that entropic causal inference is robust to the entropy of noise and for most models, reverse direction will require larger entropy. We finally demonstrate our claims on the benchmark T{\"u}bingen dataset. 

We also provide the first finite-sample analysis and provide bounds on the number of samples needed in practice. This requires showing finite sample bounds for the minimum entropy coupling problem, which might be of independent interest. The following is a summary of our contributions.
\begin{itemize}
    \item We prove the first identifiability result for the entropic causal inference framework using Shannon entropy and show that for most models, the causal direction between two variables is identifiable, if the amount of exogenous randomness does not scale with $n$, where $n$ is the number of states of the observed variables.
    \item We obtain the first bounds on the number of samples needed to employ the entropic causal inference framework. For this, we provide the first sample-bounds for accurately solving the minimum entropy coupling problem in practice, which might be of independent interest. 
    \item We show through synthetic experiments that our bounds are loose and entropic causal inference can be used even when the exogenous entropy scales with $\alpha\log(n)$ for $\alpha\!<\!1$.
    \item We employ the framework on T{\"u}bingen data to establish its performance. We also conduct experiments to demonstrate robustness of the method to latent confounders, robustness to asymmetric support size, i.e., when $X,Y$ have very different number of states, and finally establish the number of samples needed in practice. 
\end{itemize}
\textbf{Notation: } 
We will assume, without loss of generality, that if a variable has $n$ states, its domain is $[n]\coloneqq \{1,2,\hdots, n\}$. $p(x)$ is short for $p(X=x)$. $p(Y|x)$ is short for the distribution of $Y$ given $X=x$.  \emph{Simplex} is short for probability simplex, which, in $n$ dimensions is the polytope defined as $\Delta_n\coloneqq \{(x_i)_{i\in [n]}:\sum_ix_i=1,x_i\geq 0,\forall i\in [n]\}$. $\mathbbm{1}_{\{\varepsilon\}}$ is the indicator variable for event $\varepsilon$. \emph{SCM} is short for \emph{structural causal model} and refers to the functional relations between variables. For two variables where $X$ causes $Y$, the SCM is $Y=f(X,E), X\indep E$ for some variable $E$ and function $f$.

%% file: RelatedWork.tex
\section{Related Work}

There are a variety of assumptions and accompanying methods for inferring the causal relations between two observed variables \cite{Shimizu2006,Peters2014,Loh2014,Etesami2016,ghassami2017interaction}. For example, authors in \cite{Hoyer2008} developed a framework to infer causal relations between two continuous variables if the exogenous variables affect the observed variable additively. This is called the \emph{additive noise model (ANM)}. Under the assumption that the functional relation is non-linear they show identifiability results, i.e., for almost all models the causal direction between two observed variables can be identified. This is typically done by testing independence of the residual error terms from the regression variables.  %
Interestingly in \cite{kpotufe2014consistency} authors show that independence of regression residuals leads the total entropy in the true direction to be smaller than the wrong direction, which can be used for identifiability thereby arriving at the same idea we use in our paper.

A challenging setting for causal inference is the setting with discrete and categorical variables, where the variable labels do not carry any specific meaning. For example, \emph{Occupation} can be mapped to discrete values $\{0,1,2,\hdots\}$ as well as to one-hot encoded vectors. This renders methods which heavily rely on the variable values, such as ANMs, unusable. While extensions of ANMs to the discrete setting exist, they still utilize the variable values and are not robust to permuting the labels of the variables. One related approach proposed in \cite{janzing2010causal} is motivated by Occam's razor and proposes to use the Kolmogorov complexity to capture the complexity of the causal model, and assume that the true direction is "simple". As Kolmogorov complexity is not computable, the authors resort to a proxy, based on minimum description length. 

Another line of work uses the idea that causes are \emph{independent} from the causal mechanisms, which is called the \emph{independence of cause and mechanism} assumption. The notion of independence should be formalized since comparison is between a random variable and a functional relation.  In \cite{janzing2012information,janzing2015justifying}, authors propose using information geometry within this framework to infer the causal direction in deterministic systems. Specifically, they create a random variable using the functional relation based on uniform distribution and utilize the hypothesis that this variable should be independent from the cause distribution. 

%% file: Identifiability.tex
\section{Identifiability with Entropic Causality}
\label{sec:identifiability}
Consider the problem of identifying the causal graph between two observed categorical variables $X,Y$. We assume for simplicity that both have $n$ states, although this is not necessary for the results. Similar to the most of the literature, we make the causal sufficiency assumption, i.e., there are no latent confounders and also assume there is no selection bias. Then without loss of generality, if $X$ causes $Y$, there is a deterministic $f$ and an exogenous (unmeasured) variable $E$ that is independent from $X$ such that $Y=f(X,E)$, where $X\sim p(X)$ for some marginal distribution $p(X)$. Causal direction tells us that, if we intervene on $X$ and set $X=x$, we get $Y=f(x,E)$ whereas if we intervene on $Y$ and set $Y=y$, we still get $X\sim p(X)$ since $Y$ does not cause $X$.%

Algorithms that identify causal direction from data introduce an assumption on the model and show that this assumption does not hold in the wrong causal direction in general. Hence, checking for this assumption enables them to identify the correct causal direction. Entropic causality~\cite{Kocaoglu2017} also follows this recipe. They assume that the entropy of the exogenous variable is bounded in the true causal direction. We first present their relevant conjecture, then modify and prove as a theorem. 
\begin{conjecture}[\cite{Kocaoglu2017}]
\label{conj:entropic}
Consider the structural causal model $Y=f(X,E), X\in [n],Y\in[n],E\in[m]$ where $p(X),f,p(E)$ are sampled as follows: Let $p(X)$ be sampled uniformly randomly from the probability simplex in $n$ dimensions $\Delta_n$, and $p(E)$ be sampled uniformly randomly from the set of points in $\Delta_m$ that satisfy $H(E)\leq \log(n)+\mathcal{O}(1)$. Let $f$ be sampled uniformly randomly from all mappings $f:[n]\times[m]\rightarrow [n]$. Then with high probability, any $\tilde{E}\indep Y$ that satisfies $X = g(Y,\tilde{E})$ for some mapping $g:[n]\times [m]\rightarrow [n]$ entails $H(X)+H(E)< H(Y)+H(\tilde{E})$.
\end{conjecture}
In words, the conjecture claims the following: Suppose $X$ causes $Y$ with the SCM $Y=f(X,E)$. Suppose the exogenous variable $E$ has entropy that is within an additive constant of $\log(n)$. Then, for most of such causal models, any SCM that generates the same joint distribution in the wrong causal direction, i.e., $Y$ causes $X$, requires a larger amount of randomness than the true model. The implication would be that if one can compute the smallest entropy SCM in both directions, then one can choose the direction that requires smaller entropy as the true causal direction.

We modify their conjecture in two primary ways. First, we assume that the exogenous variable has constant entropy, i.e., $H(E) = \mathcal{O}(1)$. %
Unlike the conjecture, our result holds for any such $E$. %
Second, rather than the total entropy, we were able to prove identifiability by only comparing 
the entropies of the simplest exogenous variables in both directions $H(E)$ and $H(\tilde{E})$.\footnote{Entropy of the exogenous variable, or in the case of Conjecture \ref{conj:entropic} the entropy of the system, can be seen as a way to model complexity and the method can be seen as an application of Occam's razor. In certain situations, especially for ordinal variables, it might be suitable to also consider the complexity of the functions.} In Section \ref{sec:experiments}, we demonstrate that both criteria give similar performance in practice.

Our technical result requires the following assumption on $p(X)$, which,  for constant $\rho$ and $d$ guarantees that a meaningful subset of the support of $p(X)$ is \emph{sufficiently uniform}. We will later show that this condition holds with high probability, if $p(X)$ is sampled uniformly randomly from the simplex. 
\begin{assumption}[($\rho,d$)-uniformity]
\label{ass:1}
Let $X$ be a discrete variable with support $[n]$. Then there exists %
a subset $S$ %
of size $|S| \geq dn$, such that $p(X=x) \in [\frac{1}{\sqrt{\rho} n}, \frac{\sqrt{\rho}}{n}], \forall x\in S$. 
\end{assumption}
Our following theorem establishes that entropy in the wrong direction scales with $n$.
\begin{theorem}[Entropic Identifiability]\label{thm:main}
Consider the SCM $Y=f(X,E), X\indep E$, where $X\!\in\! [n],Y\!\in\![n], E\!\in\! [m]$. Suppose $E$ is any random variable with constant entropy, i.e.,  $H(E)\!=\!c\!=\!\mathcal{O}(1)$. Let $p(X)$ satisfy Assumption \ref{ass:1}$(\rho,d)$ for some constants $\rho\!\geq\! 1,d\!>\!0$. Let $f$ be sampled uniformly randomly from all mappings $f\!:\![n]\!\times\![m]\!\rightarrow\! [n]$. Then, with high probability, any $\tilde{E}$ that satisfies $X \!= \!g(Y,\tilde{E}),\tilde{E}\indep Y$ for some $g$, entails $H(\tilde{E})\!\geq\! (1 - o(1)) \log(\log(n))$. Specifically, for any $0\!<\!r\!<\!q$, $H(\tilde{E}) \!\ge\!\left(1\! -\! \frac{1+r}{1+q}\right) (0.5 \log(\log(n)) \!-\! \log(1+r)\! -\! \mathcal{O}(1)),\forall n \!\ge\! \nu(r,q,\rho,c,d)$ for some %
$\nu$. 
\end{theorem}
Theorem \ref{thm:main} shows that when $H(E)$ is a constant, under certain conditions on $p(X)$, with high probability, the entropy of any causal model in the reverse direction will be at least $\Omega(\log(\log(n)))$. %
Specifically, if a constant fraction of the support of $p(X)$ contains probabilities that are not too far from $\frac{1}{n}$, our result holds. Note that \emph{with high probability} statement is induced by the uniform measure on $f$, and it is relative to $n$. In other words, Theorem \ref{thm:main} states that the fraction of non-identifiable causal models goes to $0$ as the number of states of the observed variables goes to infinity. If a structure on the function is available in the form of a prior that is different from uniform, this can potentially be incorporated in the analysis although we expect calculations to become more tedious. 

Through the parameters $r,q$ we obtain a more explicit trade-off between the lower bound on entropy and how large $n$ should be for the result. $\nu(r,q,\rho,c,d)$ is proportional to $q$ and inversely proportional to $r$. The explicit form of $\nu$ is given in Proposition \ref{lem:HtildeBd} in the supplement.

We next describe some settings where these conditions hold: We consider the cases when $p(X)$ has bounded element ratio, $p(X)$ is uniformly randomly sampled from the simplex, or $H(X)$ is large. %
\begin{corollary}%
\label{cor:identifiability}
Consider the SCM in Theorem \ref{thm:main}. Let $H(E)\!=\!c\!=\!\mathcal{O}(1)$ and $f$ be sampled uniformly randomly. Let $p(x)$ be %
such that either 
$(a)$ $\frac{\max_xp(x)}{\min_xp(x)}\!\leq\! \rho$, or $(b)$ $p(x)$ is sampled uniformly randomly from the simplex $\Delta_n$, or $(c)$ $p(X)$ is such that $H(X)\!\geq\! \log(n)\!-\!a$ for some %
$a\!=\!\mathcal{O}(1)$. 

Then, with high probability, any $\tilde{E}$ that satisfies $X = g(Y,\tilde{E}),\tilde{E}\indep Y$ for some deterministic function $g$ entails $H(\tilde{E})\geq 0.25\log(\log(n))-\mathcal{O}(1)$. Thus, there exists $n_0$ (a function of $\rho,c$) such that for all $n\geq n_0$, the causal direction is identifiable with high probability.
\end{corollary}
The proof is given in Section \ref{app:identifiability}. Note that there is no restriction on the support size of the exogenous variable $E$. 

\textbf{Proof Sketch of Theorem \ref{thm:main}.} The full proof can be found in Appendix \ref{app:main}.
\begin{enumerate}
    \item Bound $H(\tilde{E})$ via $H(\tilde{E})\geq H(X|Y=y), \forall y\in [n]$.
    \item Characterize the sampling model of $f$ as %
    a balls-and-bins game, where each realization of $Y$ corresponds to a particular bin, each combination $(X\!=\!i,E\!=\!k)$ corresponds to a ball. %
    \item Identify a subset of ``good" bins $\mathcal U\subseteq [m]$. Roughly, a bin is ``good" if it does not contain a large mass from the balls other than the ones in $\{(i,1): i\in S\}$.
    \item Show one of the bins in $\mathcal U$, say $y=2$, has many balls from $\{(i,1):i\in S\}$. 
    \item Bound the contribution of the most-probable state of $E$ to the distribution $p(X|Y=2)$. 
    \item Characterize the effect of the other states of $E$ and identify a support for $X$ contained in $S$ on which the conditional entropy can be bounded. Use this to %
    lower bound for $H(X|Y=2)$.
\end{enumerate}

\textbf{Conditional Entropy Criterion:} 
From the proof of Proposition \ref{lem:HtildeBd} in Appendix \ref{app:main}, we have $H(\tilde{E})\! \geq\! \max_y H(X|Y\!=\!y) \geq (1 - o(1))\log (\log (n))$. Further, we have $\max_x H(Y|X\!=\!x) \leq H(E)\! \leq\! c \!=\! \mathcal{O}(1)$. Hence not only is $H(\tilde E) \!>\! H(E)$ for large enough $n$, but $\max_y H(X|Y\!=\!y) \!>\! \max_x H(Y|X\!=\!x)$ as well. Therefore, under the assumptions of Theorem \ref{thm:main}, $\max_y H(X|Y\!=\!y)$ and $\max_x H(Y|X\!=\!x)$ are sufficient to identify the causal direction: %
\begin{corollary}
Under the conditions of Theorem \ref{thm:main}, we have that $\max\limits_y H(X|Y\!\!=\!\!y) \!>\! \max\limits_x H(Y|X\!\!=\!\!x)$.
\end{corollary}

%% file: FiniteSamples.tex
\section{Entropic Causality with Finite Number of Samples} 
\label{sec:finite}
In the previous section, we provided identifiability results assuming that we have access to the joint probability distribution of the observed variables. In any practical problem, we can only access a set of samples from this joint distribution. If we assume we can get independent, identically distributed samples from $p(x,y)$, how many samples are sufficient for identifiability? %

Given samples from $N$ i.i.d. random variables $\{(X_i,Y_i)\}_{i\in [N]}$ where $(X_i,Y_i)\sim p(x,y)$, consider the plug-in estimators $\hat{p}(y) \coloneqq \frac{1}{\{N\}}\sum_{i=1}^N\mathbbm{1}_{\{Y_i=y\}}$ and $\hat{p}(x,y)\coloneqq \frac{1}{N}\sum_{i=1}^N\mathbbm{1}_{\{X_i=x\}}\mathbbm{1}_{\{Y_i=y\}}$ and define the estimator of the conditional $p(x|y)$ as $\hat{p}(x|y)\coloneqq \frac{\hat{p}(x,y)}{\hat{p}(y)}$. Define $\hat{p}(x)$ and $\hat{p}(y|x)$ similarly.
\begin{definition}
The minimum entropy coupling of $t$ random variables $U_1,U_2,\hdots,U_t$ is the joint distribution $p(u_1,\hdots,u_t)$ with minimum entropy that respects the marginal distributions of $U_i,\forall i$.
\end{definition}
The algorithmic approach of \cite{Kocaoglu2017} relies on minimum entropy couplings. Specifically, they show the following equivalence: Given $p(x,y)$, let $E$ be the minimum entropy exogenous variable such that $E\indep X$, and there exists an $f$ such that $Y=f(X,E),X\sim p(x)$ induces $p(x,y)$. Then the entropy of the minimum entropy coupling of the distributions $\{p(Y|x):x\in [n]\}$ is equal to $H(E)$.

Therefore, understanding how having a finite number of samples affects the minimum entropy couplings allows us to understand how it affects the minimum entropy exogenous variable in either direction. 
Suppose $\lvert\hat{p}(y|x)-p(y|x)\rvert \leq \delta, \forall x,y$ and $\lvert \hat{p}(x|y)-p(x|y) \rvert\leq \delta, \forall x,y$.  %
Given a coupling for distributions $p(Y|x)$, we construct a coupling for $\hat{p}(Y|x)$ whose entropy is not much larger. As far as we are aware, the minimum entropy coupling problem with sampling noise has not been studied.

Consider the minimum entropy coupling problem with $n$ marginals $\mat{p}_k=[p_k(i)]_{i\in[n]}, k\in [n]$. Let $p(i_1,i_2,\hdots,i_n)$ be a valid coupling, i.e., $\sum_{j\neq k}\sum_{i_j=1}^n p(i_1,i_2,\hdots,i_n)=p_k(i_k),\quad \forall k,i_k$. Consider the marginals with sampling noise shown as $\hat{\mat{p}}_k = [\hat{p}_k(i)]_{i\in[n]}, k\in [n]$. Suppose $\lvert\hat{p}_k(i)-p_k(i)\rvert\leq \delta, \forall i,k$. The following is shown in Section \ref{app:constant_entropy_away} of the supplement.
\begin{theorem}
\label{thm:constant_entropy_away}
Let $p$ be a valid coupling for distributions $\{\mat{p}_i\}_{i\in[n]},$ where $\mat{p}_i\in \Delta_n,\forall i\in [n]$. 
Suppose $\{\mat{q}_i\}_{i\in [n]}$ are distributions such that $\lvert \mat{q}_i(j)-\mat{p}_i(j)\rvert\leq \delta, \forall i,j\in [n]$. If $\delta \leq \frac{1}{n^2\log(n)}$, then there exists a valid coupling $q$ for the marginals $\{\mat{q}_i\}_{i\in [n]}$ such that $H(q)\leq H(p) + e^{-1}\log(e) + 2 + o(1)$.
\end{theorem}
Theorem \ref{thm:constant_entropy_away} shows that if the $l_\infty$ norm between the conditional distributions and their empirical estimators are bounded by $\delta\leq \frac{1}{n^2\log(n)}$, there exists a coupling that is within $3$ bits of the optimal coupling on true conditionals. To guarantee this with the plug-in estimators, we have the following: %
\begin{lemma}
\label{lem:concentration}
Let $X\in [n],Y\in[n]$ be two random variables with joint distribution $p(x,y)$. Let $\alpha = \min\{\min_ip(X=i),\min_jp(Y=j)\}$. Given $N$ samples $\{(X_i,Y_i)\}_{i\in[N]}$ from independent identically distributed random variables $(X_i,Y_i)\sim p(x,y)$, let $\hat{p}(X|Y=y)$, $\hat{p}(Y|X=x)$ be the plug-in estimators of the conditional distributions. If $N=\Omega(n^4\alpha^{-2}\log^3(n))$, then $\lvert\hat{p}(y|x)-p(y|x)\rvert\leq \frac{1}{n^2\log(n)}$ and $\lvert\hat{p}(x|y)-p(x|y)\rvert\leq \frac{1}{n^2\log(n)}, \forall x,y$ with high probability.
\end{lemma}
Next, we have our main identifiability result using finite number of samples:
\begin{theorem}[Finite sample identifiability]
\label{thm:finite_samples}
Let $\mathcal{A}$ be an algorithm that outputs the entropy of the minimum entropy coupling. Consider the SCM in Theorem \ref{thm:main}. Suppose $E$ is any random variable with constant entropy, i.e.,  $H(E)\!=\!c\!=\!\mathcal{O}(1)$. Let $p(X)$ satisfy Assumption \ref{ass:1}$(\rho,d)$ for some constants $\rho\!\geq\! 1,d\!>\!0$. Let $f$ be sampled uniformly randomly from all mappings $f\!:\![n]\!\times\![m]\!\rightarrow\! [n]$. 
Let $\alpha = \min\{\min_ip(X=i),\min_jp(Y=j)\}$. Given $N = \Omega( n^{4}\alpha^{-2}\log^3(n))$ samples, let $\hat{p}(X|y),\hat{p}(Y|x)$ be the plug-in estimators for the conditional distributions. Then, for sufficiently large $n$, $\mathcal{A}(\{\hat{p}(X|y)\}_y)>\mathcal{A}(\{\hat{p}(Y|x)\}_x)$ with high probability.
\end{theorem}
From the equivalence between minimum entropy couplings and minimum exogenous entropy, Theorem \ref{thm:finite_samples} shows identifiability of the causal direction using minimum-entropy exogenous variables. Similar to Corollary \ref{cor:identifiability}, the result holds when $p(X)$ is chosen uniformly randomly from the simplex:
\begin{corollary}
\label{cor:finite_simplex}
Consider the SCM in Theorem \ref{thm:main}, where $H(E)\!\!=\!\! c\!\! =\!\! \mathcal{O}(1)$, $f$ is sampled uniformly randomly. Let $p(X)$ be sampled uniformly randomly from the simplex $\Delta_n$. Given $N \!=\! \Omega( n^{8}\log^5(n))$ samples, let $\hat{p}(X|Y=y)$, $\hat{p}(Y|X=x)$ be the plug-in estimators for the conditional distributions. Then, for large enough $n$, $\mathcal{A}(\{\hat{p}(X|Y=y)\}_y)>\mathcal{A}(\{\hat{p}(Y|X=x)\}_x)$ with high probability.
\end{corollary} 

\textbf{Conditional Entropy Criterion with Finite Samples: }
Note that the sample complexity in Theorem \ref{thm:finite_samples} scales with $\alpha^{-2}$ where $\alpha\coloneqq \min\{\min_i p(X\!=\!i), \min_j p(Y\!=\!j)\}$. If either of the marginal distributions are not strictly positive, this can make the bound of Theorem \ref{thm:finite_samples} vacuous. To address this, we use an internal result from the  %
proof of Theorem \ref{thm:main}. 
In the proof %
we show that for some $i$, $p(Y\!=\!i) = \Omega(\frac{1}{n})$ and $H(X|Y\!=\!i) = \Omega(\log(\log(n)))$. Then, it is sufficient to obtain enough samples to accurately estimate $p(X|Y\!=\!i)$. Even though $i$ is not known a priori, since $p(Y=i)=\Omega(\frac{1}{n})$, estimating conditional entropies $H(X|Y\!=\!j)$ where the number of samples $\lvert\{(x,Y=j)\}_x\rvert$ exceeds a certain threshold guarantees that $p(X|Y=i)$ is estimated accurately. We have the following result:
\begin{theorem}
[Finite sample identifiability via conditional entropy]\label{thm:finite_samples_conditionals}
Consider the SCM in Theorem \ref{thm:main}, where $H(E)\!=\! c\! =\! \mathcal{O}(1)$, $f$ is sampled uniformly randomly. Let $p(X)$ satisfy Assumption \ref{ass:1}$(\rho,d)$ for some constants $\rho\!\geq\! 1,d\!>\!0$. Given $N = \Omega( n^{2} \log (n))$ samples, let $N_x$ be the number of samples where $X\!=\!x$ and similarly for $N_y$. Let $\hat{H}$ denote the entropy estimator of \cite{valiant2017estimating}. Then, for $n$ large enough, $\max_{\{y:N_y \geq n\}} \hat{H}(X|Y\!=\!y) > \max_{\{x: N_x \geq n\}} \hat{H}(Y|X\!=\!x)$ with high probability.
\end{theorem}
Theorem \ref{thm:finite_samples_conditionals} shows that $\mathcal{O}(n^2\log(n))$ samples are sufficient to estimate the large conditional entropies of the form $H(Y|x),H(X|y)$, which is sufficient for identifiability even for sparse $p(x,y)$.

%% file: Experiments.tex
\begin{figure*}[t!]
	\centering
	\begin{subfigure}[b]{0.3\linewidth}
	\includegraphics[width=\textwidth]{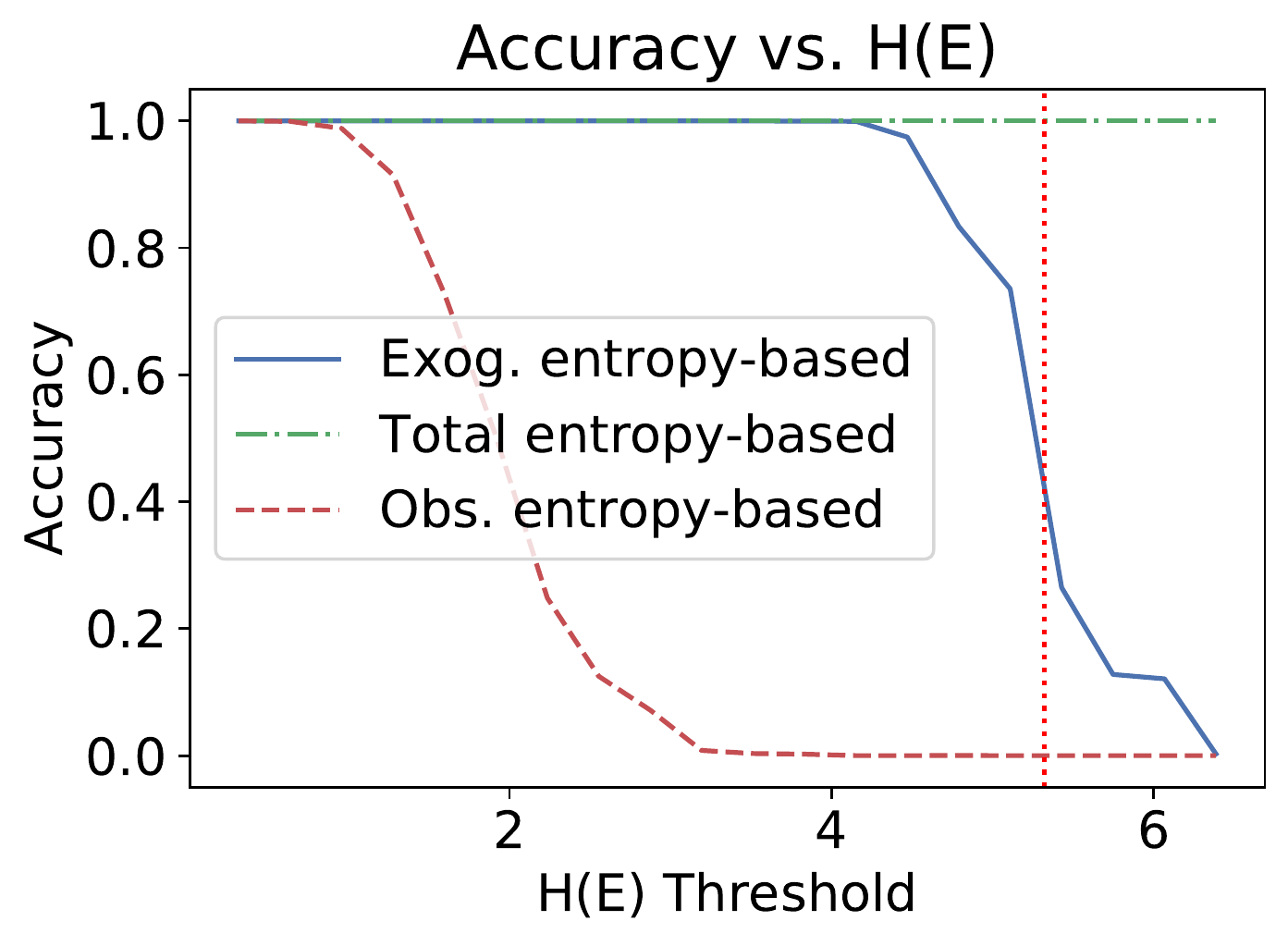}
	\caption{}
	\label{fig:impl_40_40}
	\end{subfigure}
	\begin{subfigure}[b]{0.3\linewidth}
    \includegraphics[width=\textwidth]{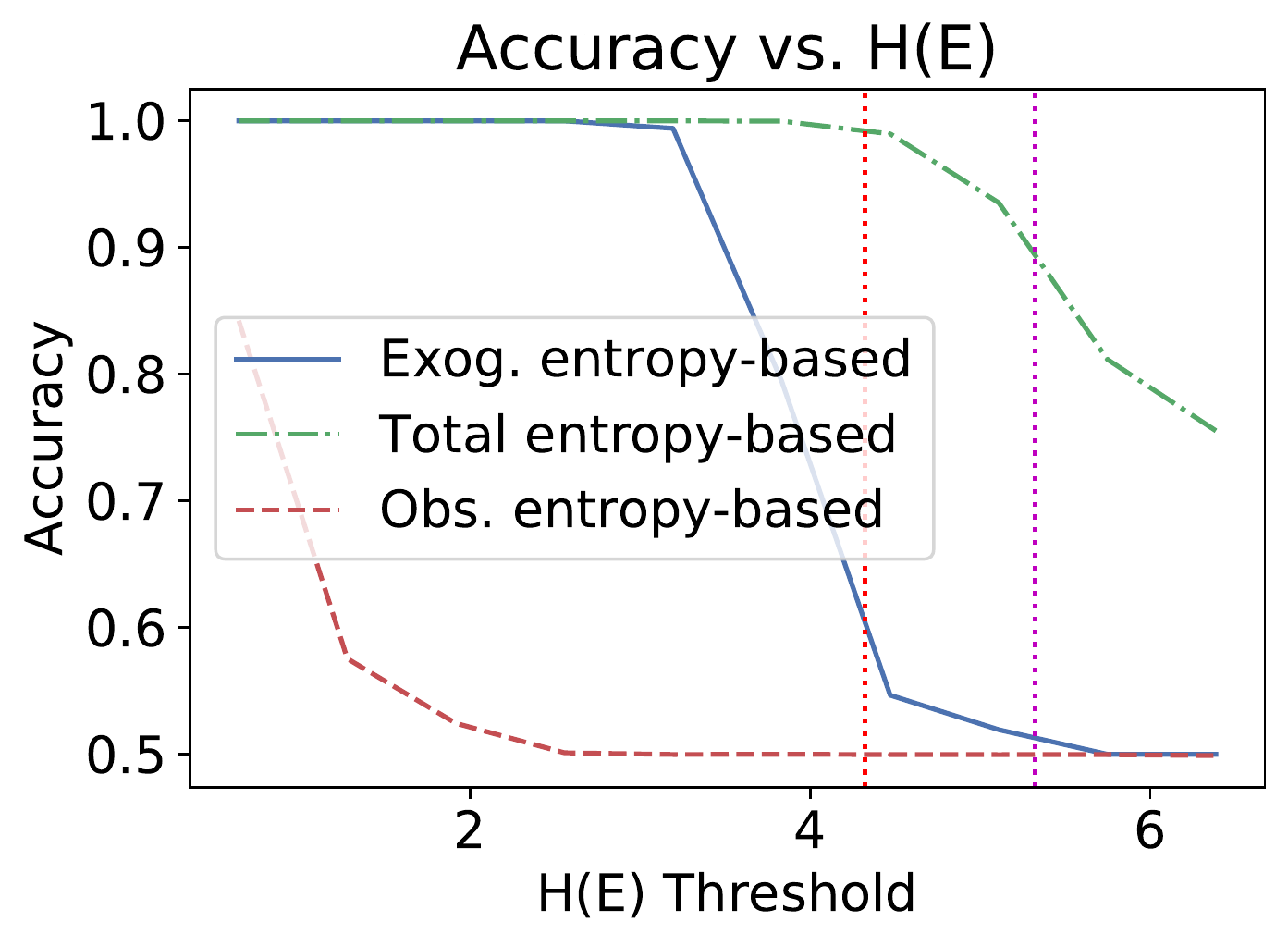}
	\caption{}
	\label{fig:impl_20_40}
	\end{subfigure}
	\begin{subfigure}[b]{0.3\linewidth}
	\includegraphics[width=\textwidth]{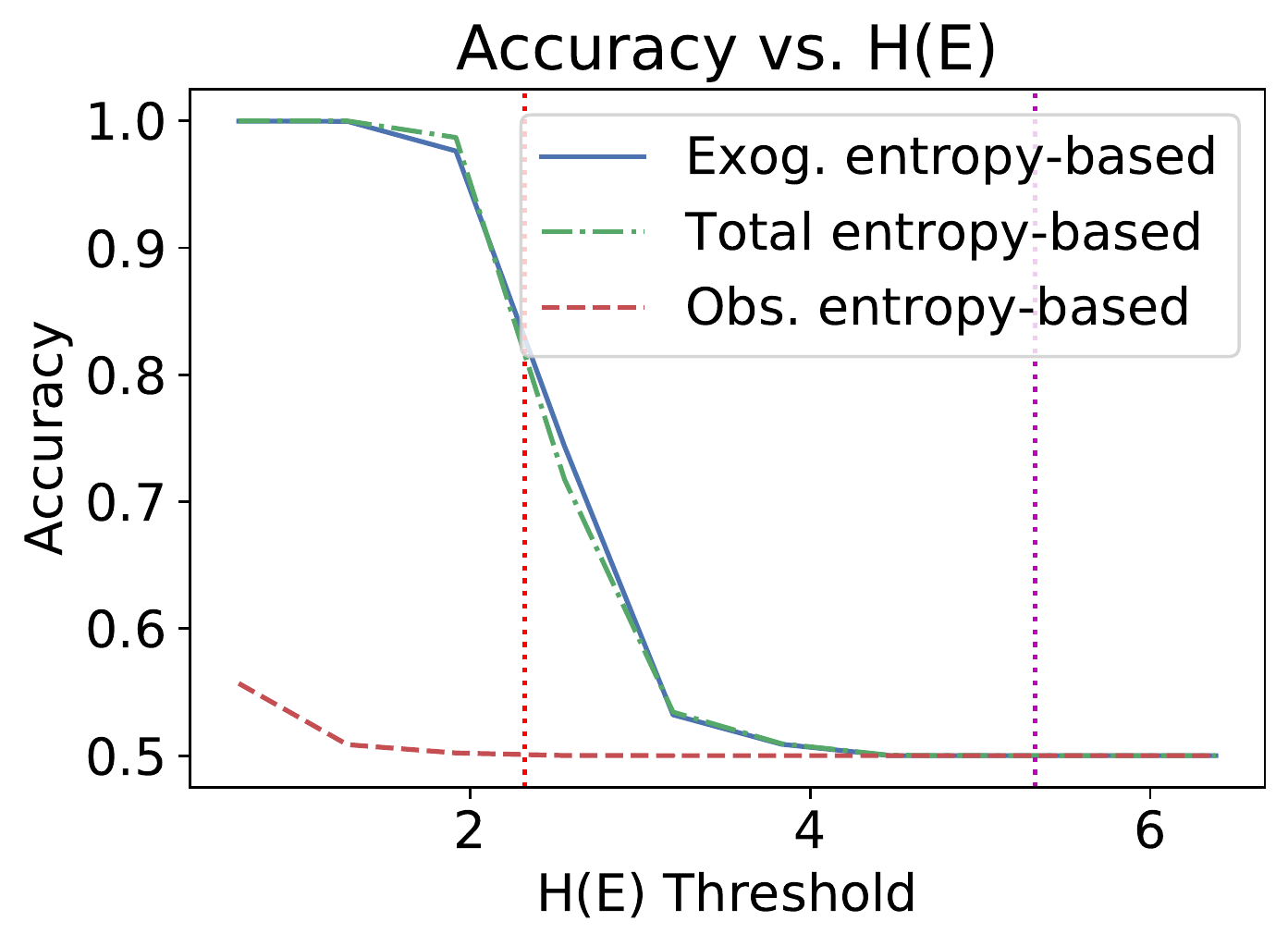}
	\caption{}
	\label{fig:impl_5_40}
	\end{subfigure}
	\caption{$m:$ number of states of $X$, $n:$ number of states of $Y$ in causal graph $X\rightarrow Y$. \textbf{(a)} $n=40, m=40$. Accuracy on simulated data: \emph{Obs. entropy-based} declares $X\rightarrow Y$ if $H(X)>H(Y)$ and $Y\rightarrow X$ otherwise; \emph{Exog. entropy-based} compares the exogenous entropies in both direction and declares $X\rightarrow Y$ if the exogenous entropy for this direction is smaller, and $Y\rightarrow X$ otherwise; \emph{Total entropy-based} compares the total entropy of the model in both directions and declares the direction with smaller entropy as the true direction as proposed in \cite{Kocaoglu2017}. \textbf{(b)} uses uniform mixture data from when $m=40,n=20$ and $m=20,n=40$. Similarly for \textbf{(c)} for $m=40,n=5$ and $m=5,n=40$. Magenta and red dashed vertical lines show $\log_2(\min\{m,n\})$ and  $\log_2(\max\{m,n\}),$ respectively.}
	\label{fig:implications_on_observed}	
\end{figure*}
\section{Experiments}
\label{sec:experiments}
In this section, we conduct several experiments to evaluate the robustness of the framework. Complete details of each experiment are provided in the supplementary material. Unless otherwise stated, the greedy minimum entropy coupling algorithm of \cite{Kocaoglu2017} is used to approximate $H(E)$ and $H(\tilde{E})$.

\textbf{Implications of Low-Exogenous Entropy Assumption.}
We investigate the implications of this assumption. Specifically, one might ask if having low exogenous entropy implies $H(X)\!>\!H(Y)$. This would be unreasonable, since there is no reason for cause to always have the higher entropy.

In Figure \ref{fig:implications_on_observed},  we evaluate the accuracy of the algorithm on synthetic data for different exogenous entropies $H(E)$. To understand the impact of the assumption on $H(X),H(Y)$, in addition to comparing exogenous entropies (\emph{Exog. entropy-based}) and total entropies (\emph{total entropy-based}) \cite{Kocaoglu2017}, we also show the performance of a simple baseline that compares $H(X)$ and $H(Y)$ (\emph{obs. entropy-based}) and declares $X\rightarrow Y$ if $H(X)>H(Y)$ and vice versa. 

We identify three different regimes, e.g., see Figure \ref{fig:impl_40_40}: Regime $1$: If $H(E)<0.2\log(n)$, we get $H(X)>H(Y)$ most of the time. All methods perform very well in this regime which we can call \emph{almost deterministic}. Regime $2$: If $0.2\log(n)<H(E)<0.6\log(n)$, accuracy of \emph{obs. entropy-based} method goes to $0$ since, on average, we transition from the regime where $H(X)>H(Y)$ to $H(X)<H(Y)$. Regime $3$: $0.6\log(n)<H(E)<0.8\log(n)$ where $H(X)<H(Y)$ most of the time. As can be seen, \emph{total entropy-based} and \emph{exog. entropy-based} methods both show (almost) perfect accuracy in Regime $1,2,3$ whereas \emph{obs. entropy-based} performs well only in Regime $1$. 

We also evaluated the effect of the observed variables having different number of states on mixture data in Figure \ref{fig:impl_20_40}, \ref{fig:impl_5_40}. In this case, framework performs well up until about $0.8\log(\min\{m,n\})$.

\textbf{Relaxing Constant Exogenous-Entropy Assumption.}
In Section \ref{sec:identifiability}, we demonstrated that the entropic causality framework can be used when the exogenous randomness is a constant, relative to the number of states $n$ of the observed variables. For very high dimensional variables, this might be a strict assumption. In this section, we conduct synthetic experiments to evaluate if entropic causality can be used when $H(E)$ scales with $n$. In particular, we test for various $\alpha\!<\!1$ the following: \emph{Is it true that the exogenous entropy in the wrong direction will always be larger, if the true exogenous entropy is $\leq\! \alpha log(n)?$} For $\alpha\!=\!\{0.2,0.5,0.8\}$, we sampled $10$k $p(E)$ from Dirichlet distribution such that $H(E)\!\approx\! \alpha \log(n)$ and calculated exogenous entropy in the wrong direction $H(\tilde{E})$. Figure \ref{fig:relaxing0.8} shows the histograms of $H(\tilde{E})$ for $\alpha\!=\! 0.8$ and $n\!=\!\{16,64,128\}$. We observe that $H(\tilde{E})$ tightly concentrates around $\beta\log(n)$ for some $\beta\!>\!\alpha$. For reference, $\alpha\log(n)$ is shown by the vertical yellow line. Similar results are observed for other $\alpha$ values which are provided in the supplementary material.

\begin{figure}
	\centering
	\begin{subfigure}[b]{0.28\linewidth}
		\includegraphics[width=\textwidth]{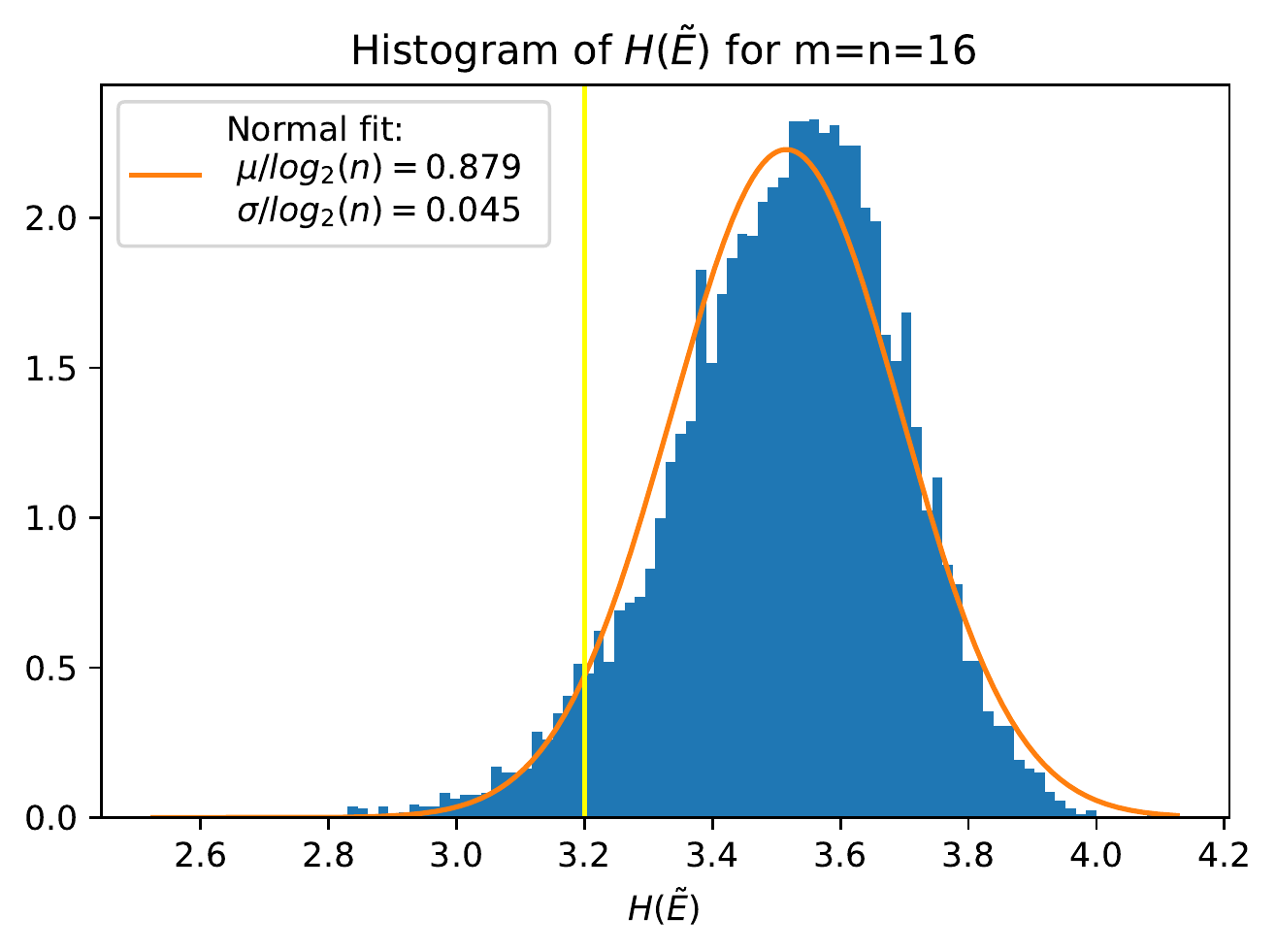}
		\label{fig:normal_fit_16}
	\end{subfigure}
	\begin{subfigure}[b]{0.26\linewidth}
		\includegraphics[width=\textwidth]{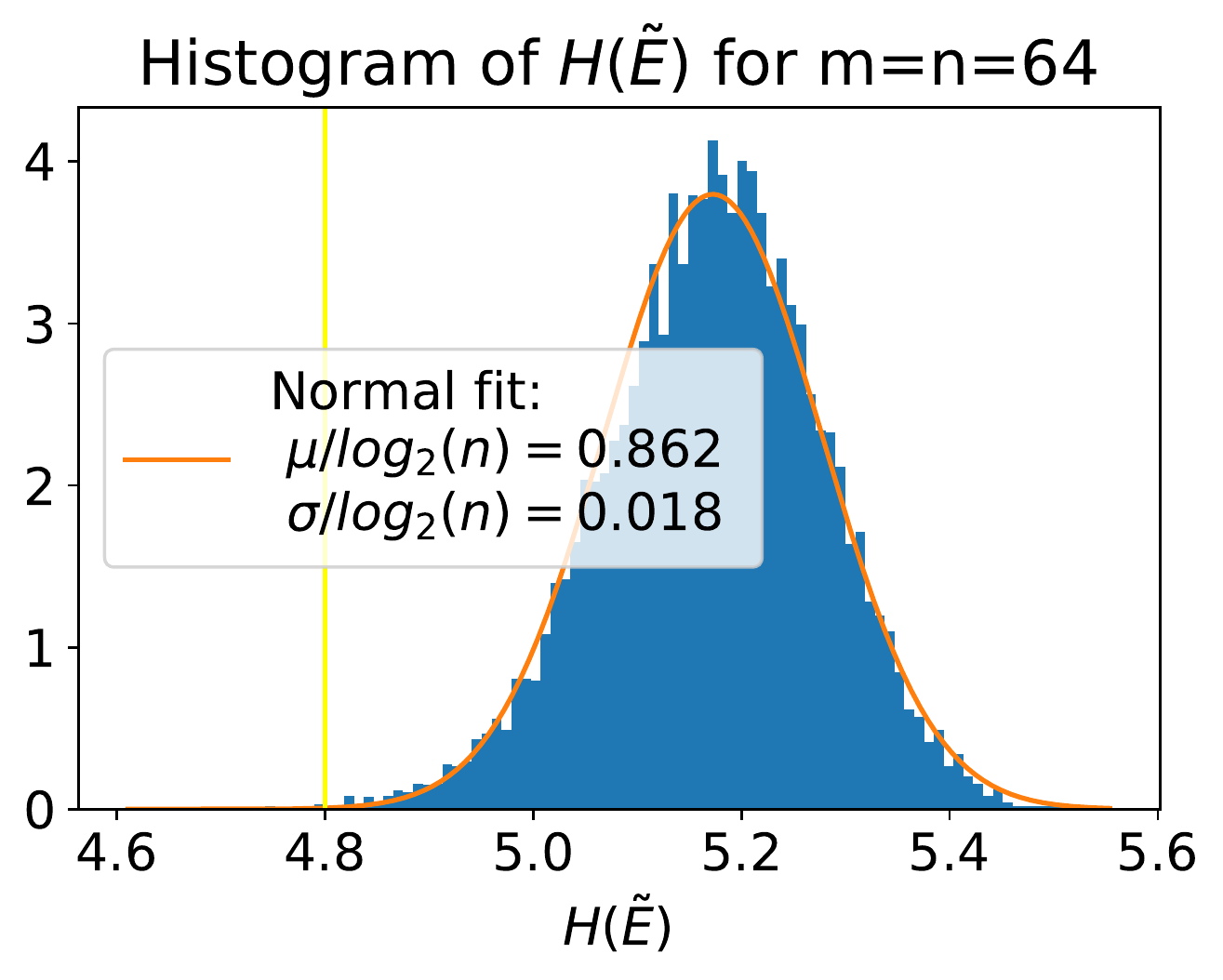}
		\label{fig:normal_fit_64}
	\end{subfigure}
	\begin{subfigure}[b]{0.28\linewidth}
		\includegraphics[width=\textwidth]{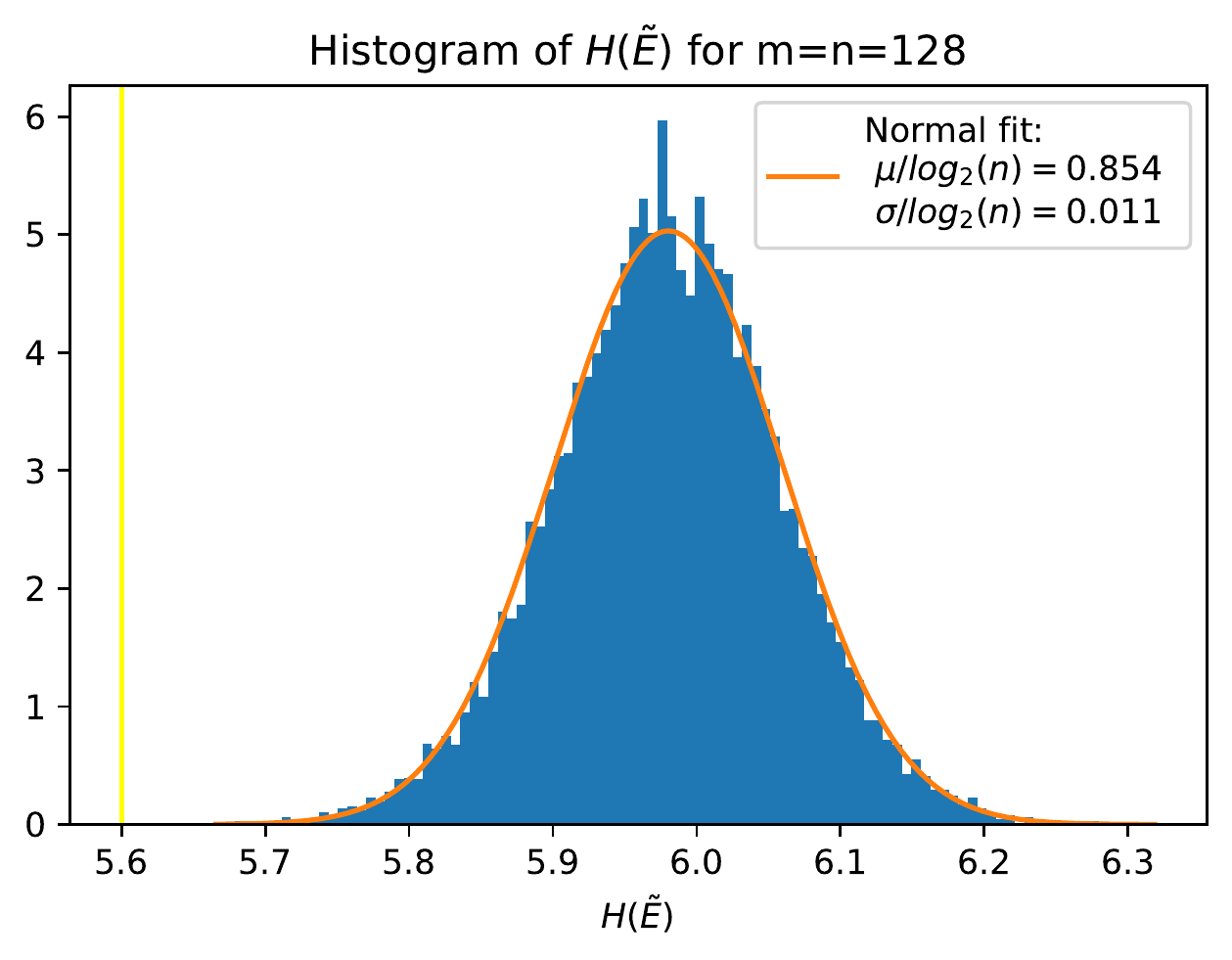}
		\label{fig:normal_fit_128}
	\end{subfigure}	
\caption{Histogram of $H(\tilde{E})$ when $H(E)\approx 0.8\log_2(n)$.  Yellow line shows $x=0.8\log_2(n)$}
\label{fig:relaxing0.8}
\end{figure}
\begin{figure*}[t!]
	\centering
	\begin{subfigure}[b]{0.3\linewidth}
		\includegraphics[width=\textwidth]{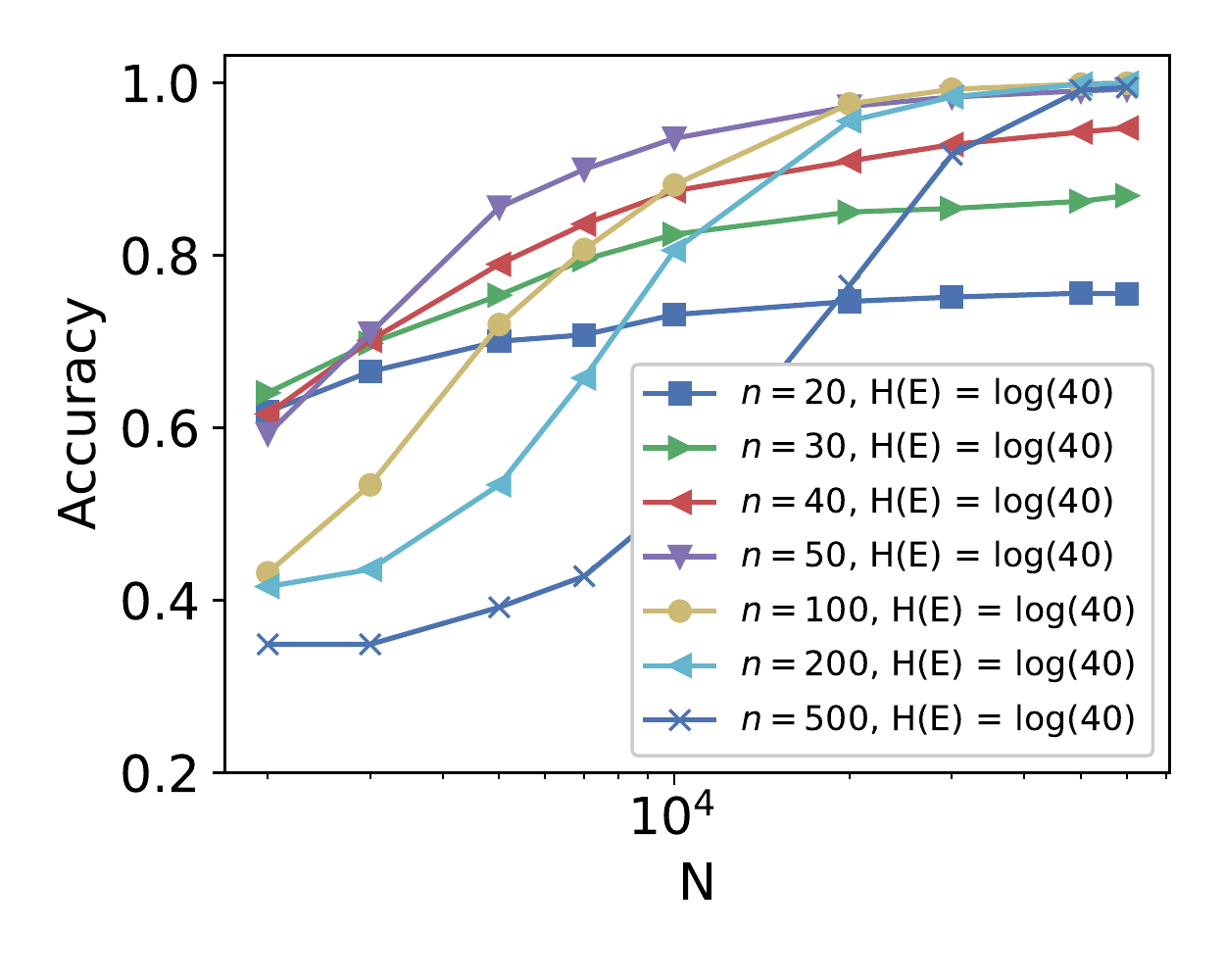}
		\caption{Identification via conditional entropies ($H(E) \approx \log(40)$).}
		\label{fig:conditional}
	\end{subfigure}
	\begin{subfigure}[b]{0.3\linewidth}
		\includegraphics[width=\textwidth]{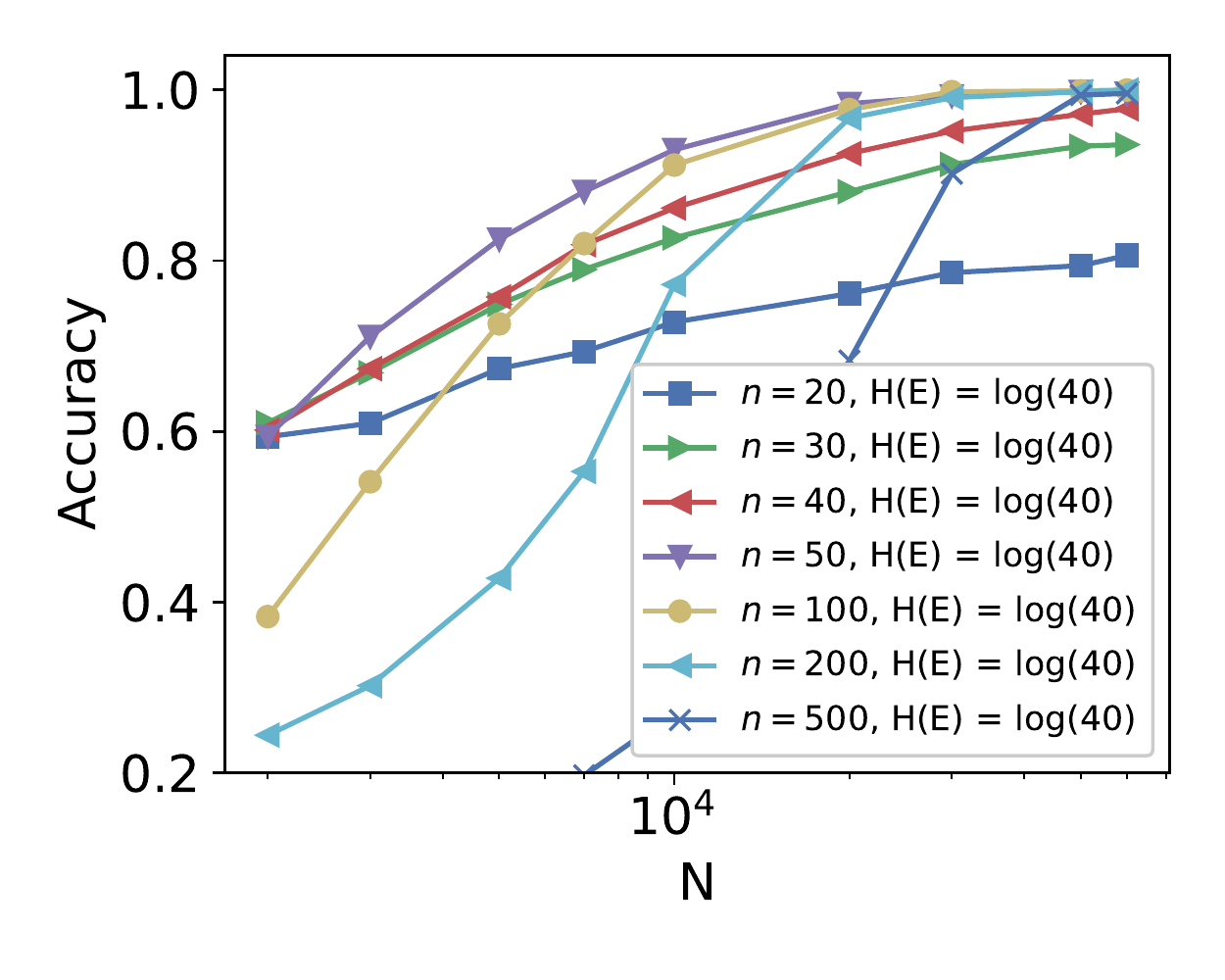}
		\caption{Identification via MEC algorithm ($H(E) \approx \log(40)$).}
		\label{fig:conditionalMEC}
	\end{subfigure}
	\begin{subfigure}[b]{0.3\linewidth}
		\includegraphics[width=\textwidth]{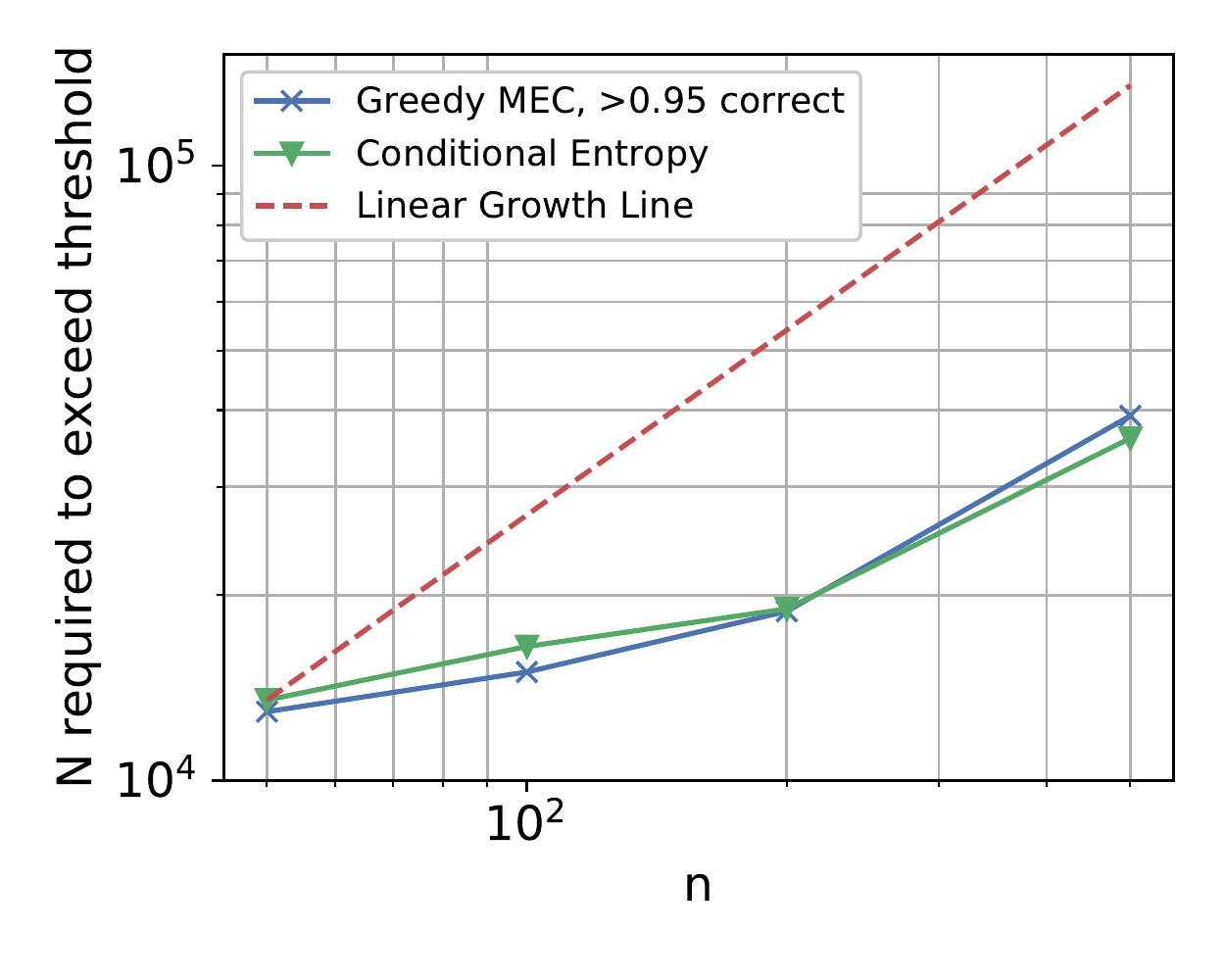}
		\caption{Number of samples vs. support size of observed variables.}
		\label{fig:NreqdLN40}
	\end{subfigure}
	\caption{%
	(a) Probability of correctly discovering the causal direction $X \rightarrow Y$ as a function of $n$ and number of samples $N$, using the conditional entropies as the test. (b) Probability of correctly discovering the causal direction $X \rightarrow Y$ using the greedy MEC algorithm. (c) Samples $N$ required to reach 95\% correct detection as a function of $n$, derived from the plots in Figure \ref{fig:conditional} and Figure \ref{fig:conditionalMEC}.  }
\end{figure*}
\textbf{Effect of Finite Number of Samples.}
In Section \ref{sec:finite}, we identified finite sample bounds for entropic causality framework, both using the exogenous entropies $H(E),H(\tilde{E})$ and using conditional entropies of the form $\max_yH(X|Y\!=\!y),\max_xH(Y|X\!=\!x)$. We now test if the bounds are tight.

We observe two phases and a transition phenomenon in between. The first phase occurs for small values of $n$, for $n\in\{20,30,40\}$. Here, the fraction of identifiable causal models does not reach $1$ as the number of samples is increased, but saturates at a smaller value. This is expected since exogenous noise is relatively high, i.e., $H(E)\geq \log(n)$. For $n>40$, or equivalently, when $H(E)\leq \log(n)$, increasing number of samples increases accuracy to $1$, as expected. 

The greedy MEC criterion has slightly better performance (by $\approx 5\%$), indicating more robustness. This may be due to a gap between $H(\tilde{E})$ and $H(X|Y\!=\!y)$ since greedy-MEC output is not limited by $\log(n)$ unlike conditional entropy. %
In contrast to the $\tilde{O}(n^8)$ bound, the number of samples needed has a much better dependence on $n$. Figure \ref{fig:NreqdLN40} includes a dashed linear growth line for comparison.

\textbf{Effect of Confounding}
The equivalence between finding the minimum entropy exogenous variable and finding the minimum entropy coupling relies on the assumption that there are no unobserved confounders in the system. Despite lack of theory, it is useful to experimentally understand if the method is robust to \emph{light confounding}. One way to assess the effect of confounding is through its entropy: If a latent confounder $L$ is a constant, i.e., it has zero entropy, it does not affect the observed variables. In this section, we simulate a system with light confounding by limiting the entropy of the latent confounder and observing how quickly this degrades the performance of the entropic causality approach. 

The results are given in Figure \ref{fig:effect_of_confounding}. The setting is similar to that of Figure \ref{fig:implications_on_observed}. We set $H(E)\approx 2$ and show accuracy of the method as entropy of the latent $L$ is increased. Perhaps surprisingly, the effect of increasing the entropy of the confounder is very similar to the effect of increasing the entropy of the exogenous variable. This shows that the method is robust to light latent confounding. 
\begin{figure}[t!]
	\centering
	\begin{subfigure}[b]{0.28\linewidth}
	\includegraphics[width=\textwidth]{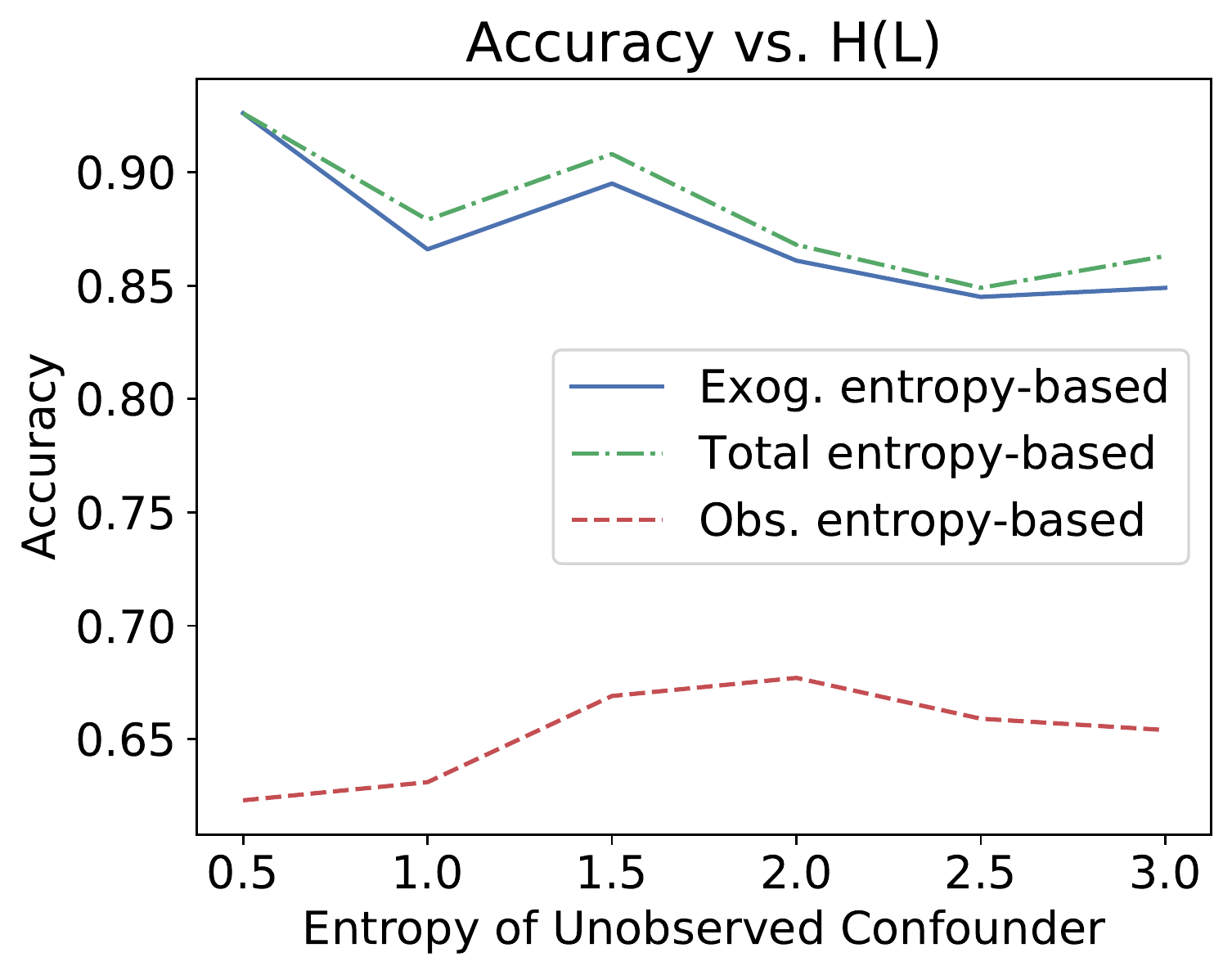}
	\caption{}
	\label{fig:conf_40_40}
	\end{subfigure}
	\begin{subfigure}[b]{0.28\linewidth}
	\includegraphics[width=\textwidth]{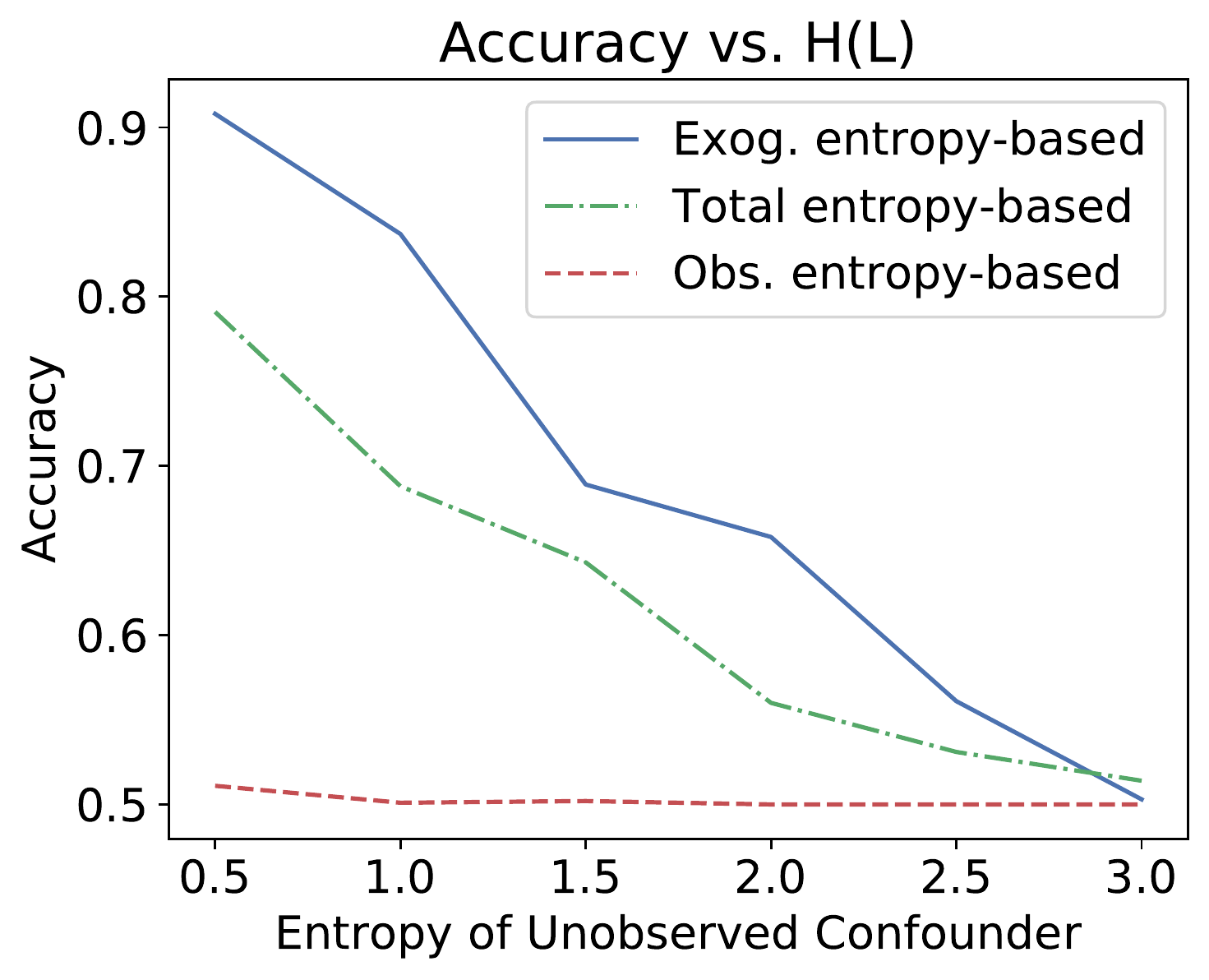}
	\caption{}
	\label{fig:conf_20_40}
	\end{subfigure}
	\begin{subfigure}[b]{0.28\linewidth}
	\includegraphics[width=\textwidth]{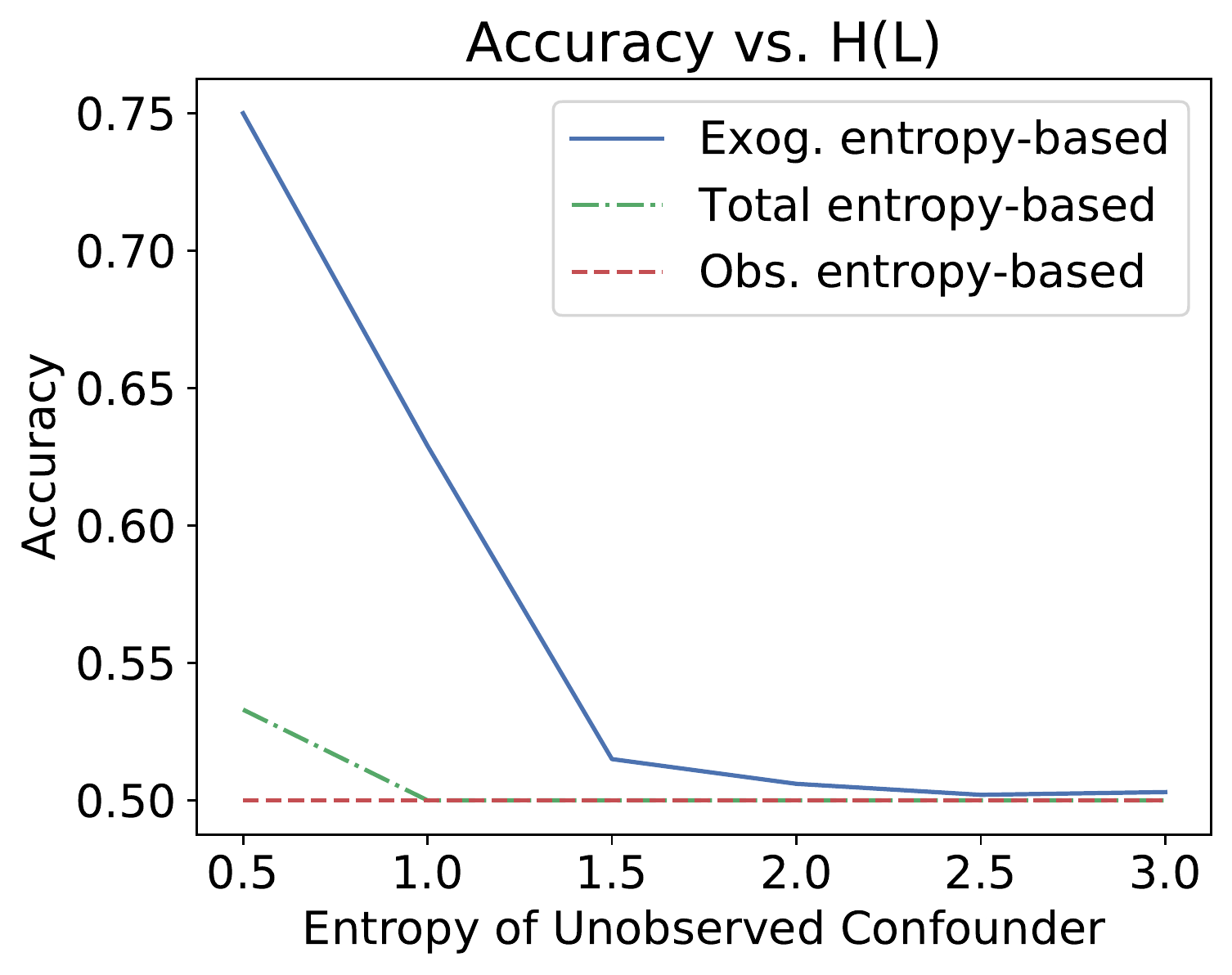}	
	\caption{}
	\label{fig:conf_5_40}
	\end{subfigure}
	\caption{Accuracy on simulated data with \emph{light} confounding. Number of states and data are identical to those in Figure \ref{fig:implications_on_observed}. We use exogenous entropy of $2$ bits and add a confounder $L$. This can be interpreted as replacing some bits of the exogenous variable in Figure \ref{fig:implications_on_observed}  with those of a latent confounder. %
	Surprisingly, performance for $H(E)\!=\!2,H(L)\!=\!t$ is similar to the performance when $H(E)\!=\!2+t$ in Figure \ref{fig:implications_on_observed}. This indicates that the proposed method is robust to latent confounders, as long as the total exogenous and confounder entropy %
	is not very close to 
	$\min\{\log(n),\log(m)\}$.} 
	\label{fig:effect_of_confounding}	
\end{figure}
\input{tuebingen}

\textbf{T{\"u}bingen Cause-Effect Pairs}
In \cite{Kocaoglu2017}, authors employed the total entropy-based algorithm on T{\"u}bingen data~\cite{mooij2016distinguishing} and showed that it performs similar to additive noise models with an accuracy of $~64\%$. %
Next, %
we test if entropic causality can be used when we only compare exogenous entropies. 

The challenge of applying entropic causality on T{\"u}bingen data is that most of the variables are continuous. Therefore, before applying the framework, one needs to quantize the data. The authors chose a uniform quantization, requiring both variables have the same number of states. We follow a similar approach. For $b\in\{5,10,20\}$, the value of $n$ is chosen for both $X,Y$ as the minimum of $b$, $N/10$, $N_x^{uniq}$ and $N_y^{uniq}$, where $N$ is the number of samples available for pair $X,Y$ and $N_x^{uniq},N_y^{uniq}$ are the number of unique realizations of $X,Y$, respectively. 

As a practical check for the validity of our key assumption, we make a decision based on the following: For a threshold $t$, algorithm makes a decision only for pairs for which either $H(E)\!\leq\! t\log(n)$ or $H(\tilde{E})\!\leq\! t\log(n)$. We report the accuracies in Table \ref{tab:tuebingen}. As we expect, for stricter thresholds, accuracy is improved, supporting the assumption that in real data, the direction with the smaller exogenous entropy is likely to be the true direction. Performance is most consistent with $b=10$. 

To check the stability of performance in regards to quantization, we conducted an experiment where we perturb the quantization intervals and take majority of $5$ independent decisions. This achieves qualitatively similar (it is sometimes better, sometimes worse) performance shown in Table \ref{tab:tuebingenQ} in the appendix.
Exploring best practices for how to quantize %
continuous data %
is an interesting avenue for future work.

We now compare performance with other leading methods on this dataset. The total-entropy approach for Entropic Causal Inference achieved $64.21\%$ accuracy at $100\%$ decision rate in \cite{Kocaoglu2017}. ANM methods are evaluated on this data in \cite{mooij2016distinguishing}, where they emphasize two ANM methods with consistent performance that achieve $63 \pm 10\%$ and $69 \pm 10\%$ accuracy. IGCI methods are also evaluated in \cite{mooij2016distinguishing} and were found to vary greatly with implementation and perturbations of data. No IGCI method had consistent performance. LiNGAM methods are evaluated in \cite{hyvarinen2013pairwise} and reported nonlinear approaches with $62\%$ and $69\%$ accuracy. Of these, only Entropic Causal Inference and IGCI can handle categorical data. As comparison with different approaches is difficult given limited data, we suggest assessing the MEC in both directions when deciding how to use our approach in combination with other methods.

%% file: tuebingen.tex
\begin{table}[t!]
\footnotesize
\centering
\scalebox{0.85}
{
\parbox{.25\linewidth}{5-state quantization}\:%
\parbox{.75\linewidth}{
\begin{tabular}{|c|c|c|c|c|c|c|}
\hline
Threshold ($\times \log$ support) & 0.7 & 0.8 & 0.85 & 0.9 & 1.0 & 1.2 \\\hline
\# of pairs & 14 & 25 & 34 & 42 & 57 & 85 \\\hline
Accuracy (\%) & 85.7 & 64.0 & 58.8 & 57.1 & 63.2 & 60.0\\\hline 

\end{tabular}
}
}
\scalebox{0.85}
{
\parbox{.25\linewidth}{\textbf{10-state quantization}}\:%
\parbox{.75\linewidth}{
\begin{tabular}{|c|c|c|c|c|c|c|}
\hline
Threshold ($\times \log$ support) & 0.7 & 0.8 & 0.85 & 0.9 & 1.0 & 1.2 \\\hline
\# of pairs & 13 & 23 & 34 & 46 & 67 & 85 \\\hline
Accuracy (\%) & 84.6 & 73.9 & 70.6 & 63.0 & 61.2 & 56.5\\\hline 

\end{tabular}
}
}
\scalebox{0.85}
{
\parbox{.25\linewidth}{20-state quantization}\:%
\parbox{.75\linewidth}{
\begin{tabular}{|c|c|c|c|c|c|c|}
\hline
Threshold ($\times \log$ support) & 0.7 & 0.8 & 0.85 & 0.9 & 1.0 & 1.2 \\\hline
\# of pairs & 12 & 21 & 41 & 52 & 76 & 85 \\\hline
Accuracy (\%) & 75.0 & 61.9 & 53.7 & 51.9 & 51.3 & 49.4\\\hline 

\end{tabular}
}
}

\caption{{Performance on T{\"u}bingen causal pairs with low exogenous entropy in at least one direction.}}
\label{tab:tuebingen}
\vspace{-5mm}
\end{table}

%% file: Discussion.tex
\section{Discussion}
In this section we discuss several aspects of our method in relation with prior work. First, note that our identifiability result holds \emph{with high probability} under the measure induced by our generative model. This means that, even under our assumptions, not all causal models will be identifiable. However, the non-identifiable fraction vanishes as $n$, i.e., the number of states of $X,Y$ goes to infinity. In essence, this is similar to many of the existing identifiability statements that show identifiability except for an adversarial set of models~\cite{Hoyer2008}. Specifically in \cite{Kocaoglu2017}, the authors show that under the assumption that the exogenous variable has small support size, causal direction is identifiable with probability $1$. This means that the set of non-identifiable models has Lebesgue measure zero. This is clearly a stronger identifiability statement. %
However, this is not surprising if we compare the assumptions: Bounding the support size of a variable bounds its entropy, but not vice verse. Therefore, our assumption can be seen as a relaxation of the assumption of \cite{Kocaoglu2017}. Accordingly, a weaker identifiability result is expected.

Next, we emphasize that our key assumption, that in the true causal direction the exogenous variable has small entropy, is not universal, i.e., one can construct cause-effect pairs where the anti-causal direction requires less entropy. \cite{janzing2019cause} provides an example scenario: Consider a ball traveling at a fixed and known velocity from the initial position $X$ towards a wall that may appear or disappear at a known position with some probability. Let $Y$ be the position of the ball after a fixed amount of time. Clearly we have $X\rightarrow Y$. If the wall appears, the ball ends up in a different position ($y_0$) from the one it would if the wall does not ($y_1$). Then the mapping $X\rightarrow Y$ requires an exogenous variable to describe the behavior of the wall. However, simply by looking at the final position, we can infer whether wall was active or not, and accordingly infer what the initial position was deterministically. This shows that our key assumption is not always valid and should be evaluated depending on the application in mind. 

Finally note that the low-entropy assumption should not be enforced on the exogenous variable of the cause, since this would imply that $X$ has small entropy. This brings about a conceptual issue to extend the idea to more than two variables: Which variables' exogenous noise should have small entropy? For that setting, we believe the original assumption of \cite{Kocaoglu2017} may be more suitable: Assume that the total entropy of the system is small. In the case of more than two variables, this means total entropy of all the exogenous variables is small, without enforcing bounds on specific ones.

%% file: Conclusion.tex
\section{Conclusion}
In this work, we showed the first identifiability result for learning the causal graph between two categorical variables using the entropic causal inference framework. %
We also provided the first finite-sample analysis. We conducted extensive experiments to conclude that the framework, in practice, is robust to some of the assumptions required by theory, such as the amount of exogenous entropy and causal sufficiency assumptions. We evaluated the performance of the method on T{\"u}bingen dataset.

%% file: Appendix.tex
\section*{\fontsize{13}{15}\selectfont Supplementary Material\\ Entropic Causal Inference: Identifiability and Finite Sample Results}%

\pagenumbering{arabic}
\appendix

\section{\fontsize{11}{15}\selectfont Proof of Theorem \ref{thm:main}}\label{app:main}
\textbf{Step 1. Bounding $H(\tilde{E})$ by $H(\tilde{E})\geq H(X|Y=y), \forall y$:}
Consider any $\tilde{E}\indep Y$ for which there exists a deterministic map $g$ such that $X = g(\tilde{E},Y)$. We have
\begin{align*}
    p(X=x|Y=y) &= p(g(\tilde{E},Y)=x|Y=y)\\
    &=p(g(\tilde{E},y)=x) = p(g_y(\tilde{E})=x),
\end{align*}
for $g_y(e)\coloneqq g(e,y), \forall e,y$, since $\tilde{E}\indep Y$. Due to data processing inequality, it follows that $    H(\tilde{E})\geq H(X|Y=y)$.

In \cite{Kocaoglu2017}, this analysis is used to show that the minimum entropy exogenous variable $\tilde{E}$ can be obtained by solving the minimum entropy coupling problem on the conditional distributions $p(X|Y=y)$. Here, we use the conditional entropies to lower bound the entropy of the exogenous variable $\tilde{E}$. Therefore, in the rest of our analysis we attempt to show that under the given assumptions, with high probability, $H(X|Y=y)$ is large for some value of $y$.

\textbf{Step 2. Generative process as a balls and bins game:}
In order to analyze the conditional distributions $p(X|Y=y)$ we relate the generative model to a balls and bins game:

Consider a deterministic map $f:[n]\times[m] \rightarrow [n]$. Let $p(X=i)=x_i$ and $p(E=k)=e_k$. Without loss of generality, assume that $X$ and $E$ are labeled in decreasing probability order. In other words, 
$e_k\geq e_l$ if $k<l$ and $x_i\geq x_j$ if $i<j$.\footnote{This relabeling of $X,E$ is without loss of generality since realization of $f$ is symmetric across rows and columns.} Let $\mat{M}$ be the matrix defined as $\mat{M}_{i,k} \coloneqq f(i,k)$. The probability distribution $p(Y|X)$ is determined by the causal mechanism, i.e., the structural equation $Y=f(X,E)$. The conditional distributions in the wrong causal direction, i.e., $p(X|Y)$ can then be calculated as follows: 
\begin{equation*}
    p(X=i|Y=j)=\frac{1}{Z}x_i\sum_{k=1}^m\mathbbm{1}_{\{\mat{M}_{i,k}=j\}}e_k.
\end{equation*}
$Z = \sum\limits_{i=1}^nx_i\sum\limits_{k=1}^m\mathbbm{1}_{\{\mat{M}_{i,k}=j\}}e_k$ is the normalizing constant.  

To sample $f$ uniformly randomly from all the mappings is equivalent to filling the entries of $\mat{M}$ independently and uniformly randomly from $\mathcal{Y}=[n]$. A small example is given in Table 1, which shows a realization of $f$ through matrix $\mat{M}$, and illustrates how this affects $p(X|Y=1)$. 
\input{table1}

Any realization of $f$ corresponds to a realization of matrix $\mat{M}$. The first column is of special interest to us because it corresponds to the value of $E$ with the highest probability. The realization of $\mat{M}$ can be thought of as a balls and bins process, with the cells corresponding to balls and each entry $\mat{M}_{i,k}$ corresponding to which bin that cell's ball landed in.

\textbf{Step 3. Identify a set of ``good" bins:}
Each coordinate $(i, k)$ is a ball and the value of $\mat{M}_{i,k}$ is the identity of the bin this ball is placed in. We utilize the existence of a set $S$ as described in the theorem statement as follows: We focus on the set of balls corresponding to the cells $(i, 1)$ for $i \in S$. Our goal is to identify a bin which contains a large fraction of these balls. We also want this bin to not contain too much probability mass from balls outside of the set $S$ in order to get a close bound in \textbf{Step 6}.

Recall that each bin $y$ contains mass $x_ie_k$ when $\mat{M}_{i,k}=y$. To restrict our search of a good bin, we first discard all the bins that contain a large mass from entries of $\mat{M}$ that are either in rows corresponding to $x \notin S$ or columns other than the first column. Let $p(X,Y,E)$ represent the joint distribution between $X,Y,E$. Then we discard every value of $y$ where $\sum\limits_{x\notin S}\sum\limits_{e=1}^{m} p(x,y,e) + \sum\limits_{x \in S}\sum\limits_{e=2}^{m} p(x,y,e)$ %
is large. We pick the threshold of $\frac{2}{n}$ and define the set $\mathcal{B}$ accordingly: 
\begin{align*}
    \mathcal{B}=\bigg\{y: \sum\limits_{x \notin S} p(x , y)
     + \sum\limits_{x \in S} p(x, y, E>1) > \frac{2}{n}\bigg\}.  
\end{align*} 
We know that $\lvert\mathcal{B}\rvert \leq \frac{n}{2}$, since otherwise the total mass would exceed $1$.\footnote{The probabilities we sum correspond to disjoint events, hence the total probability cannot exceed $1$.} Let $\mathcal{U}\coloneqq [n]\backslash \mathcal{B}$. Then $\lvert \mathcal{U}\rvert \geq n/2$. Note that $\mathcal{B}$ and $\mathcal{U}$ are determined in a manner not affected by the realized values of $\mat{M}_{x,1}$ for $x \in S$. We will next focus on only the values of $y\in \mathcal{U}$, and later quantify the following claim: A significant fraction of the probability mass that falls in any bin in $\mathcal{U}$ is due to entries from $\mat{M}_{x,1}$ for $x\in S$. Therefore, for one of these bins $y \in \mathcal{U}$, we can focus on obtaining a lower bound of $H(X|Y=y,X\in S, \mat{M}_{X,1}=y)$ to later show that $H(X|Y=y)$ cannot be much smaller.

\textbf{Step 4. Show a bin from $\mathcal{U}$ has many balls from the first column of $\mat{M}$ and rows in $S$:} We focus our attention to the balls in $S$ and bins in $\mathcal{U}$. We want to show that %
$\exists y \in \mathcal{U}$ such that $\mat{M}_{x,1}=y$ for a large number of values of $x \in S$. Recall that since $\lvert S \rvert\geq dn$, we have at least $dn$ balls falling into $n$ bins. Moreover, since $\lvert \mathcal{U}\rvert\geq n/2$, at least $n/2$ of these bins are ``good" for us. First, we show that, with high probability, at least $\frac{dn}{4}$ of the $dn$ balls fall in the bins in $\mathcal{U}$. 
\begin{lemma}
\label{lem:sub_ballsandbins}
Consider the process of uniformly randomly throwing $dn=\Theta(n)$ balls into $n$ bins.\footnote{Uniformity follows from uniformity of $f$.} Let $\mathcal{U}$ be an arbitrary, fixed subset of bins with size $|\mathcal{U}| \ge \frac{n}{2}$. Then with high probability, at least $\frac{dn}{4}$ balls fall into the bins in $\mathcal{U}$. Moreover, these balls are also uniformly randomly thrown.
\end{lemma}
The above lemma, proven in Appendix \ref{app:sub_ballsandbins} is directly applicable to our setting, even though $\mathcal{U}$ is a random variable. This is because the realization of the entries of $\mat{M}$ outside the rows $S$ or outside the first column, which determines the set $\mathcal{U}$ are independent from the entries in $\mat{M}$ in the rows $S$ and in the first column. In other words, how balls are thrown into the bins in $\mathcal{U}$ is not affected by how $\mathcal{U}$ is chosen. 

We want to use this to show that there is a bin $y \in \mathcal{U}$ such that the conditional distribution $p(X\lvert Y=y)$ is due to many balls $x \in S$ where $\mat{M}_{x,1}=y$. We have shown that with high probability at least $\frac{dn}{4}$ balls land in bins corresponding to $y \in \mathcal{U}$. We apply a bound from Theorem 1 of \cite{raab1998balls}, which implies that with high probability when there are $b$ bins and $\eta b$ balls ($\eta = \Theta(1)$), the most loaded bin has at least $\frac{\ln(b)}{\ln(\ln(b)) + \ln\left(\frac{1}{\eta}\right)}$ balls. We know that with high probability we have some number of balls in range $[\frac{nd}{4},nd]$ in some number of good bins in range $[\frac{n}{2},n]$. In terms of the established bound on the most loaded bin, this means $\eta \ge \frac{d}{4}$ and $b \in [\frac{n}{2},n]$. If we substitute valid values of $\eta$ and $b$ that minimize the lower bound, we know that with high probability the heaviest loaded bin among $\mathcal{U}$ conditional distributions has at least $\frac{\ln(n)-\ln(2)}{\ln(\ln(n)) + \ln(\frac{4}{d})}$ balls. Without loss of generality, suppose this bin has label $2$. We show that $H(X | Y=2)$ is large using the above bound.

\textbf{Step 5. Bounding $H(X|Y=2)$:}
Next, we obtain a lower bound for $H(X|Y=2)$. We utilize the following lemma, proved in Section \ref{app:entropy_fraction_mass} of the supplement:
\begin{lemma}
\label{lem:entropy_fraction_mass}
Let $X$ be a discrete random variable with distribution $[p_1,p_2,\hdots,p_n]$. Consider the random variable $X'$ with distribution $[\frac{p_i}{\sum_{j\in S'}p_j}]_i$ for any $S' \subseteq [n]$. Then, $H(X) \ge \mu H(X')$, where $\mu = \sum_{i\in S'}p_i$.
\end{lemma}
To use this lemma, we consider a specific distribution induced on the support of $X\lvert Y=2$. First, let us define the following: For any subset $S'\subseteq [n],y\in [n]$, let $X_{S',y}$ be the discrete variable with the following distribution: 
\begin{equation}
    p(X_{S',y}\! =\! i)\! =\! \frac{p(X=i\lvert Y=y)}{\sum_{l \in S'} p(X=l\lvert Y=y)}, \forall i \in S'.
\end{equation}
We focus on $X_{S', 2}$, where $S' = \{i: i \in S, \mat{M}_{i,1} = 2\}$. We first show that $H(X_{S',2})$ is large, and then show the total mass $\mu = \sum_{i \in S'} p(X=i\lvert Y=2)$ that $X_{S',2}$ contributes to $(X|Y=2)$ is large, which allows us to use Lemma \ref{lem:entropy_fraction_mass}.

To show $H(X_{S',2})$ is large, we use the following lemma from \cite{Cicalese2017h}:
\begin{lemma}[Theorem 2 of \cite{Cicalese2017h}]
\label{lem:Cicalese}
Let $X$ be a strictly positive discrete random variable on $n$ states such that $ \frac{\max_{i}p(X=i)}{\min_i{p(X=i)}}\leq \rho$. Then \begin{equation*}
    H(X)\geq \log(n)-\left(\frac{\rho\ln(\rho)}{\rho-1}-1-\ln\left(\frac{\rho\ln(\rho)}{\rho-1}\right)\right)\frac{1}{\ln(2)}.
\end{equation*}
\end{lemma}
To lower bound $H(X_{S',y})$ using the above lemma, we obtain an upper bound to $\rho' \coloneqq \frac{\max_i {p(X_{S',2}=i)}}{\min_i {p(X_{S',2}=i)}}$ by utilizing our knowledge that $H(E) = c$. For each value $i \in S'$, we know that $\mat{M}_{i,1}=2$. Thus, $p(X_{S',2}=i) \ge \frac{x_i e_1}{\mu}$. Also $p(X_{S',2}=i)\leq \frac{x_i\sum_{k=1}^me_k}{\mu} = \frac{x_i}{\mu}$ and $\frac{\max_{i \in S} x_i}{\min_{i \in S} x_i} \le \rho$. Therefore $\rho' \le \frac{\max_i \frac{x_i}{\mu}}{\min_i \frac{x_i e_1}{\mu}} \le \frac{\rho}{e_1}$.

In order to understand how small $e_1$ can be under the given constraints, we obtain a useful characterization for constant entropy distributions. The following lemma shows that the maximum probability value for any discrete distribution with constant entropy is a constant away from zero.
\begin{lemma}
\label{lem:cons_max_prob}
Let $E$ be a discrete random variable with $m$ states, with the probability distribution $[e_1,e_2, \hdots, e_m]$, where without loss of generality $e_i\geq e_j, \forall j>i$. If $H(E)\leq c$ %
then $e_1\geq 2^{-c}$.%
\end{lemma}
The proof is given in Section \ref{app:cons_max_prob} in the supplement.

 Applying Lemmas \ref{lem:entropy_fraction_mass}-\ref{lem:cons_max_prob}, with some derivation we show in Section \ref{app:HtildeBd} of the supplement that:
 \begin{proposition}[\textbf{Step 6}]\label{lem:HtildeBd}
 Under the conditions stated above, 
 \begin{align*}
 &H(\tilde{E}) \ge \max_y H(X | Y=y) \ge H(X|Y=2) \\& \ge (1 - o(1)) [\log(\log(n)) - \log(\log(\log(n))) - \mathcal{O}(1)].
 \end{align*}
 Furthermore, to make the trade-off between the strength of the lower bound and assumptions on $n$ more explicit, when $n \ge \nu(r,q,\rho,c,d)$ with \begin{align*}
 \nu(r,q,\rho,c,d) = \max\{4,e^{\left(\frac{4}{d}\right)^{1/r}},2e^{q^22^{2(c+1)}\rho}\},
 \end{align*}
 we have
 \begin{align*}
 &H(\tilde{E}) \ge \max_y H(X | Y=y)\ge H(X|Y=2) \\&\ge \left(1 - \frac{1+r}{1+q}\right)\left(0.5 \log(\log(n)) - \log(1+r) - \mathcal{O}(1)\right).
 \end{align*}
 \end{proposition}

 This completes the proof of Theorem \ref{thm:main}. \hfill \qed
 
\textbf{Potential Improvements and Limitations: } 
In our analysis, we use $\max_y H(X|Y=y)$ to bound $H(\tilde{E})$. One potential improvement might be obtained by considering the gap between $H(\tilde{E})$ and the collection $\{H(X|Y=y)\}_y$ for a given $p(x,y)$. \cite{Kocaoglu2017} showed that the smallest $H(\tilde{E})$ is given by the minimum entropy coupling of the conditional distributions $\{p(X|Y=y)\}_y$. Follow-up works have developed minimum-entropy coupling algorithms \cite{Cicalese2017,Kocaoglu2017b,rossi2019greedy} and obtained approximation guarantees. However there is currently no tight analysis characterizing this entropy gap. 

Note that the original conjecture proposes that $H(E)\leq \log(n)+\mathcal{O}(1)$ is sufficient. This is a very strong statement and we believe, even if it is true, it requires a much deeper understanding on the minimum entropy couplings than is currently available in the literature. We do, however, provide evidence in Section \ref{sec:experiments} that $H(E)\leq \alpha\log(n)$ for $\alpha<1$ seems sufficient for identifiability.

One point in our analysis that is related to this setting when $H(E)$ scales with $n$, is that we only considered the first column of the matrix $\mat{M}$, i.e., we have only taken into account the probability values of the form $x_i e_1$ contributing to the entropy of $H(X|Y=y)$. As long as the function $f$ is sampled uniformly randomly in the considered generative model, this approach cannot give $H(\tilde{E})\gg\log(\log(n))$ due to the support size of $X$ being upper bounded by $\mathcal{O}(\log(n))$ with high probability from the balls and bins perspective. For when $H(E)$ is very small, we do expect this to be a reasonable approach as the remaining columns have very small probability values, hence very small impact. However, for going beyond the current analysis and for proving identifiability when $H(E)$ scales with $n$, we strongly believe that the effect of the remaining columns should be considered.

\section{\fontsize{11}{15}\selectfont Proof of Lemma \ref{lem:sub_ballsandbins}}
\label{app:sub_ballsandbins}
Let $\varepsilon$ be the event that less than $\frac{dn}{4}$ balls fall in the bins in $\mathcal{U}$. We provide an upper bound for the probability of this event $P(\varepsilon)$. Consider the indicator variables each corresponding to the event that a particular ball lands in $\mathcal{U}$. These indicator variables are independently and identically distributed, where each has probability $\frac{\mathcal{U}}{n} \ge \frac{1}{2}$ of being $1$. We use Hoeffding's inequality to bound $P(\varepsilon)$. Let $S_{dn}$ be the sum of the $dn$ indicator variables (i.e., the number of the balls that land in bins corresponding to $\mathcal{U}$) and $E_{dn}$ be the expected sum of the indicator variables ($E_{dn} = dn\left(\frac{\mathcal{U}}{n}\right)$).
\begin{align}
&P(\varepsilon)= P\left(S_{dn} < \frac{dn}{4}\right)\label{eq:event_bound2}\\
&\le P\left(|S_{dn} - E_{dn}| > \left|E_{dn} - \frac{dn}{4}\right|\right)\label{eq:event_bound3}\\
&\le P\left(|S_{dn} - E_{dn}| >  \frac{dn}{2} - \frac{dn}{4}\right)\label{eq:event_bound4}\\
&\le P\left(|S_{dn} - E_{dn}| > \frac{dn}{4}\right)=2e^{-\frac{dn}{8}}\label{eq:event_bound5}
\end{align} 
(\ref{eq:event_bound3}) to (\ref{eq:event_bound4}) is due the fact that for all valid values of $\mathcal{U}$, it holds that $E_{dn} = dn (\frac{\mathcal{U}}{n}) \ge \frac{dn}{2}$. (\ref{eq:event_bound5}) is due to Hoeffding's inequality. As such, $P(\varepsilon) \le 2e^{-\frac{dn}{8}}$. Thus, with high probability there are at least $\frac{dn}{4}$ balls that fall into bins corresponding to $\mathcal{U}$. Since balls are thrown independently and uniformly at random, conditioned on the balls that land in $\mathcal{U}$, they are thrown independently and uniformly at random.
\hfill\qed

\section{\fontsize{11}{15}\selectfont Proof of Lemma \ref{lem:entropy_fraction_mass}}
\label{app:entropy_fraction_mass}
Recall that $\mu = \sum_{i\in S'}p(X=i)$. We have
\begin{align*}
H(X)
&\ge \sum_{i \in S'} p(X = i) \log{\left(\frac{1}{p(X=i)}\right)}\\
&= \mu\left(\sum_{i \in S'} \frac{p(X = i)}{\mu} \log{\left(\frac{1}{p(X=i)}\right)}\right)\\
&\ge \mu\left(\sum_{i \in S'} \frac{p(X = i)}{\mu} \log{\left(\frac{\mu}{p(X=i)}\right)}\right)\\
&= \mu\left(\sum_{i \in S'} p(X'=i) \log{\left(\frac{1}{p(X'=i)}\right)}\right)\\
&= \mu H(X').\text{\hspace{1.93in}}\qed
\end{align*}

\section{\fontsize{11}{15}\selectfont Proof of Lemma \ref{lem:cons_max_prob}}\label{app:cons_max_prob}
We show the contrapositive. Suppose that $p_1\leq \varepsilon < 2^{-c}$. We have $p_i\leq p_1, \forall i \in [m]$.
We consider all such distributions and find the one with smallest entropy:
\begin{equation}
\begin{aligned}
    \min_{p_1\geq p_2,\hdots \geq p_m} \quad & H([p_1,p_2,\hdots,p_m]) \\
    \mathrm{s.t.} \quad & \sum_i p_i = 1 \\ 
     \quad & \varepsilon \geq p_i \geq 0,  \forall i \in [m]
\end{aligned}
\end{equation}
For simplicity, suppose $\frac{1}{\varepsilon}$ is an integer. We show that the solution to the above optimization problem is strictly greater than $c$ using majorization theory. For any given $p$, define the vector $u_p = [\sum_{j=1}^i p_j]_i$. Recall that a probability distribution $p$ majorizes another distribution $q$ if $u_p(i)\geq u_q(i), \forall i \in [m]$. Also if $p$ majorizes $q$, we have $H(p)\leq H(q)$.

Consider all distributions in the feasible region of the above problem. For any $p^*$, consider the vector $u_{p^*}$. Clearly, $u_{p^*}(1)\geq \varepsilon$. Since $p_2\leq p_1< \varepsilon$, we have that $u_{p^*}(2)\leq 2\varepsilon$. Similarly, we have $u_{p^*}(i)\leq \varepsilon$. The uniform distribution achieves this upper bounding $u$ vector, establishing that the uniform distribution majorizes every other distribution in the feasible set. Then for any distribution in the feasible region, we get that $H(p)\geq \log(\frac{1}{\varepsilon})>c$.

Suppose $\frac{1}{\varepsilon}$ is not an integer. Let $t$ be the largest integer such that $t\varepsilon\leq 1$. Then the above argument leads to the distribution with entropy 
\begin{equation}
    H = t\varepsilon\log\left(\frac{1}{\varepsilon}\right) + (1-t\varepsilon)\log\left(\frac{1}{1-t\varepsilon}\right).
\end{equation}
Next, we show that if $\varepsilon<2^{-c}$, above value is greater than $c$. We can rewrite
\begin{align}
    H &= t\varepsilon\log\left( \frac{1}{\varepsilon} \right)+(1 - t \varepsilon)\log\left(\frac{1}{1-t\varepsilon}\right)\\
    &\geq t\varepsilon\log\left( \frac{1}{\varepsilon} \right)+(1-t\varepsilon)\log\left(\frac{1}{\varepsilon}\right)\\
    &= \log\left( \frac{1}{\varepsilon} \right) > c
\end{align}
since $1-t\varepsilon \le \varepsilon$. This concludes the proof.\hfill\qed

\section{\fontsize{11}{15}\selectfont Proof of Proposition \ref{lem:HtildeBd}}
\label{app:HtildeBd}

By Lemma \ref{lem:cons_max_prob} we then know $\rho' \le \frac{\rho}{e_1} \le \rho 2^c$, and the size of the support of $X_{S',2}$ is the number of balls in the most loaded bin which is at least $\frac{\ln{(n)}-\ln{(2)}}{\ln{(\ln{(n)})} + \ln{(\frac{4}{d})}}$. Using Lemma \ref{lem:Cicalese}, we conclude $H(X_{S',2}) \ge \log{(\frac{\ln{(n)}-\ln{(2)}}{\ln{(\ln{(n)})} + \ln{(\frac{4}{d})}})} - (\frac{\rho 2^c \ln{(\rho 2^c)}}{\rho 2^c -1} - 1 - \ln{(\frac{\rho 2^c \ln{(\rho 2^c)}}{\rho 2^c -1})})\frac{1}{\ln{(2)}}$.

Using our previous results, we know that $\min_{i \in S'} p(X=i , Y=2) \ge \min_{i \in S'} e_1 x_i \ge \frac{e_1}{\sqrt{\rho} n} \ge \frac{2^{-c}}{\sqrt{\rho} n}$. Then, $p(X \in S', Y=2) \ge (\frac{\ln{(n)}-\ln{(2)}}{\ln{(\ln{(n)})} + \ln{(\frac{4}{d})}})(\frac{2^{-c}}{\sqrt{\rho} n}) = \frac{\ln{(n)}-\ln{(2)}}{(\ln{(\ln{(n)})} + \ln{(\frac{4}{d})})\sqrt{\rho} n 2^c}$. Additionally:

\begin{align}
&p(X \notin S', Y=2)= \sum_{i \in S^c}\sum_{j = 1}^m p(X=i,Y=2, E=j)  \nonumber\\
&\hspace{0.5in}+\sum_{i \in S, i \notin S'}\sum_{j=1}^m p(X=i,Y=2,E=j)\label{eq:other_mass2}\\
&= \sum_{i \in S^c}\sum_{j = 1}^m p(X=i,Y=2, E=j) \nonumber\\
&+  \sum_{i \in S, i \notin S'}\sum_{j=2}^m p(X=i,Y=2,E=j)
\label{eq:other_mass3}\\
&\le \sum_{i \in S^c}\sum_{j = 1}^m p(X=i,Y=2, E=j) \nonumber\\
&+ \sum_{i \in S}\sum_{j=2}^m p(X=i,Y=2,E=j) \le \frac{2}{n}. \label{eq:other_mass5}
\end{align}

 We go from \eqref{eq:other_mass2} to (\ref{eq:other_mass3}) by realizing that for any $i \in S$, $p(X=i, Y=2, E=1)>0$ only if $\mat{M}_{x,1}=2$ and thus $i \in S'$. We simplify \eqref{eq:other_mass3} by definition of $\mathcal{U}$. As such, $p(X \in S' | Y=2) = \frac{p(X \in S', Y=2)}{p(X \in S', Y=2) + p(X \notin S', Y=2)} \ge \frac{\frac{\ln{(n)}-\ln{(2)}}{(\ln{(\ln{(n)})} + \ln{(\frac{4}{d})})\sqrt{\rho} n 2^c}}{\frac{\ln{(n)}-\ln{(2)}}{(\ln{(\ln{(n)})} + \ln{(\frac{4}{d})})\sqrt{\rho} n 2^c} + \frac{2}{n}} = \frac{\ln{(n)} - \ln{(2)}}{\ln{(n)}-\ln{(2)} + (\ln{(\ln{(n)})} + \ln{(\frac{4}{d})}) \sqrt{\rho} 2^{c+1}}$. Thus, we have shown that $H(X_{S',2}) \ge \log{(\frac{\ln{(n)}-\ln{(2)}}{\ln{(\ln{(n)})} + \ln{(\frac{4}{d})}})} - (\frac{\rho 2^c \ln{(\rho 2^c)}}{\rho 2^c -1} - 1 - \ln{(\frac{\rho 2^c \ln{(\rho 2^c)}}{\rho 2^c -1})})\frac{1}{\ln{(2)}}$ and $P(X \in S', Y=2) \ge \frac{\ln{(n)} - \ln{(2)}}{\ln{(n)}-\ln{(2)} + (\ln{(\ln{(n)})} + \ln{(\frac{4}{d})}) \sqrt{\rho} 2^{c+1}}$.

 Using Lemma \ref{lem:entropy_fraction_mass} we have:
 \begin{align}
     &H(\tilde{E}) \geq H(X | Y=2) \ge P(X \in S', Y=2)(H(X_{S',2})) \nonumber\\
     &\geq \left(\frac{\ln{(n)} - \ln{(2)}}{\ln{(n)}-\ln{(2)} + (\ln{(\ln{(n)})} + \ln{(\frac{4}{d})}) \sqrt{\rho} 2^{c+1}}\right)\nonumber\\
     &\hspace{0.3in}\left(\log{(\frac{\ln{(n)}-\ln{(2)}}{\ln{(\ln{(n)})} + \ln{(\frac{4}{d})}})} \right.\nonumber\\
     &\left.\hspace{0.3in}- \left(\frac{\rho 2^c \ln{(\rho 2^c)}}{\rho 2^c -1} - 1 - \ln{(\frac{\rho 2^c \ln{(\rho 2^c)}}{\rho 2^c -1})}\right)\frac{1}{\ln{(2)}}\right)\nonumber\\ 
     &= \left(1 - \frac{(\ln{(\ln{(n)})} + \ln{(\frac{4}{d})}) \sqrt{\rho} 2^{c+1}}{\ln{(n)}-\ln{(2)} + (\ln{(\ln{(n)})} + \ln{(\frac{4}{d})}) \sqrt{\rho} 2^{c+1}}\right)\nonumber\\
     &\left(\log{(\frac{\ln{(n)}-\ln{(2)}}{\ln{(\ln{(n)})} + \ln{(\frac{4}{d})}})}\right. \nonumber\\
     &\hspace{0.3in}\left.- \left(\frac{\rho 2^c \ln{(\rho 2^c)}}{\rho 2^c -1} - 1 - \ln{(\frac{\rho 2^c \ln{(\rho 2^c)}}{\rho 2^c -1})}\right)\frac{1}{\ln{(2)}}\right). \label{eq:HeLower}
 \end{align}
 Since $c=O(1)$ and $d = \Theta(1)$, this lower bound is asymptotically $H(\tilde{E}) \ge \max_y H(X | Y=y) \ge H(X | Y=2) \ge (1 - o(1)) (\log{(\log{(n)})} - \log{(\log{(\log{(n)})})} - \mathcal{O}(1))$. 
 
 Now when $n \ge \nu(r,q,\rho,c,d)$, we can lower bound the $(1-o(1))$ term as:
 \begin{align}
     &1 - \frac{(\ln{(\ln{(n)})}+\ln{(\frac{4}{d})})\sqrt{\rho}2^{c+1}}{\ln{(n)}-\ln{(2)}+(\ln{(\ln{(n)})}+\ln{(\frac{4}{d})})\sqrt{\rho}2^{c+1}} \label{eq:first_term1}\\
     &\ge 1 - \frac{(1+r)\ln{(\ln{(n)})} \sqrt{\rho} 2^{c+1}}{\ln{(n/2)}+\ln{(\ln{(n)})}\sqrt{\rho}2^{c+1}} \label{eq:first_term2}\\
     &\ge 1 - \frac{(1+r) \sqrt{\ln{(n/2)}} \sqrt{\rho}2^{c+1}}{\ln{(n/2)}+\sqrt{\ln{(n/2)}} \sqrt{\rho} 2^{c+1}}\label{eq:first_term3}\\
     &= 1 - \frac{1 + r}{1 + \frac{\ln{(n/2)}}{\sqrt{\rho}2^{c+1}}} \label{eq:first_term4}\\
     &\ge 1 - \frac{1+r}{1+q} \label{eq:first_term5}
 \end{align}
 
 We bound from (\ref{eq:first_term1}) to (\ref{eq:first_term2}) by using $n \ge e^{\left(\frac{4}{d}\right)^{1/r}}$ which implies $\ln{(\ln{(n)})} + \ln{(\frac{4}{d})} \le (1+r)\ln{(\ln{(n)})}$. We go from (\ref{eq:first_term2}) to (\ref{eq:first_term3}) by using $\sqrt{\ln{(n/2)}} \ge \ln{(\ln{(n)})}$ when $n \ge 3$. We bound from (\ref{eq:first_term4}) to (\ref{eq:first_term5}) by using $n \ge 2e^{q^22^{2(c+1)}\rho}$. Next, we lower bound the term $\log{\left(\frac{\ln{(n)} - \ln{(2)}}{\ln{(\ln{(n)})} + \ln{(\frac{4}{d})}}\right)}$.
 
 \begin{align}
     &\log{\left(\frac{\ln{(n)} - \ln{(2)}}{\ln{(\ln{(n)})} + \ln{(\frac{4}{d})}}\right)}\label{eq:second_term1}\\
     &\ge \log{\left(\frac{\ln{(n/2)}}{(1+r)\ln{(\ln{(n)})}}\right)}\label{eq:second_term2}\\
     &\ge \log{\left(\sqrt{\ln{(n/2)}}\right)} - \log{(1+r)}\label{eq:second_term3}\\
     &\ge 0.5 \log{(0.5 \log{(n/2)})} - \log{(1+r)}\label{eq:second_term5}\\
     &\ge 0.5 \log{(\log{(n)})} - \log{(1+r)} - 1\label{eq:second_term8}
 \end{align}
 
 We bound from (\ref{eq:second_term1}) to (\ref{eq:second_term2}) by using $\ln{(\ln{(n)})}+\ln{(\frac{4}{d})} \le (1+r) \ln{(\ln{(n)})}$. We bound from (\ref{eq:second_term2}) to (\ref{eq:second_term3}) using $\sqrt{\ln{(n/2)}} \ge \ln{(\ln{(n)})}$. We then substitute all of these bounds into our previous lower bound on $H(\tilde{E})$ \eqref{eq:HeLower} yielding:
 \begin{align*}
     &H(\tilde{E}) \ge \left(1 - \frac{1+r}{1+q}\right) \bigg(0.5 \log{(\log{(n)})} - \log{(1+r)}\\
     & \left. - \mathcal{O}(1) -\frac{1}{\ln(2)}\!\!\left(\frac{\rho 2^c \ln(\rho 2^c)}{\rho 2^c -1} - 1 - \ln\left(\frac{\rho 2^c \ln{(\rho 2^c)}}{\rho 2^c -1}\right)\!\!\right)\!\!\right)\!\!\\
     &= \left(1 - \frac{1+r}{1+q}\right)\left(0.5 \log(\log(n)) - \log(1+r) - \mathcal{O}(1)\right).
 \end{align*}

\section{\fontsize{11}{15}\selectfont Proof of Corollary \ref{cor:identifiability}}
\label{app:identifiability}
\subsubsection{Condition (a): Bounded Ratio}
We know that $\frac{\max_xp(x)}{\min_xp(x)}\leq \rho$. Since $\sum_x p(x) = 1$, $\min_xp(x)\leq \frac{1}{n}\leq \max_xp(x)$ and we have $\frac{\max_xp(x)}{1/n}\leq \rho\Rightarrow \max_xp(x)\leq \frac{\rho}{n}$ and similarly $\min_xp(x)\geq \frac{1}{\rho n}$. Then using Theorem \ref{thm:main}, when $n\ge \nu(r=1, q=3, \rho^2, c, d = 1)$, $H(\tilde{E}) \ge \max_{y} H(X|Y=y) \ge 0.25 \log{(\log{(n)})} - \mathcal{O}(1)$ with high probability (where the $\mathcal{O}(1)$ term is a function of only $\rho,c$). As such, there exists an $n_0$ (which is a function of only $\rho,c$) such that for all $n > n_0$, the causal direction is identifiable with high probability.

\subsection{Condition (b): Sampled Uniformly on the Simplex}
We first show that when the distribution of $X$ is uniformly sampled from the simplex, there exist a set $S$ that satisfies the assumptions of Theorem \ref{thm:main} with high probability.
\begin{lemma}
When the $x_i$ are sampled uniformly from the simplex, there exists a subset of the support with size at least $ (e^{-\frac{1}{\sqrt{\rho}}} - \frac{1}{\sqrt{\rho}} - \delta)n$ for which all $x_i$ are within a factor of $\sqrt{\rho}$ from $\frac{1}{n}$ and make up total probability mass $\geq \left(e^{-\frac{1}{\sqrt{\rho}}} - \frac{1}{\sqrt{\rho}} - \delta\right)  \frac{1}{\sqrt{\rho}}$, with probability $> 1 - 2e^{-2 \delta^2 n}$ for $\rho, n \ge 1$, $\delta > 0$.
\end{lemma}

\begin{proof}
 Let us call a probability ``small'' if $x_i \le \frac{1}{\sqrt{\rho}n}$. We want to show that with high probability (at least $ 1 - 2e^{-2 \delta^2 n}$), there are at most $(1 - e^{-\frac{1}{\sqrt{\rho}}} + \delta)n$ small $x_i$. Using Theorem 3 of \cite{marsaglia1961uniform}, we know that for each $x_i$ in a Dirichlet distribution with $\alpha = 1$ (i.e., the uniform distribution over the probability simplex) and support size $n$, $P(x_i > z) = (1-z)^{n-1}$ (This is by setting $a_i=z$ and $a_j=0,\forall j\neq i$ and using the fact that $P(x_i=0)=0,\forall i\in[n]$). As such, $P(x_i \le z) = 1 - (1-z)^{n-1}$. The probability that $x_i$ is small is then equal to $P(x_i \le \frac{1}{\sqrt{\rho}n}) = 1 - (1-\frac{1}{\sqrt{\rho}n})^{n-1}$. This value is non-decreasing when $n \ge 1$, and approaches $1-e^{-\frac{1}{\sqrt{\rho}}}$ as $n$ approaches infinity. Hence when $n \ge 1$, the probability that any $x_i$ is ``small'' is upper-bounded by $1-e^{-\frac{1}{\sqrt{\rho}}}$. We want to show that the outcome that there are more than $(1 - e^{-\frac{1}{\sqrt{\rho}}} + \delta)n$ small $x_i$ will not happen with high probability. To do this, we note that all $x_i$ in a symmetric Dirichlet distribution are negatively associated (this follows from Lemma \ref{lem:neg_assoc} in Section \ref{app:negAssoc}). This implies that the probability that there are at least $(1 - e^{-\frac{1}{\sqrt{\rho}}} + \delta)n$ small $x_i$ is upper-bounded by the probability that there are at least that many $x_i$ when we treat the $x_i$ as if they are i.i.d. random variables. This allows us to use Hoeffding's inequality. Let $S_n$ be the total number of small $x_i$ and $E_n$ be the expected number of small $x_i$. Since $E_n \le (1 - e^{-\frac{1}{\sqrt{\rho}}})n$, then $P(S_n > (1 - e^{-\frac{1}{\sqrt{\rho}}} + \delta)n) \le P(|S_n - E_n | > \delta n) < 2e^{-2 \delta^2 n}$. As such, the probability that there are at most $(1 - e^{-\frac{1}{\sqrt{\rho}}} + \delta)n$ small $x_i$ is at least $(1 - 2e^{-2 \delta^2 n})$.

Let us call an $x_i$ ``big'' if $x_i \ge \frac{\sqrt{\rho}}{n}$. There are at most $\frac{n}{\sqrt{\rho}}$ big $x_i$, since otherwise their total probability mass would exceed $1$.

Next, consider the subset of $x_i$ that are neither ``big'' nor ``small''. They are in the range $[\frac{1}{\sqrt{\rho}n},\frac{\sqrt{\rho}}{n}]$. We know that with high probability $(1 - 2e^{-2 \delta^2 n})$ there are at most $(1 - e^{-\frac{1}{\sqrt{\rho}}} + \delta)n$ small $x_i$ and at most $\frac{n}{\sqrt{\rho}}$ big $x_i$. This means our desired subset has size at least $\left(n - (1 - e^{-\frac{1}{\sqrt{\rho}}} + \delta)n - \frac{n}{\sqrt{\rho}}\right) = \left(e^{-\frac{1}{\sqrt{\rho}}} - \frac{1}{\sqrt{\rho}} - \delta\right)n$ with probability at least $1 - 2e^{-2 \delta^2 n}$.
\end{proof}

As such, if we set $\rho = 25$ and $\delta = 0.1$, there exists a subset of the support of size $\ge (e^{- \frac{1}{\sqrt{25}}} - \frac{1}{\sqrt{25}} - 0.1)n \ge 0.5n$ where all $x_i$ are within a factor of $\sqrt{25}=5$ from $\frac{1}{n}$ with probability $>1-2e^{-2 (0.1)^2 n} = 1- 2e^{-0.02n}$. Using Theorem \ref{thm:main}, we conclude that when $n \ge \nu(r=1,q=3,\rho=25,c,d=0.5)$, $H(\tilde{E}) \ge \max_y H(X | Y=y) \ge 0.25 \log{(\log{(n)})} - \mathcal{O}(1)$ with high probability (where the $\mathcal{O}(1)$ term is a function of only $c$). As such, there exists an $n_0$ (which is a function of only $c$) such that for all $n > n_0$, the causal direction is identifiable with high probability.

\subsection{Condition (c): High Entropy}
We show that when $X$ has entropy within an additive constant of $\log{(n)}$, there exists a set $S$ that satisfies the assumptions of Theorem \ref{thm:main}.
\begin{lemma}
For any distribution $X$ with support size $n$ and entropy $\ge \log{(n)} - a$, there exists a subset $S$ with all $x_i \in . [\frac{3}{40n},\frac{2^{2b}}{n}]$ for $i \in S$, and support size $|S| \ge \frac{n}{2^{2b+3}}$, where $b = \max\{a,2\}$.  
\end{lemma}
\begin{proof}
Let us call an $x_i$ ``large'' if $x_i \ge \frac{2^{2b}}{n}$, and $\mu_{\mathrm{large}}$ be the total probability mass contributed by large $x_i$. The upper bound for the sum of the terms in the formula for $H(X)$ corresponding to large $x_i$ is $\mu_{\mathrm{large}} \log{(\frac{n}{2^{2b}})}$. The upper bound for the sum of the terms in Shannon entropy corresponding to $x_i$ that are not large is $(1-\mu_{\mathrm{large}}) \log{(\frac{n}{1-\mu_{\mathrm{large}}})}$. Since entropy is greater than $\log(n)-a$ and $b=\max\{a,2\}$, we have that entropy is greater than or equal to $\log(n)-b$ as well. Then, for the total entropy to be at least $\log(n)-b$ it must be true that $\mu_{\mathrm{large}} \log{(\frac{n}{2^{2b}})} + (1-\mu_{\mathrm{large}}) \log{(\frac{n}{1-\mu_{\mathrm{large}}})} \ge \log{(n)} - b$. It follows that $2b \mu_{\mathrm{large}} + (1-\mu_{\mathrm{large}}) \log{\left(1-\mu_{\mathrm{large}}\right)} \le b$. For $x\geq 0$, we have that $(1-x)\log(1-x)\geq -1.5x$. Then we have $2\mu_{\mathrm{large}}(b-0.75)  \le b$, or equivalently $\mu_{\mathrm{large}} \leq \frac{b}{2(b-0.75)} $. %
Since $b\geq 2$, we have that $\mu_{\mathrm{large}} \le 0.8$. %

Let us call an $x_i$ ``small'' if it is $\le \frac{0.075}{n}$, and let $\mu_{\mathrm{small}}$ be the total probability mass in small $x_i$. Even if all $x_i$ were small (although that would be impossible), $\mu_{\mathrm{small}} \le 0.075$. As such, $\mu_{\mathrm{small}}+\mu_{\mathrm{large}} \le \frac{7}{8}$. This means at least $\frac{1}{8}$ total probability mass belongs to $x_i \in [\frac{0.075}{n}, \frac{2^{2b}}{n}]$. Our subset $S$ of $X$ will be all of these $x_i$. Since every element in $X$ is upper-bounded by $\frac{2^{2b}}{n}$, $S$ has a support size of at least $\frac{\frac{1}{8}}{\frac{2^{2b}}{n}} = \frac{n}{2^{2b+3}}$.
\end{proof}
We can therefore satisfy the conditions of Theorem \ref{thm:main} with $d = \frac{1}{2^{2\max\{a,2\}+3}}$ and $\rho \le (\frac{40}{3}2^{2 \max\{a,2\}})^2\leq 2^{4\max\{a,2\}+8}$. Using Theorem \ref{thm:main}, we conclude that when $n \ge \nu(r=1, q=3, \rho = 2^{4\max\{a,2\}+8}, c, d=\frac{1}{2^{2\max\{a,2\}+3}})$, $H(\tilde{E}) \ge \max_{y} H(X | Y=y) \ge 0.25 \log{(\log{(n)})} - \mathcal{O}(1)$ with high probability (where the $\mathcal{O}(1)$ term is a function of only $a,c$). Hence there exists an $n_0$ (a function of only $a,c$) such that for all $n > n_0$, the causal direction is identifiable with high probability.

\section{\fontsize{11}{15}\selectfont Proof of Theorem \ref{thm:constant_entropy_away}}
\label{app:constant_entropy_away}
Given the random variables $U_i, i\in [n]$ with marginal distributions $\mat{p_i}(u_i)$, let $p(u_1,u_2,\hdots,u_n)$ be a valid coupling. Then $p$ satisfies $\mat{p_i}(u_i)=\sum_{k\neq i}\sum_{u_k\in [n]}p(u_1,u_2,\hdots,u_n)$ holds for all $i,u_i$. Therefore, for all $i,u_i$, we can define $S_{i,u_i}=\{(u_j)_{j\neq i}:p(u_1,u_2,\hdots u_n)>0\}$. $S_{i,u_i}$ contains the coordinates in the coupling that contribute non-zero mass to satisfy the $i^{th}$ marginal distribution, specifically the probability that variable $U_i$ takes the value $u_i$. Let us define the function $g_{i,u_i}((u_j)_{j\neq i})\coloneqq p(u_1,\hdots,u_n)$. Then equivalently, we can write $\mat{p_i}(u_i) = \sum_{t\in S_{i,u_i}}g_{i,u_i}(t)$.

Consider a noisy version of the marginal distributions: Let $\hat{\mat{p}}_{\mat{i}}$ be the noisy marginals where $\lvert\hat{\mat{p}}_{\mat{i}}(u_i)-\mat{p}_{\mat{i}}(u_i)\rvert\leq \delta$ for all $i,u_i$. Our strategy is to start with the coupling $p(u_1,\hdots,u_n)$ and convert it to a coupling for the noisy marginals. Let us define $T_{i}^+(p)\coloneqq \{u_i: \sum_{k\neq i}\sum_{u_k\in [n]}p(u_1,u_2,\hdots,u_n) <\hat{\mat{p}}_{\mat{i}}(u_i) \},T_{i}^-(p)\coloneqq \{u_i:\sum_{k\neq i}\sum_{u_k\in [n]}p(u_1,u_2,\hdots,u_n) >\hat{\mat{p}}_{\mat{i}}(u_i) \}$. In words, $T_i^+(p)$ shows the coordinates of the $i^{th}$ noisy marginal which has excess mass compared to the mass induced by coupling $p$. Similarly, $T_i^-(p)$ shows the coordinates of the $i^{th}$ noisy marginal for which the coupling $p$ has more mass than needed. We update $p$ in two stages: First, we update $p$ so that $T_i^-(p)=\emptyset$. In the second stage, we further update $p$ so that $T_i^+(p)=\emptyset$ and $T_i^-(p)=\emptyset$, which shows that the updated $p$ is a valid coupling for the noisy marginals $\hat{\mat{p}}_{\mat{i}}$. We finally bound the entropy of the new coupling relative to the initial coupling we started with.

First we observe the following: Consider any $u_i\in T_i^-$. Then there exists a function $h_{i,u_i}(t) $ such that  
\begin{align}
&0\leq h_{i,u_i}(t)\leq g_{i,u_i}(t), \forall t\in S_{i,u_i},\\
&\sum_{t\in S_{i,u_i}}h_{i,u_i}(t)=\hat{\mat{p_i}}(u_i).
\end{align}
This is true since $\sum_{t\in S_{i,u_i}}g_{i,u_i}(t)=\mat{p_i}(u_i)$ and $\hat{\mat{p}}_{\mat{i}}(u_i)<\mat{p_i}(u_i),\forall u_i \in T_i$. We can describe the first phase as follows: For each $i\in[n]$ and $u_i\in T_i^-$, we pick an arbitrary $h_{i,u_i}$ and update $p$ to match the entries of $h_{i,u_i}$. Notice that each update of $p$ changes the corresponding $h,g$ functions. Our construction proceeds by updating these functions every time $p$ is updated as given above. This procedure is summarized in Algorithm \ref{first_phase}.
\begin{algorithm}[t]
\caption{Phase I}
\begin{algorithmic}
\STATE \textbf{Input:} Valid coupling $p_{\mathrm{init}}$ for the marginals $\{\mat{p_i}\}_{i\in[n]}$. Noisy marginals $\{\hat{\mat{p}}_{\mat{i}}\}$
\STATE $p\leftarrow p_{\mathrm{init}}$.
\STATE Construct $g_{i,u_i},S_{i,u_i},T_i^+,T_i^-$ from $p_{\mathrm{init}}$ for all $i,u_i$.
\WHILE{$\exists i\in[n]$ s.t. $T_i^-\neq \emptyset$}
    \STATE Pick arbitrary $h_{i,u_i}$ for all $u_i$ such that 
    \begin{align*}
        &0\leq h_{i,u_i}(t)\leq g_{i,u_i}(t), \forall t\in S_{i,u_i},\\
        &\sum_{t\in S_{i,u_i}}h_{i,u_i}(t)=\hat{\mat{p}}_{\mat{i}}(u_i).
    \end{align*}
    \STATE Update $p$ as follows:
    \begin{equation}
    \label{eq:update}
        p(u_1,u_2,\hdots,u_n) \leftarrow h_{i,u_i}((u_j)_{j\neq i}), \forall (u_j)_{j\neq i}\in S_{i,u_i}
    \end{equation}
    \STATE Construct $g_{i,u_i},S_{i,u_i},T_i^+,T_i^-$ from $p$ for all $i,u_i$.
\ENDWHILE
\STATE return $p$
\end{algorithmic}
\label{first_phase}
\end{algorithm}

Note that the size of $T_i^-$ after an update is at least one less than the size of $T_i^-$ before the update. To see this, note that after the update in (\ref{eq:update}), $u_i\notin T_i^- $. Also by reducing elements of $p$, we can never add a new element to $T_i^-$ for any $i$ by definition of $T_i^-$. Therefore, after at most $\sum_i\lvert T_i^- \rvert$ applications of the above update for the initial sets $T_i^-$, we have $T_i^-=\emptyset, \forall i\in[n]$. Since there are at most $n$ elements in $T_i^-$ and $n$ such sets, the first phase terminates in at most $n^2$ steps. 

Let $p$ be the output of Algorithm \ref{first_phase} in the rest of the proof. In the second phase, we consider the updated $T_i^+$. Our strategy here is to distribute the remaining mass in each marginal as its own coupling and add this coupling to $p$ that is the output of Algorithm \ref{first_phase}. Let us represent the excess probability mass in coordinate $u_i$ of marginal $i$ relative to coupling $p$ by $r_{i,u_i}$. Note that $r_{i,u_i}(p) \coloneqq \hat{\mat{p}}_{\mat{i}}(u_i)-\sum_{k\neq i}\sum_{u_k\in [n]}p(u_1,u_2,\hdots,u_n)$ may increase at each step of the first phase. The exact increase in this gap for each $i,u_i$ depends on the choice of $h_{i,u_i}$ function at each step. However, we can bound the total gap per marginal at the end of first phase as $\sum_{u_i\in[n]}r_{i,u_i}(p)\leq \delta n^2,\forall i $. Each step of Algorithm \ref{first_phase} can add a mass of at most $\delta$ to each marginal at each step (it terminates after at most $\sum_i |T_i^-|$ steps) and at the beginning of first phase, each coordinate of each marginal has at most $\delta$ excess mass (there are $\sum_i |T_i^+|$ coordinates with excess mass). As such, there is at most $\sum_i \delta |T_i^-| + \sum_i \delta |T_i^+| \le \delta n^2$ total gap per marginal at the end of the first phase. Let $p(u_1,\hdots,u_n)$ be the output of Algorithm \ref{first_phase}. \cite{Kocaoglu2017} showed a greedy minimum entropy coupling algorithm that produces a coupling with support at most $n^2$. Let $q(u_1,u_2\hdots, u_n)$ be the output of this greedy algorithm when given the excess marginal mass as its input. Then we have that $v\coloneqq p+q$ is a valid coupling for the noisy marginals. This is because, by feeding the greedy algorithm the excess marginal mass, we guarantee that the marginals of $v$ are correct. Moreover, all cells in the coupling are in range $[0,1]$ as no cell in $p$ or $q$ has negative value and their sum has the correct marginals.

Next, define the distribution $s:2\times [n]^{n}\rightarrow [0,1]$ as follows: \begin{align}
    s(0,u_1,u_2,\hdots,u_n) &= p(u_1,u_2,\hdots,u_n),\\ s(1,u_1,u_2,\hdots,u_n) &= q(u_1,u_2\hdots, u_n). 
\end{align}
From the argument above, it is easy to see that $s$ is a valid probability distribution, i.e., it has non-negative entries and its entries sum to $1$.

We compare entropy of the obtained coupling $v$ with entropy of $s$ and that with entropy of the initial coupling $p_{\mathrm{init}}$. First, it is easy to see from concavity of entropy and Jensen's inequality that $H(v)\leq H(s)$. Let $\bar{H}$ be the extended entropy operator that admits vectors outside the simplex as input, for vectors whose entries are between $0$ and $1$: $\bar{H}(p(x))=-\sum_x p(x)\log(p(x))$. We have the following lemma that allows us to compare $\bar{H}(p)$ with $H(p_{\mathrm{init}})$:
\begin{lemma}
\label{lem:entropy_reduction}
Let $\mat{p}=[p_1,p_2,\hdots,p_n]$ be a discrete probability distribution. Let $\mat{q}=[q_1,q_2,\hdots,q_n]$ be a non-negative vector such that $q_i\leq p_i,\forall i \in [n]$. Then $\bar{H}(\mat{q})\leq \bar{H}(\mat{p})+\frac{\log(e)}{e}$.
\end{lemma}
The proof is in Section \ref{app:entropy_reduction} in the supplement.

From the lemma, we can conclude that $\bar{H}(p)\leq H(p_{\mathrm{init}})+\frac{\log(e)}{e}$. Finally, the maximum entropy contribution of $q$ is when it induces uniform distribution over $n^2$ states. Since the total mass of $q$ is $\delta n^2$, we have 
\begin{align}
    \bar{H}(q)&\leq n^2 \left(\frac{\delta n^2}{n^2} \log{
    \left(\frac{n^2}{\delta n^2}\right)}\right)\\
    &=\delta n^2 \log\left(\frac{1}{\delta}\right)
\end{align}
Suppose $\delta\leq \frac{1}{n^2\log(n)}$. Then we can further bound $\bar{H}(q)\leq 2+ \frac{\log{(\log{(n)})}}{\log{(n)}}\leq 2+o(1)$ since $\delta \log\left(\frac{1}{\delta}\right)\leq \frac{2\log(n)+\log(\log(n))}{n^2\log(n)}$ if $\delta<\frac{1}{n^2\log(n)}$.

Bringing it all together, we obtain the following chain of inequalities:
\begin{align}
    H(v)&\leq H(s) = \bar{H}(p)+\bar{H}(q)\\
    &\leq H(p_{\mathrm{init}})+\frac{\log(e)}{e}+2+o(1).
\end{align}
This concludes the proof. \hfill \qed

\section{\fontsize{11}{15}\selectfont Proof of Lemma \ref{lem:entropy_reduction}}
\label{app:entropy_reduction}
If $p_i<\frac{1}{\exp(1)}, \forall i$, due to monotonicity of $-p\log(p)$ in $p$, we have $\bar{H}(\mat{q})\leq \bar{H}(\mat{p})$. 

In general, no more than $2$ states can satisfy $p_i>\frac{1}{\exp(1)}$. Therefore, $\bar{H}(q)$ can only be larger than $\bar{H}(p)$ due to two states. Let us call these two states $p_1,p_2$ without loss of generality. Reducing the probability of any other state only gives a looser bound. 

We can obtain the largest entropy increase by solving the following optimization problem:
\begin{equation}
\label{eq:entropy_gap}
\begin{aligned}
&\underset{p_1,p_2}{\max} & & \mathbbm{1}_{\{p_1>1/e\}}\left(\frac{\log(e)}{e}-p_1\log\left(\frac{1}{p_1}\right)\right)\\
&&&\hspace{0.2in}+\mathbbm{1}_{\{p_2>1/e\}}\left(\frac{\log(e)}{e}-p_2\log\left(\frac{1}{p_2}\right)\right)\\
& \text{subject to}
& & p_1+p_2 \leq 1,  \\
& && p_1\geq 0,p_2\geq 0
\end{aligned}
\end{equation}
Suppose $p_1>1/e$ and $p_2<1/e$. Then the solution is simply to set $p_1=1$ since this minimizes the entropy contribution of $p_1$. This gives a gap of $\frac{\log(e)}{e}$. Due to symmetry, we only need to investigate the case where $p_1>1/e$ and $p_2>1/e$. In this case, we have the following optimization problem:
\begin{equation}
\label{eq:entropy_gap2}
\begin{aligned}
&\underset{p_1,p_2}{\min} & & p_1\log\left(\frac{1}{p_1}\right)+p_2\log\left(\frac{1}{p_2}\right)\\
& \text{subject to}
& & p_1+p_2 \leq 1,  \\
& && p_1\geq 1/e,p_2\geq 1/e
\end{aligned}
\end{equation}
This is a concave minimization problem and the solution has to be at the boundary of the convex constraint region. If $p_1 = 1/e$, the maximum gap is obtained when $p_2$ is maximized to $p_2 = 1-1/e$ which gives a gap that is strictly less than $\frac{\log(e)}{e}$, hence we can discard this solution for the maximum entropy gap. $p_2=1/e$ gives the same solution from symmetry. When $p_1+p_2 = 1$, the problem reduces to minimizing the binary entropy function, which again is minimized at the boundary. The boundary in this case is where either $p_1=1/e$ or $p_2=1/e$. Therefore, both probabilities being greater than $1/e$ cannot yield a better bound. \hfill\qed

\section{\fontsize{11}{15}\selectfont Proof of Lemma \ref{lem:concentration}}\label{app:concentration}
\textbf{Joint Probabilities}. First, we bound the estimates of the entries of the joint distribution between $X$ and $Y$. Both $X$ and $Y$ have $n$ states which we index as $i = 1,\dots, n$ and $j = 1,\dots, n$ respectively. Hence the joint distribution has $n^2$ states. Probability that $X = i$ and $Y = j$ is shown as $p_{ij}$. Suppose $N$ samples from $N$ independent, identically distributed random variables are drawn as $\{(x_k,y_k)\}_{k\in[N]}$. This yields the empirical probability estimates ($I$ is the indicator function)
\[
\hat{p}_{ij} = \frac{1}{N} \sum_{k = 1}^N I(x_k = i \: \& \: y_k = j).
\]
Note that each of these estimates are averages of Bernoulli random variables with success probability $p_{ij}$. We also consider the marginal probability empirical estimates
\[
\hat{p}_{i}^X = \frac{1}{N} \sum_{k = 1}^N I(x_k = i )
\]
and
\[
\hat{p}_{j}^Y = \frac{1}{N} \sum_{k = 1}^N I( y_k = j).
\]
which are also averages of $N$ Bernoulli random variables (with success probabilities $p_i^X$ and $p_j^Y$ respectively). 

Since these estimates are clearly correlated with one another, our approach will be to use concentration results on individual entries of the joint distribution and then do a union bound over all $n^2 + 2n$ probabilities.
Note that $I(x_k = i \: \& \: y_k = j) = 1$ with probability $p_{ij}$ and 0 otherwise. Thus by Hoeffding's inequality \cite{vershynin2018high},
\[
\mathbb{P}\left\{|\hat{p}_{ij} - p_{ij}| \geq t\right\} \leq 2\exp \left(- 2t^2 N\right).\numberthis \label{eq:hoeffding}
\]
We can define an event $\mathcal{A}$ where all the probability estimates are within $t$ of the truth:
\begin{align*}
\mathcal{A} =& \left\{\max_{i,j \in 1,\dots,n}|\hat{p}_{ij} - p_{ij}| \leq t  \right\}\bigcap \left\{\max_{i \in 1,\dots,n}|\hat{p}_{i}^X - p_{i}^X| \leq t  \right\}\bigcap \left\{\max_{j \in 1,\dots,n}|\hat{p}_{j}^Y - p_{j}^Y| \leq t  \right\}.
\end{align*}
Starting with \eqref{eq:hoeffding} and taking the union bound over all $n^2 + 2n$ probabilities in the joint and marginal distribution, we obtain
\begin{align}\label{eq:joints}
\mathbb{P}(\mathcal{A}) &> 1- 2(n^2+2n)\exp \left(- 2t^2 N\right) \\\nonumber &> 1- 4\exp(2 \ln(n) -2t^2N).
\end{align}

\textbf{Conditional Probabilities}. Given the above bound on the estimates of the joint probabilities, we formulate bounds on the conditional probability estimates. Recall that
\[
P(X = i | Y = j) = \frac{P(X=i,Y=j)}{P(Y=j)} = \frac{p_{ij}}{\sum_{i=1}^n p_{ij}}.
\]
Using the plug-in approach, we have
\[
\hat{p}_{i|j} = \frac{\hat{p}_{ij}}{\hat{p}_j^Y}.
\]
Note that it is critical for $\hat{p}_j^Y$ to be bounded away from zero, otherwise a small error in $\hat{p}_{ij}$ may cause a large error in $\hat{p}_{i|j}$. 
In what follows, we set  
\[
\alpha = \frac{\min_{j = 1,\dots, n} p^Y_j}{2}.  
\]
$\alpha$ will naturally appear in the number of samples, and notably must depend on $n$. Note that the case of $\sum_{i=1}^n {p}_{ij} = 0$ is allowable since if that is the case $Y = j$ will never occur and corresponding probability estimates will all be zero and the conditional probabilities will not be of interest.

Now consider any $t < \alpha$, assume that event $\mathcal{A}$ holds. We then have that all $\hat{p}^Y_j > p^Y_j - t > 2\alpha - t > \alpha$. Combined with the fact that under event $\mathcal{A}$, $|\hat{p}_{ij} - p_{ij}| < t$ and $t \geq 0$, it is easy to check that
\begin{align*}
\hat{p}_{i|j} - p_{i|j} &= \frac{\hat{p}_{ij}}{\hat{p}_j^Y} - \frac{{p}_{ij}}{{p}_j^Y} \\
&< \frac{{p}_{ij}+t}{{p}_j^Y-t} - \frac{{p}_{ij}}{{p}_j^Y} \\
&= \frac{p_{ij}p_j^Y + t p_j^Y - p_{ij}p_j^Y + tp_{ij}}{p_j^Y(p_j^Y - t)}\\
&< \frac{t p_j^Y + tp_{ij}}{p_j^Y \alpha}\\
&< \frac{2t}{\alpha},
\end{align*}
where the last inequality follows since $p_{ij} < p^Y_j$ by definition. Similarly, 
\begin{align*}
p_{i|j} - \hat{p}_{i|j}  &=  \frac{{p}_{ij}}{{p}_j^Y} - \frac{\hat{p}_{ij}}{\hat{p}_j^Y} \\
&<  \frac{{p}_{ij}}{{p}_j^Y} - \frac{{p}_{ij}-t}{{p}_j^Y+t} \\
&= \frac{p_{ij}p_j^Y + t p_{ij} - p_{ij}p_j^Y + tp_{j}^Y}{p_j^Y(p_j^Y + t)}\\
&< \frac{t p_j^Y + tp_{ij}}{p_j^Y 2\alpha}\\
&< \frac{t}{\alpha},
\end{align*}
hence
\[
|\hat{p}_{i|j} - p_{i|j}| < \frac{2t}{\alpha}.
\]

Since by \eqref{eq:joints} the event $\mathcal{A}$ holds with probability at least $1 - 4\exp(2 \log (n) -2t^2N)$, we have
\begin{equation}
\mathbb{P}\left(\max_{i,j \in 1,\dots,n}|\hat{p}_{i|j} - p_{i|j}| \geq \frac{2t}{\alpha}\right) \leq 4\exp(2 \ln{(n)} -2t^2N).
\label{eq:highProb}
\end{equation}

The derivation of the bound for the conditional probability estimates in the other direction is similar and relies on the same event $\mathcal{A}$ holding. Hence the probability the bounds hold in both directions simultaneously remains $\mathbb{P}(\mathcal{A})$.

\textbf{Achieving error of $\delta = 1/(n^2 \ln{(n)})$}. Let $\alpha = \frac{\min\{\min_xp(x),\min_yp(y)\}}{2}$. Suppose we want $2t/\alpha = 1/(n^2 \ln{(n)})$. Then we need $t = 1/(2n^2 \alpha^{-1} \ln{(n)})$. Note that $t<\alpha$ as required above. Suppose further that we want this to hold with probability at least $1 - 4/n$. By the above, we require
\begin{align*}
2 \ln(n) -2t^2N &< -\ln(n) \\
3 \ln(n) &< \frac{2N}{4 n^{4}\alpha^{-2} \ln^2 (n)}\\
6  n^{4} \alpha^{-2} \ln^3 (n)  &< N
\end{align*}
Hence $N$ needs to be $\Omega(n^{4}\alpha^{-2} \ln^3 (n))$. \hfill \qed

\section{\fontsize{11}{15}\selectfont Proof of Theorem \ref{thm:finite_samples}}
\label{app:finite_samples}
From the equivalence between the minimum entropy coupling problem and the problem of finding the exogenous variable with minimum entropy, the output of $\mathcal{A}(\{\hat{p}(Y|X=x)\}_x)$ is the smallest entropy of any exogenous variable for the causal model $X\rightarrow Y$. Similarly, this claim holds for $\mathcal{A}(\{\hat{p}(X|Y=y)\}_y)$ as well. From Theorem \ref{thm:main}, entropy in the direction $Y\rightarrow X$ scales with $n$ using $p(X|Y=y)$. From Theorem 6 of \cite{ho2010interplay}, it can be seen that the given sampling error can induce an entropy difference of at most $o(1)$ in the conditional entropies. Hence, even with noisy conditionals, $\max_y\hat{H}(X|Y=y)$ scales with $n$, implying that $\mathcal{A}(\{\hat{p}(X|Y=y)\}_y)$ scales with $n$. In the forward direction, the true exogenous variable provides a valid coupling under the true joint distribution without sampling noise. From Lemma \ref{thm:constant_entropy_away}, given $N$ samples, there exists a valid coupling in the forward direction that is constant entropy away from the true exogenous variable. Hence $\mathcal{A}(\{p(Y|X=x)\}_x)$ is constant. Since $\mathcal{A}(\{p(X|Y=y)\}_y)$ scales with $n$, the result follows.

\section{\fontsize{11}{15}\selectfont Proof of Theorem \ref{thm:finite_samples_conditionals}}
\label{app:finite_samples_conditionals}
We first show that the $H(X|Y=2)$ conditional entropy will have enough samples to be included in the criterion listed in Theorem \ref{thm:finite_samples_conditionals}. As $N = \Omega(n^2 \log(n))$, we have at least $c_1 n^2 \log(n)$ samples for $c_1 = \Theta(1)$. As shown in the proof of Theorem \ref{thm:main}, $p(Y=2) = \Omega(\frac{1}{n}) \ge \frac{c_4}{n}$ where $c_4 = \Theta(1)$. Following a rejection sampling approach, we use Hoeffding's inequality to show that if $c_1 n^2 \log n$ samples are drawn from the joint distribution, then with probability $1 - o(1)$ we will successfully draw $\Omega(n \log{(n)})$ independent samples from the distribution $p(X|Y=2)$. Specifically, let $S_n$ denote the number of samples (out of $c_1 n^2 \log(n)$ total samples from the joint distribution) for which $Y=2$, and $E_n = \mathbb{E}[S_n]$ denote the expected number of such samples. We have $E_n \ge (c_1 n^2 \log(n)) (\frac{c_4}{n}) = c_1 c_4 n \log(n)$. Hence using Hoeffding's inequality, $P\left(S_n < \frac{c_1 c_4 n \log(n)}{2}\right) \le P\left(|S_n - E_n| > \frac{c_1 c_4 n \log(n)}{2}\right) < 2e^{-\frac{2 (c_1 c_4 n \log(n))^2}{c_1 n^2 \log(n)}} = 2e^{-2 c_1 c_4^2 \log(n)} = o(1)$. %
Hence $S_n \geq \frac{c_1 c_4 n \log(n)}{2} \gg n$ with probability $1-o(1)$. Thus the $\hat{H}(X|Y=2)$, which we use for identifiability, will have sufficient number of samples to be included in the criterion in Theorem \ref{thm:finite_samples_conditionals}.

We now show that each conditional entropy in the criterion in Theorem \ref{thm:finite_samples_conditionals} will have error bounded by a constant with high probability. Immediately following from Corollary 1.12 of \cite{valiant2017estimating}, for a distribution $D$ with support size $n$, $|H(D)-\hat{H}(D)| \le 1$ with probability $1-e^{-n^{c_2}}$ given a sample of size at least $\frac{c_3 n}{\log(n)}$ where $c_2,c_3 = \Theta(1)$.
Since we only calculate conditional entropy estimates with $\ge n$ samples, the number of samples $n \gg \frac{c_3 n}{\log(n)}$ for all considered conditional entropies. Hence the total probability of any computed conditional entropy estimate being off by more than $1$ is $\le n e^{-n^{c_2}} = o(1)$ by the union bound. Since by the proof of Theorem \ref{thm:main} we know $\max_{x} H(X|Y=y) \le c \ll \Omega(\log(\log(n))) \le H(X | Y=2)$, it immediately follows that $\max_{x,\hat{p}(X=x)N\ge n} \hat{H}(Y|X=x) \le c+1 \ll \Omega(\log(\log(n)))-1 \le \max_{y,\hat{p}(Y=y)N \ge n} \hat{H}(X|Y=y)$.\hfill\qed

\section{\fontsize{11}{15}\selectfont Proof of Corollary \ref{cor:finite_simplex}}
\label{app:finite_simplex}
This generative model satisfies the assumptions of Theorem \ref{thm:finite_samples} following from the proof of Corollary \ref{cor:identifiability}. As such, under this generative model for sufficiently large $n$ and $N = \Omega(n^4 \alpha^{-2} \log^3{(n)})$ samples, $\mathcal{A}(\{\hat{p}(X|Y=y)\}_y)>\mathcal{A}(\{\hat{p}(Y|X=x)\}_x)$ with high probability.

We show a lower bound on $\alpha$ with high probability, under this generative model. As mentioned in the proof of Corollary \ref{cor:identifiability}, under this generative model, for any $i$, $P(x_i \le z) = 1 - (1-z)^{n-1}$. We aim to show that with high probability, $x_i \ge \frac{1}{n^2 \log(n)}, \forall i \in [n]$ when $n$ is sufficiently large. 

We lower bound the probability of this not happening as $(1 - (1-\frac{1}{n^2 \log(n)})^{n-1})n$ by the union bound. Note that $\lim_{n \rightarrow \infty} \frac{(1 - (1-\frac{1}{n^2 \log(n)})^{n-1})n}{1 / \log(n)} = 1$. 

Hence for sufficiently large $n$ the probability that there exists an $x_i < \frac{1}{n^2 \log(n)}$ is upper bounded by $\frac{2}{\log(n)}$. Thus, we have a high probability lower bound for $\alpha$. We substitute this for $\alpha$ in our lower bound for the number of required samples in the previous paragraph. This yields that under this generative model for sufficiently large $n$ and $N = \Omega(n^8 \log^5{(n)})$ samples, $\mathcal{A}(\{\hat{p}(X|Y=y)\}_y)>\mathcal{A}(\{\hat{p}(Y|X=x)\}_x)$ with high probability.

\section{\fontsize{11}{15}\selectfont Proof of Negative Association}
\label{app:negAssoc}
\begin{lemma}
\label{lem:neg_assoc}
Let $[x_i]_{i\in [n]}$ be a vector, uniformly randomly sampled from the probability simplex in $n$ dimensions. Then $[x_i]_{i\in [n]}$ is negatively associated. 
\end{lemma}
\begin{proof}
Let $x_i=\frac{z_i}{\sum_jz_j},$ where each $z_i$ is independent and identically distributed exponential random variable with mean $1$, i.e. distributed as $\mathrm{Exp}(1)$. Then $[x_i]_i$ is a discrete probability distribution uniformly randomly chosen from the simplex in $n$ dimensions. We will show that $x_i$ are negatively associated. The following argument is provided by \cite{297441} as an answer on the online forum \url{https://mathoverflow.net/}, which we reproduce here for completeness.

Consider the following theorem:
\begin{theorem}\cite{joag1983negative}
Let $z_1,z_2,\hdots, z_n$ be $n$ random variables with log-concave probability densities. Then $(z_1,z_2,\hdots,z_n)$ conditioned on $\sum_{i\in [n]}z_i$ are negatively associated. 
\end{theorem}

Note that exponential distribution is log-concave. Hence the theorem is applicable in our setting. Furthermore, the distribution induced on $(\frac{z_i}{\sum_{j\in [n]}z_j})_{i\in[n]}$ is identical to the distribution induced on $(z_1,z_2,\hdots,z_n)$ conditioned on $\sum_{i\in [n]}z_i=1$. This concludes the proof.
\end{proof}

\section{Additional Experiments and Experimental Details}
\subsection{Experimental Details}
In this section, we provide the complete details of every experiment given in the main text, as well as provide additional results that we were not able to present in the main text due to space constraints. 

\textbf{Sampling low-entropy exogenous variables: }
We use Dirichlet distribution to sample the distribution for the exogenous variable from the probability simplex. Dirichlet has the parameter $\alpha$ which affects the entropy of the distribution obtained by sampling the corresponding Dirichlet distribution: Smaller $\alpha$ values lead to sampling distributions with smaller entropy. Suppose we want to sample distributions for $E$ such that $H(E)\leq \theta$. Since a good $\alpha$ value for this $\theta$ is not known a priori, we use the following adaptive sampling scheme: Suppose we want to sample $N$ distributions for $E$ such that $H(E)\leq \theta$.  We initialize with $\alpha^{(0)}=1$ and obtain $10N$ samples from Dirichlet with parameters $\alpha^{(0)}$. If there are at least $N$ samples out of $10N$ which has entropy less than $\theta$, we are done. If not, we set $\alpha^{(1)}=0.5\alpha^{(0)}$ and iterate until for a particular $\alpha^{(i)}$ such that at least $N$ out of $10N$ samples satisfy the entropy condition. 

\textbf{Details about Figure \ref{fig:implications_on_observed}: }
We set $E$ to have $mn$ number of states where $m,n$ are the number of states of $X$ and $Y$, respectively. It can be shown that this many number of states is sufficient to obtain any joint distribution. We uniformly randomly sample the function $f$ in the structural equation $Y=f(X,E)$. We also independently and uniformly randomly sample $p(X)$ from the simplex, i.e., we obtain samples from Dirichlet distribution with parameter $\alpha=1$. 
For $m=n=40$, we choose $20$ values of $\theta$, i.e., entropy thresholds for the exogenous variable $E$, uniformly spaced in the range $[0,\log(m)]$. For $m\neq n$, we choose $10$ $\theta$ values in the range $[0,\log(\max\{m,n\})]$.  

When $m\neq n$, we use a mixture data as follows: We obtain $10000$ samples from the graph $X\rightarrow Y$ and we obtain $10000$ samples from $X\leftarrow Y$. We operate on this mixed data. This is done to reflect the fact that, there is no reason for the cause or the effect variable to have less or more number of states. Accuracy shown in the figures reflect the fraction of times each algorithm correctly identifies the true causal direction. Total entropy-based compares $H(X)+H(E)$ and $H(Y)+H(\tilde{E})$ where $E$ and $\tilde{E}$ are the outputs of the greedy minimum entropy coupling algorithm in the direction $X\rightarrow Y$ and $X\leftarrow Y$, respectively.

\textbf{Details about Figure \ref{fig:relaxing0.8}: }
We sample exogenous variable using the above adaptive sampling method so that, for each value of $n$, we have $H(E)\leq 0.8\log(n)$. The other details are identical (e.g., $10000$ samples for each configuration.) Due to the sampling method, we observe that most of the samples are very close to $H(E)\approx 0.8\log(n)$. We then obtain the histogram plots for $H(\tilde{E})$, where $\tilde{E}$ is the output of the greedy minimum entropy coupling algorithm in the wrong direction. As observed, data fits well to a Gaussian and is highly concentrated around $0.854\log(n)$.

\textbf{Details about Figure \ref{fig:effect_of_confounding}: }
In this section, we introduce a latent confounder $L$. First, distribution of $L$ and distribution of $E$ are sampled independently. Then the distributions $p(X|l),p(Y|x,l,e)$ are sampled uniformly randomly from the simplex for every configuration of $x,l,e$. We use the adaptive sampling described above to sample $E$ such that $H(E)\leq 2$. Using the same sampling method, we sweep through different entropy thresholds for the latent confounder $L$ and sample such that $H(L)\leq \phi$ for $\phi\in \{0.5,1,1.5,2,2.5,3\}$. The settings for $m,n$ and how data is mixed is identical to the procedure used to obtain Figure \ref{fig:implications_on_observed}: When $m\neq n$, we use uniformly mixed data from $X\rightarrow Y$ and $X\leftarrow Y$. For each configuration, we obtain $1000$ total number of samples and report the accuracy of the method to identify the true causal direction. 

\subsection{Relaxing constant exogenous entropy assumption}
As indicated in Section \ref{sec:experiments}, we provide additional experiments for $\alpha=0.2$ and $0.5$ in Figure \ref{fig:relaxing0.2} and Figure \ref{fig:relaxing0.5}, respectively. As can be seen, for both $\alpha$ values, i.e., when $H(E)\leq \alpha\log(n)$, $H(\tilde{E})$ highly concentrates around $\beta\log(n)$ for some $\beta>\alpha$.

\begin{figure}
	\centering
	\begin{subfigure}[b]{0.32\linewidth}
		\includegraphics[width=\textwidth]{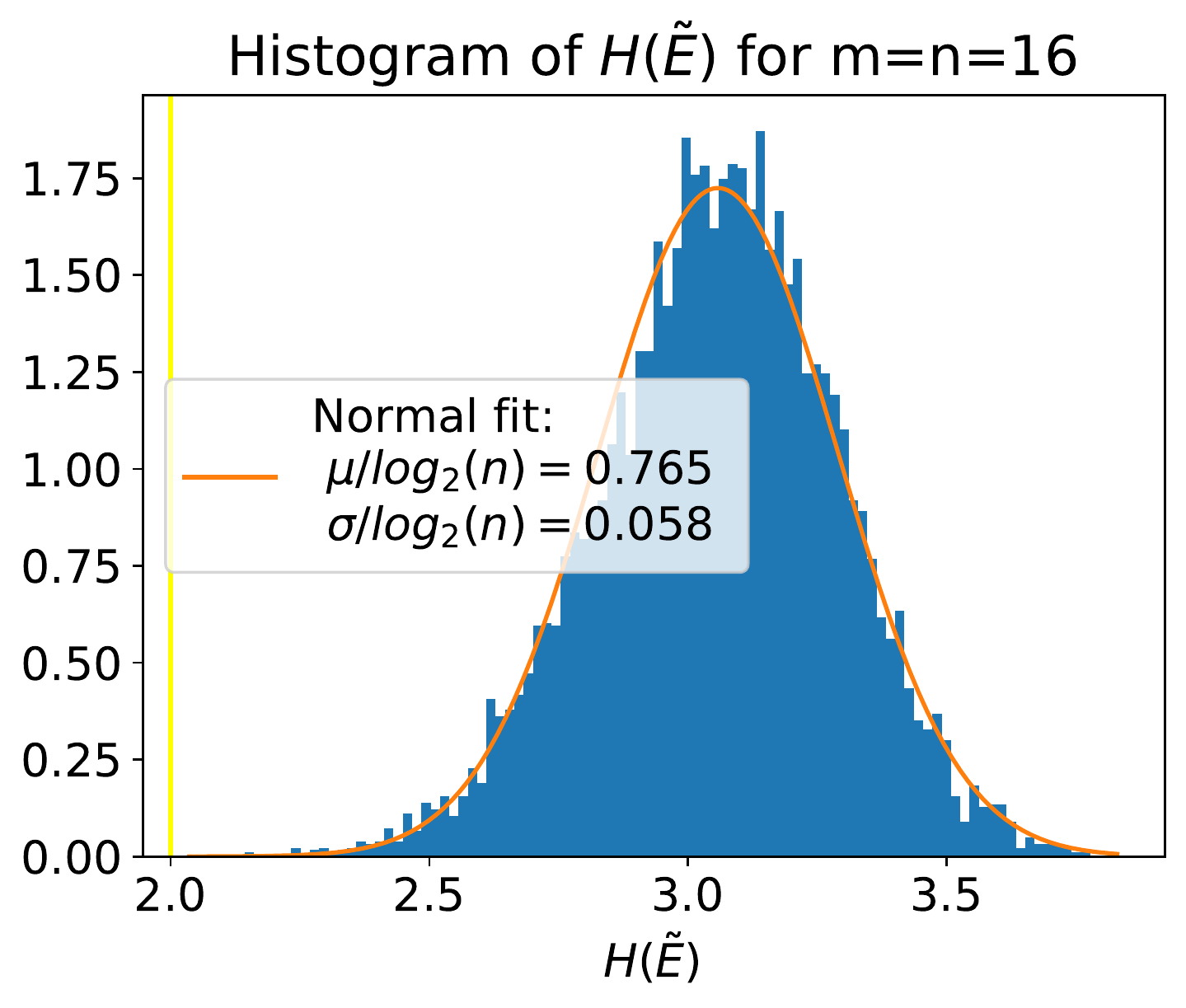}
		\caption{}
		\label{fig:normal_fit_16_0.5}
	\end{subfigure}
	\begin{subfigure}[b]{0.3\linewidth}
		\includegraphics[width=\textwidth]{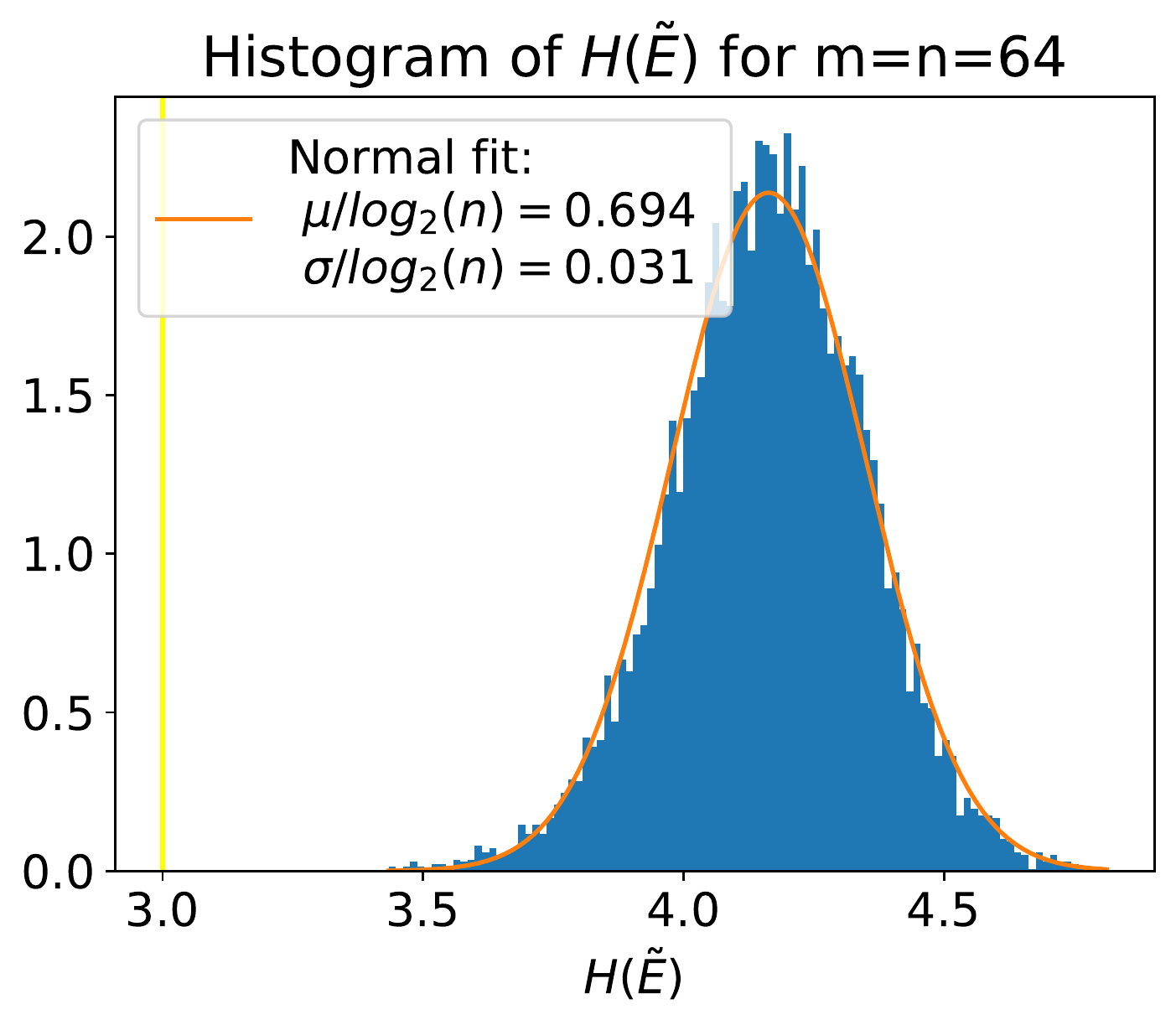}
		\caption{}
		\label{fig:normal_fit_64_0.5}
	\end{subfigure}
	\begin{subfigure}[b]{0.32\linewidth}
		\includegraphics[width=\textwidth]{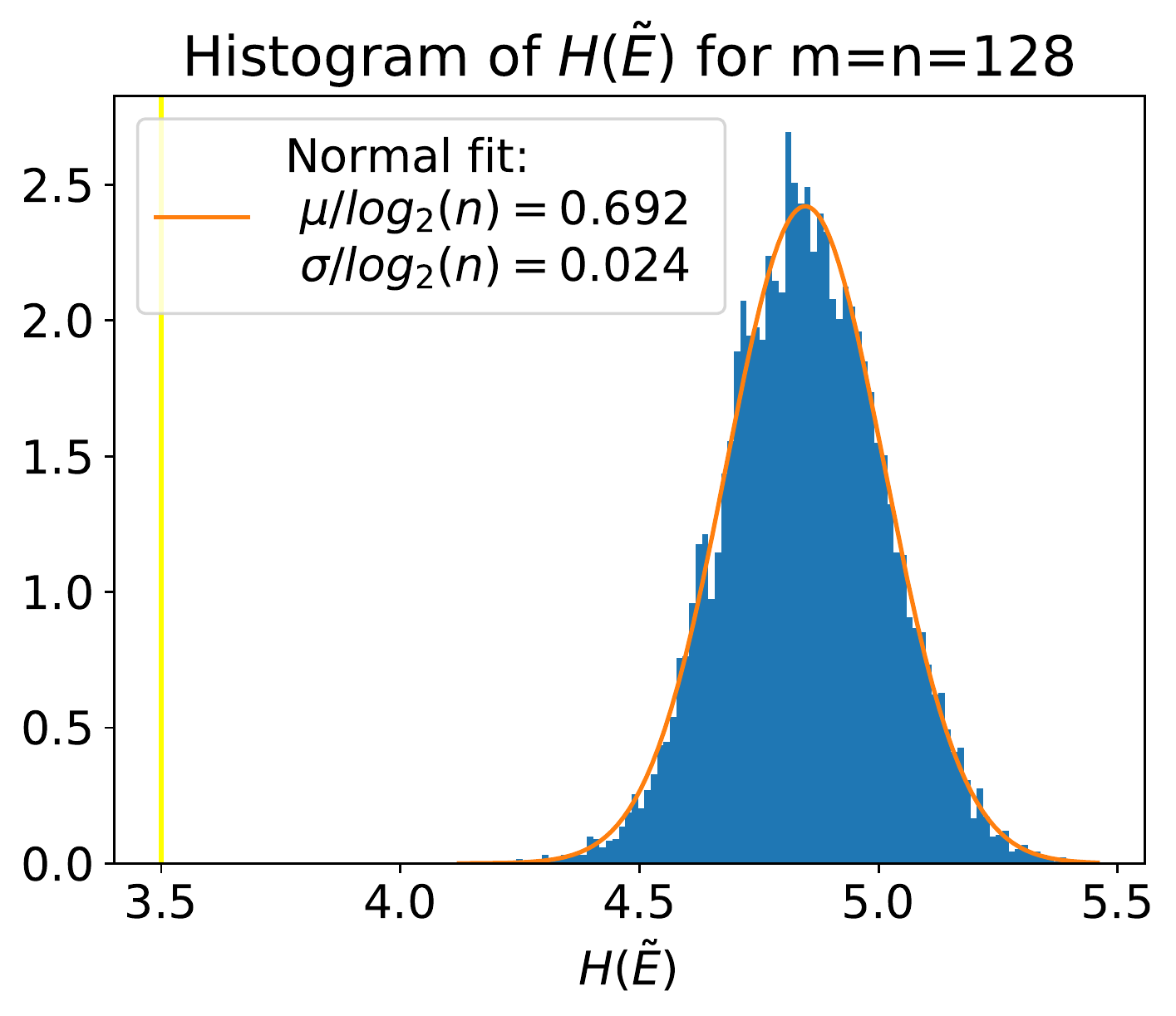}
		\caption{}
		\label{fig:normal_fit_128_0.5}
	\end{subfigure}	
\caption{Histogram of $H(\tilde{E})$ when $H(E)\approx 0.5\log_2(n)$.  Yellow line shows $x=0.5\log_2(n)$}
\label{fig:relaxing0.5}
\end{figure}

\begin{figure}
	\centering
	\begin{subfigure}[b]{0.32\linewidth}
		\includegraphics[width=\textwidth]{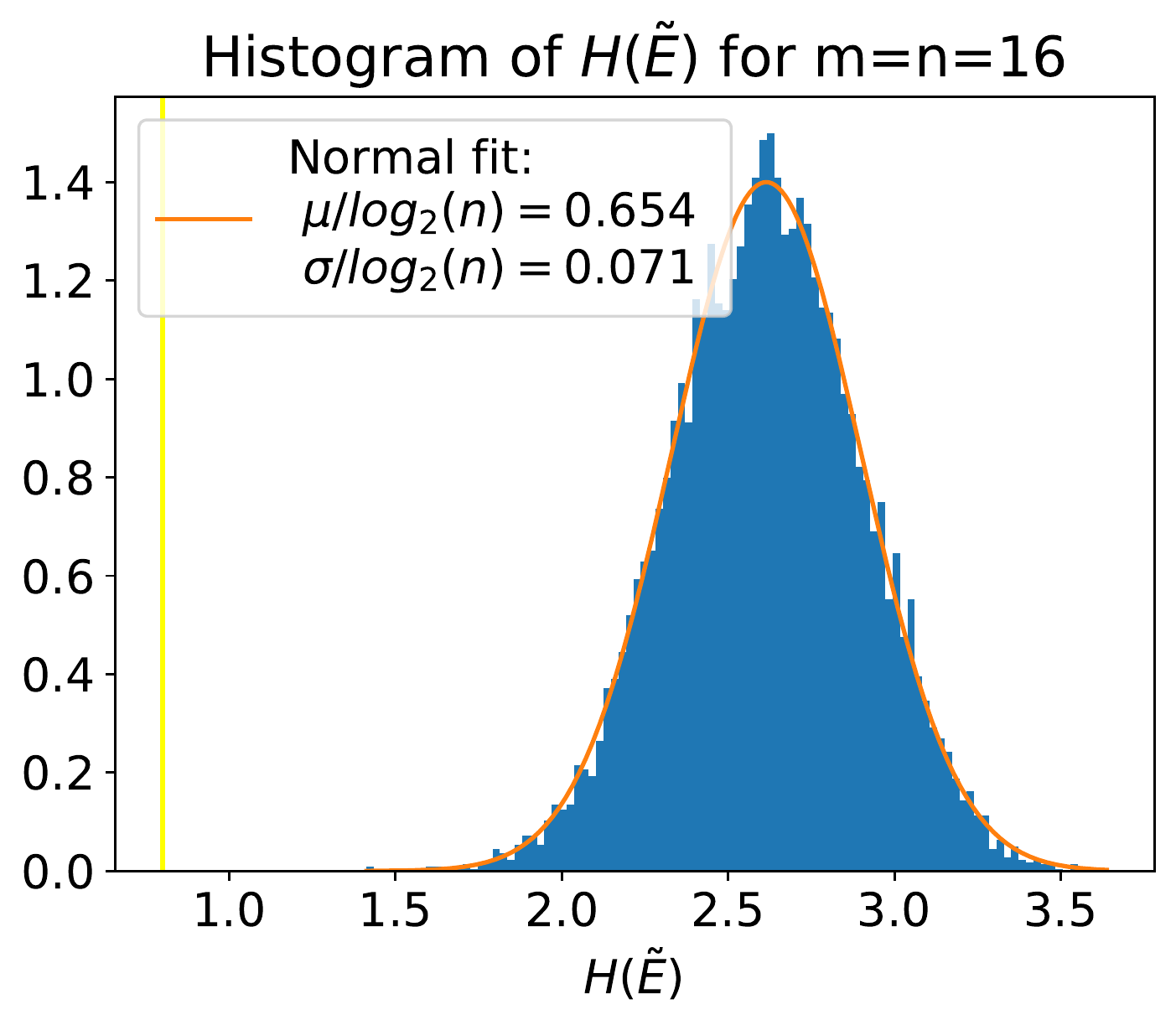}
		\caption{}
		\label{fig:normal_fit_16_0.2}
	\end{subfigure}
	\begin{subfigure}[b]{0.3\linewidth}
		\includegraphics[width=\textwidth]{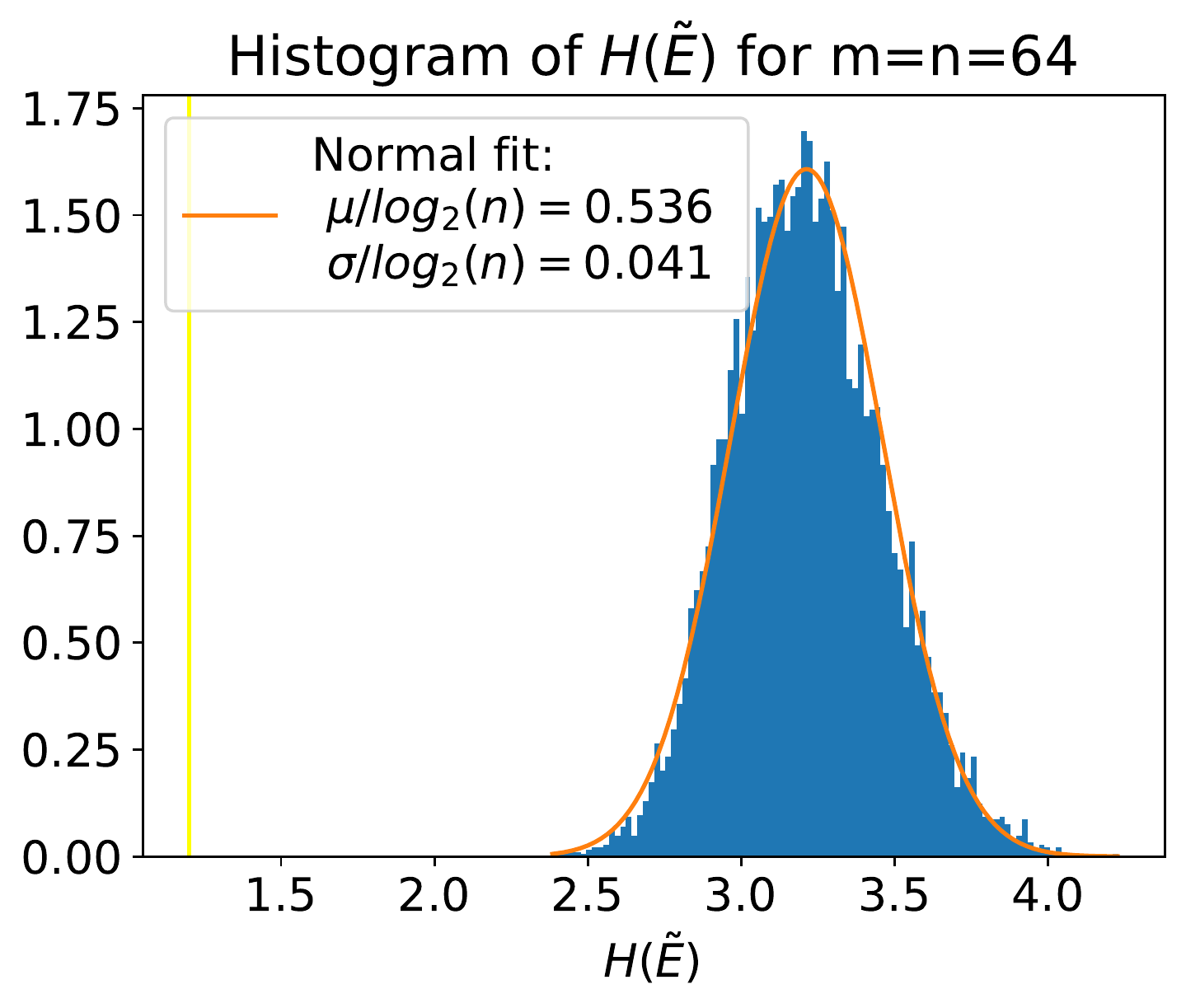}
		\caption{}
		\label{fig:normal_fit_64_0.2}
	\end{subfigure}
	\begin{subfigure}[b]{0.32\linewidth}
		\includegraphics[width=\textwidth]{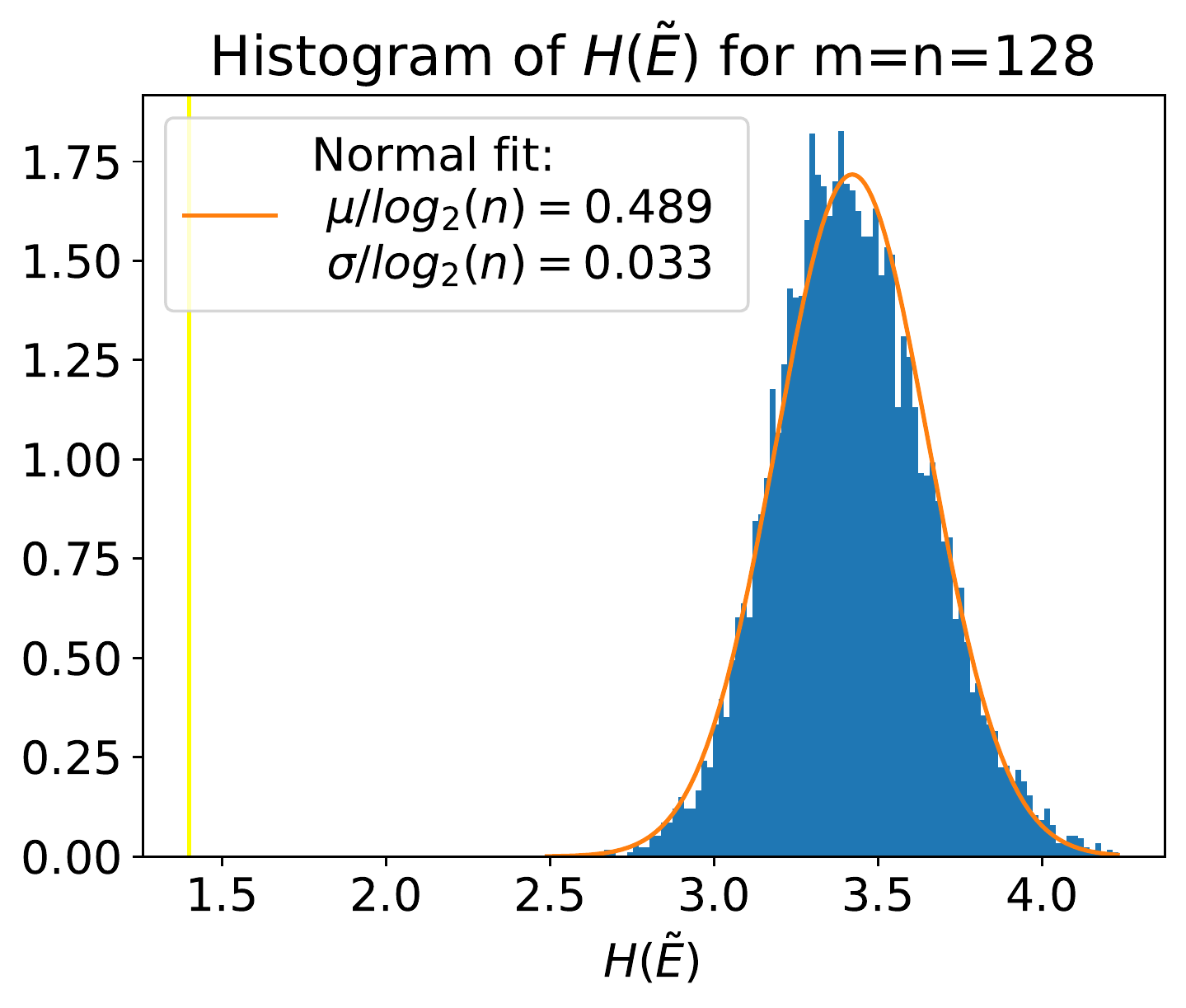}
		\caption{}
		\label{fig:normal_fit_128_0.2}
	\end{subfigure}	
\caption{Histogram of $H(\tilde{E})$ when $H(E)\approx 0.2\log_2(n)$.  Yellow line shows $x=0.2\log_2(n)$}
\label{fig:relaxing0.2}
\end{figure}

\subsection{Additional results on the finite sample regime}
\label{app:results}

Figure \ref{fig:extras} shows results on finite sample identifiability for the setting considered in the figure in the main text, except with smaller $H(E) \leq \ln(4)$. 
\begin{figure*}[t!]
	\vskip -0.11in
	\centering
	\begin{subfigure}[b]{0.3\linewidth}
		\includegraphics[width=\textwidth]{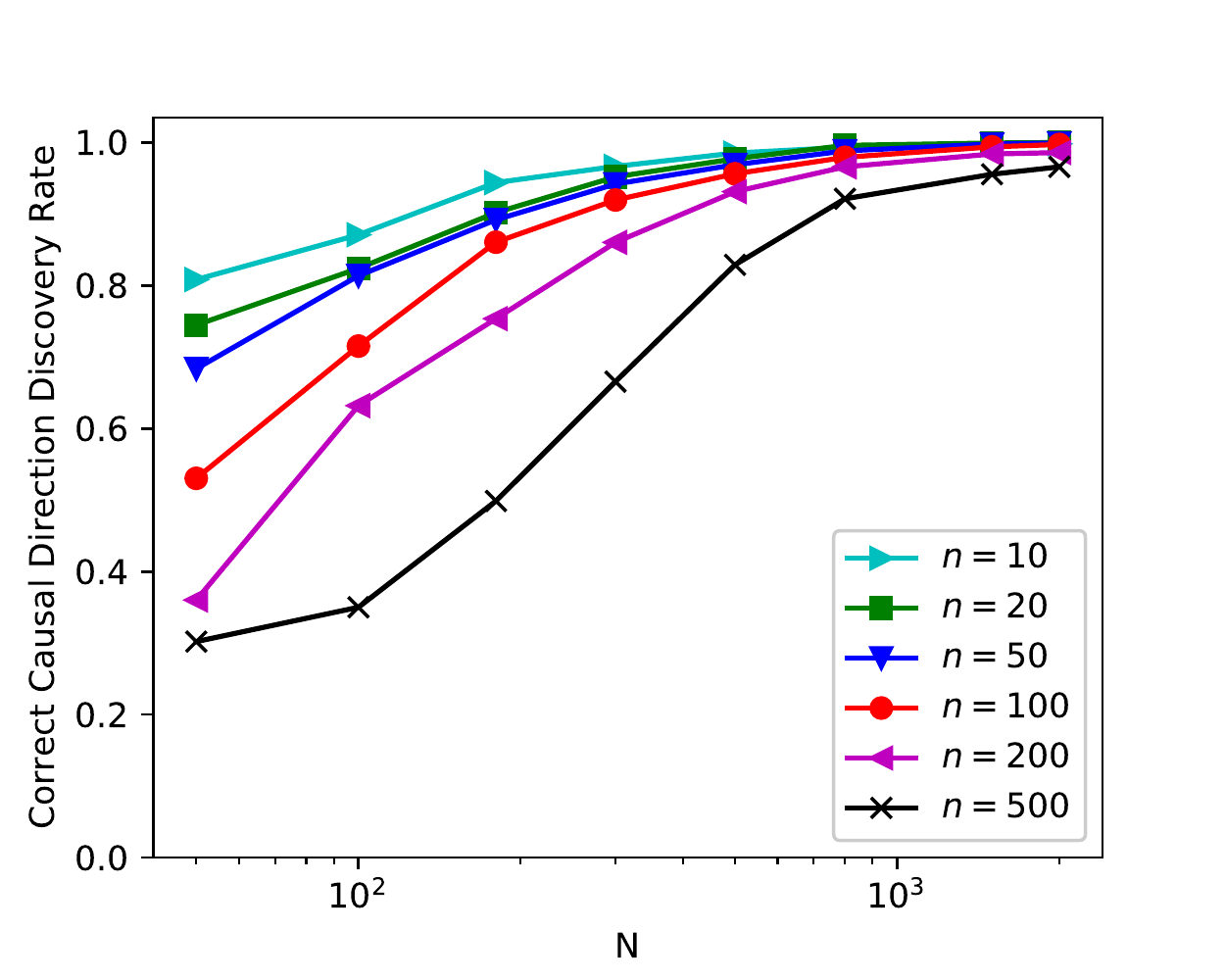}
		\vspace{-7mm}
		\caption{Identification via conditional entropies ($H(E) \leq \ln(4)$).}
		\label{fig:conditional1}
	\end{subfigure}
	\begin{subfigure}[b]{0.3\linewidth}
		\includegraphics[width=\textwidth]{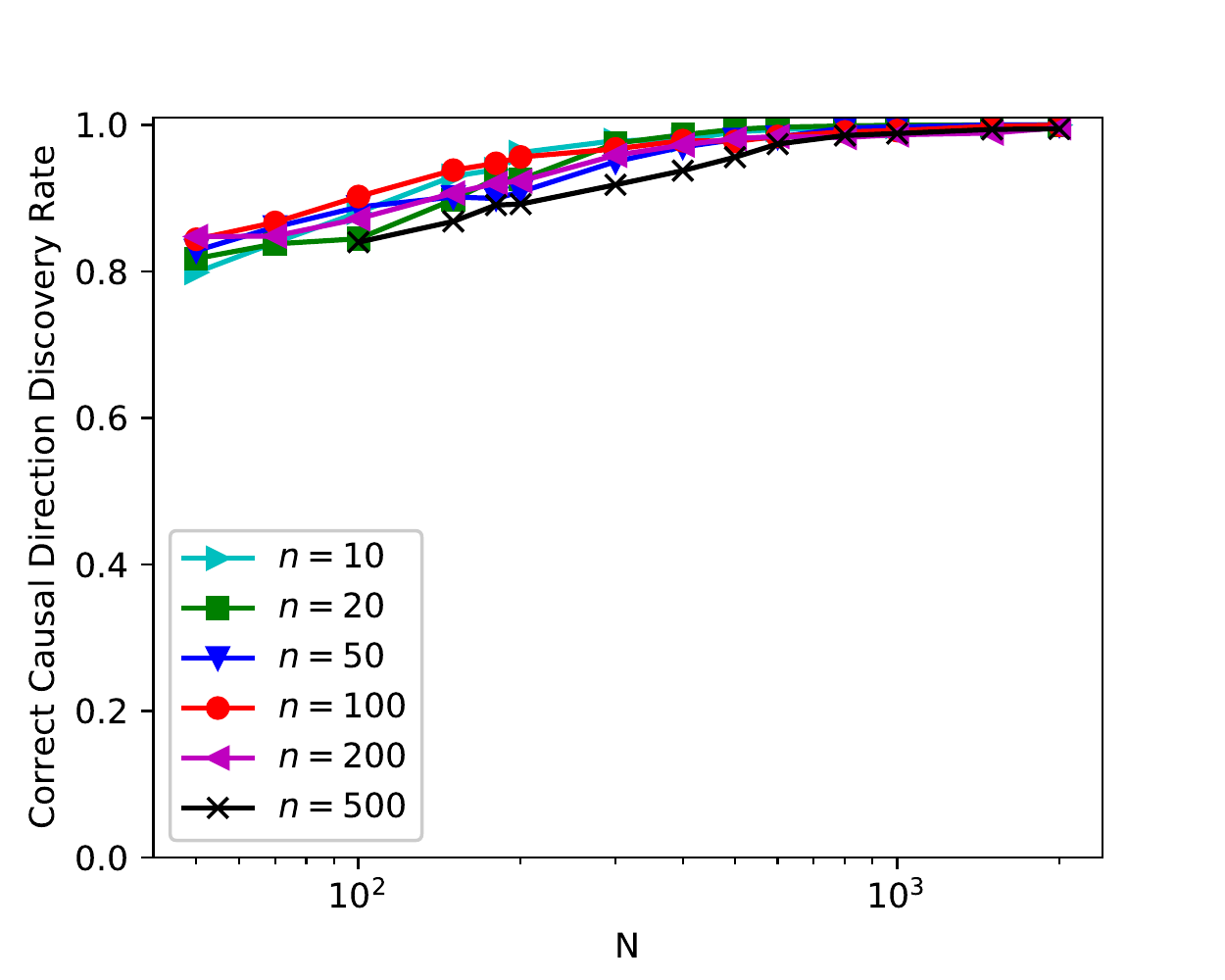}
		\vspace{-7mm}
		\caption{Identification via MEC algorithm ($H(E) \leq \ln(4)$).}
		\label{fig:conditionalMEC1}
	\end{subfigure}
	\begin{subfigure}[b]{0.28\linewidth}
		\includegraphics[width=\textwidth]{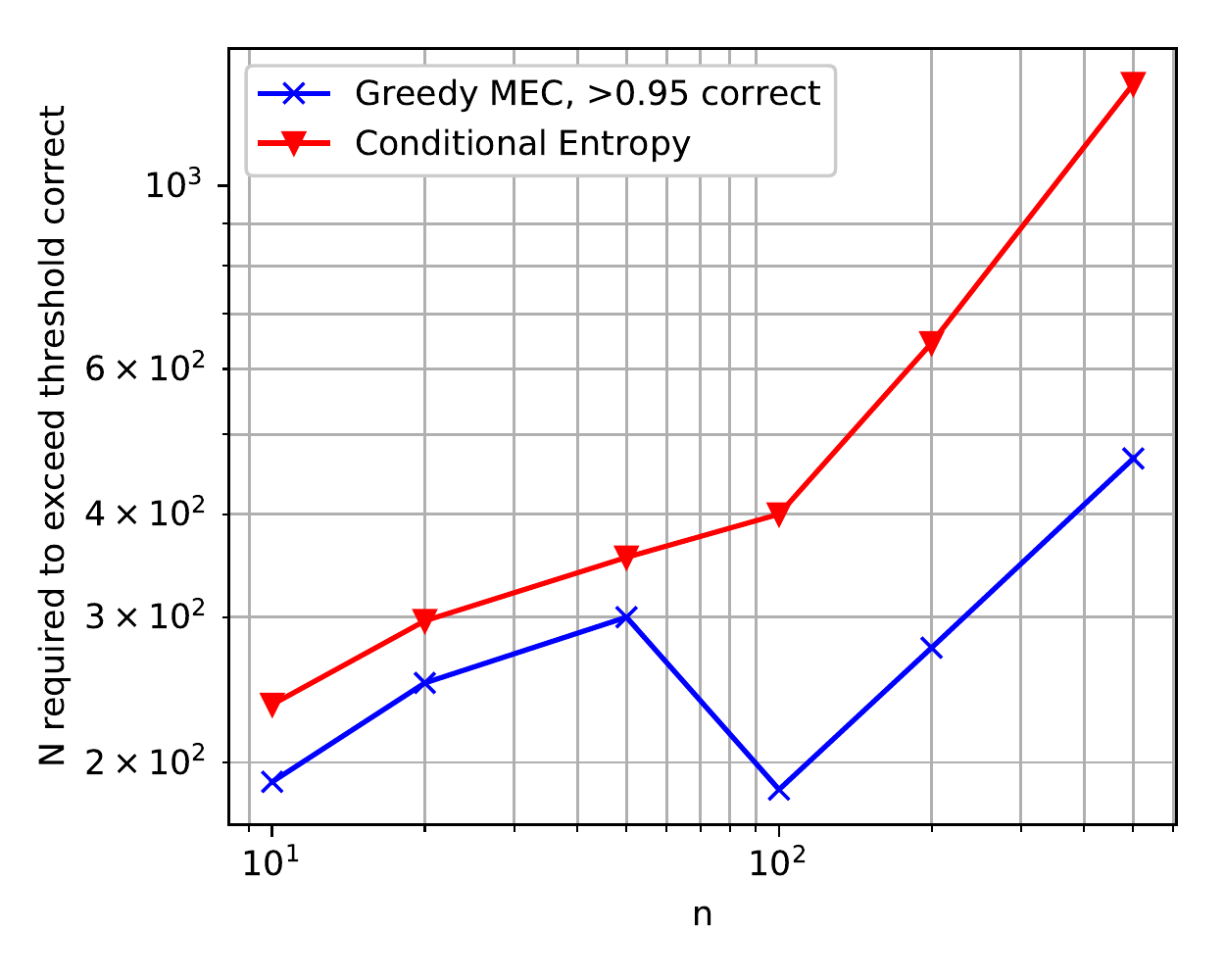}
		\vspace{-7mm}
		\caption{Number of samples vs. support size of observed variables.}
		\label{fig:Nreqd1}
 	\end{subfigure}
	\caption{Finite sample identifiability of the causal direction via entropic causality. (a) Probability of correctly discovering the causal direction $X \rightarrow Y$ as a function of $n$ and number of samples $N$, using the conditional entropies as the test. (b) Probability of correctly discovering the causal direction $X \rightarrow Y$ using the greedy MEC algorithm to test the direction. (c) Samples $N$ required to reach 95\% correct detection as a function of $n$, derived from the plots in Figure \ref{fig:conditional1} and Figure \ref{fig:conditionalMEC1}.  }
	\label{fig:extras}
\end{figure*}

Results for $p(X)$ drawn from $\mathrm{Dir}(1)$ are shown Figure \ref{fig:Dir1}, as described in the main text. We find that the greedy MEC performance degrades to a level that is similar to the conditional entropy criterion. This might be explained by the fact that if $p(X|Y=y)$ are close to uniform, then the gap between $H(\tilde{E})$ and $H(X|Y=y)$ vanishes. 

\begin{figure*}[t!]
	\vskip 0.2in
	\centering
	\begin{subfigure}[b]{0.32\linewidth}
		\includegraphics[width=\textwidth]{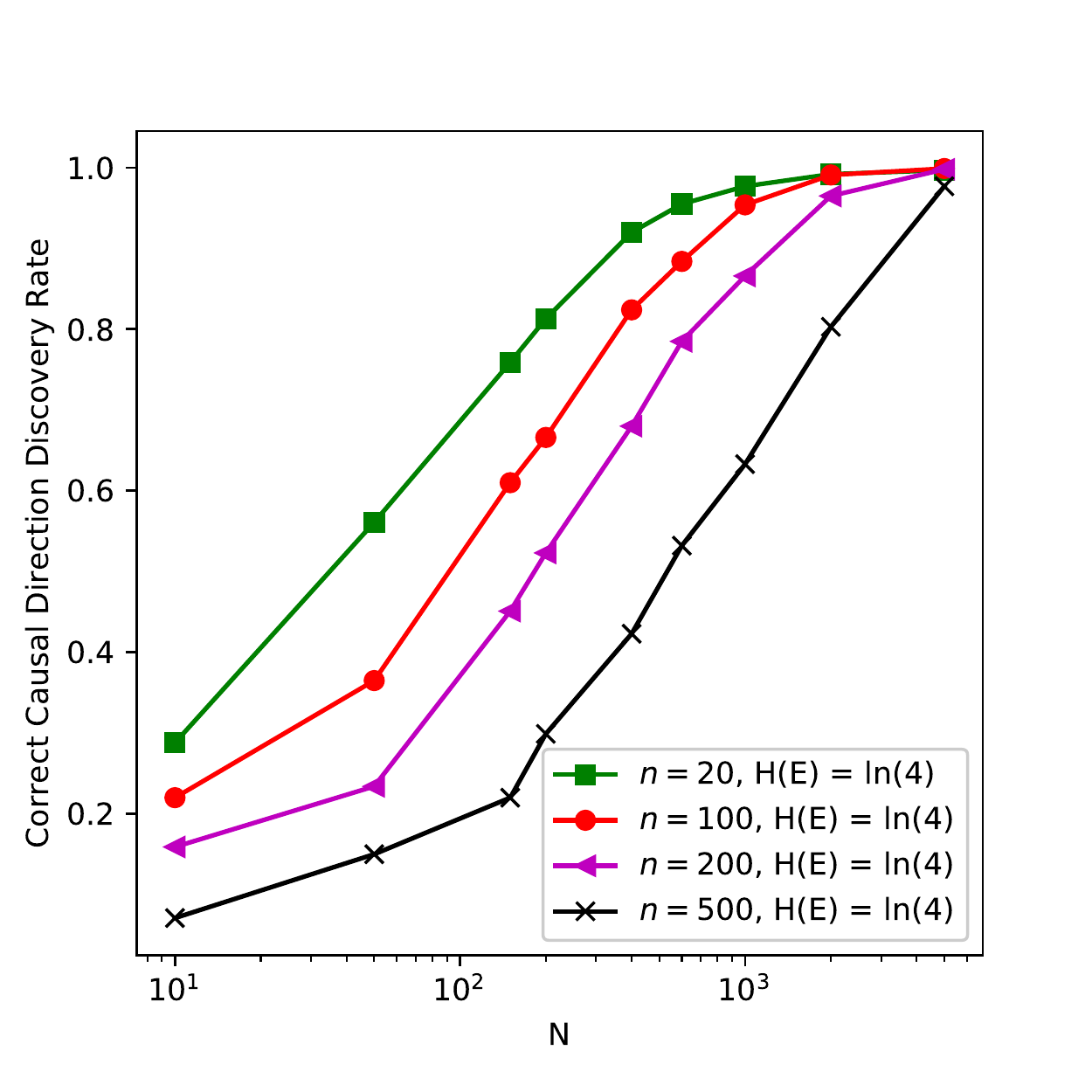}
		\caption{Identification via conditional entropies ($H(E) = \ln(4)$).}
		\label{fig:conditionalDir1}
	\end{subfigure}
	\begin{subfigure}[b]{0.32\linewidth}
		\includegraphics[width=\textwidth]{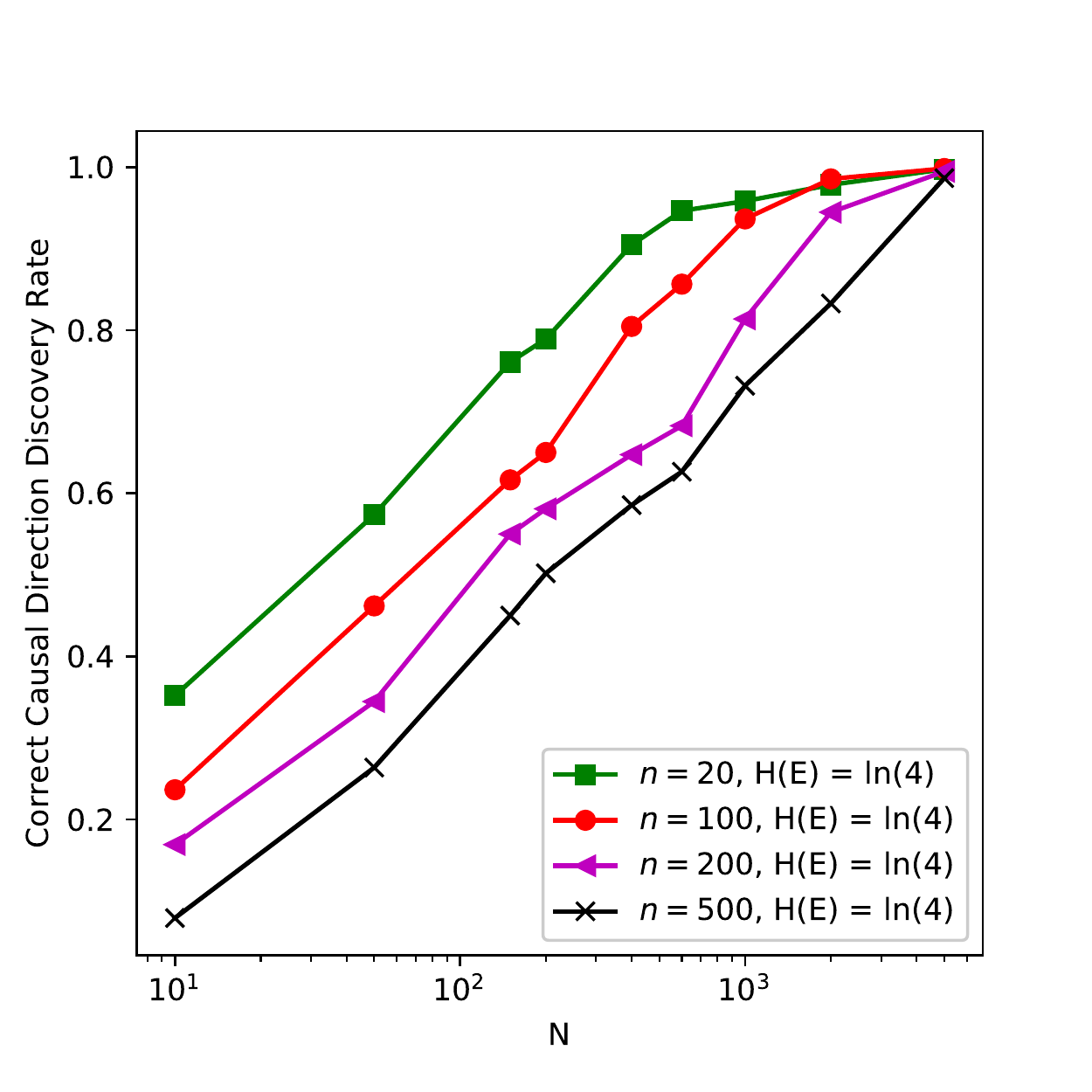}
		\caption{Identification via MEC algorithm ($H(E) = \ln(4)$).}
		\label{fig:conditionalMECDir1}
	\end{subfigure}
	\begin{subfigure}[b]{0.30\linewidth}
		\includegraphics[width=\textwidth]{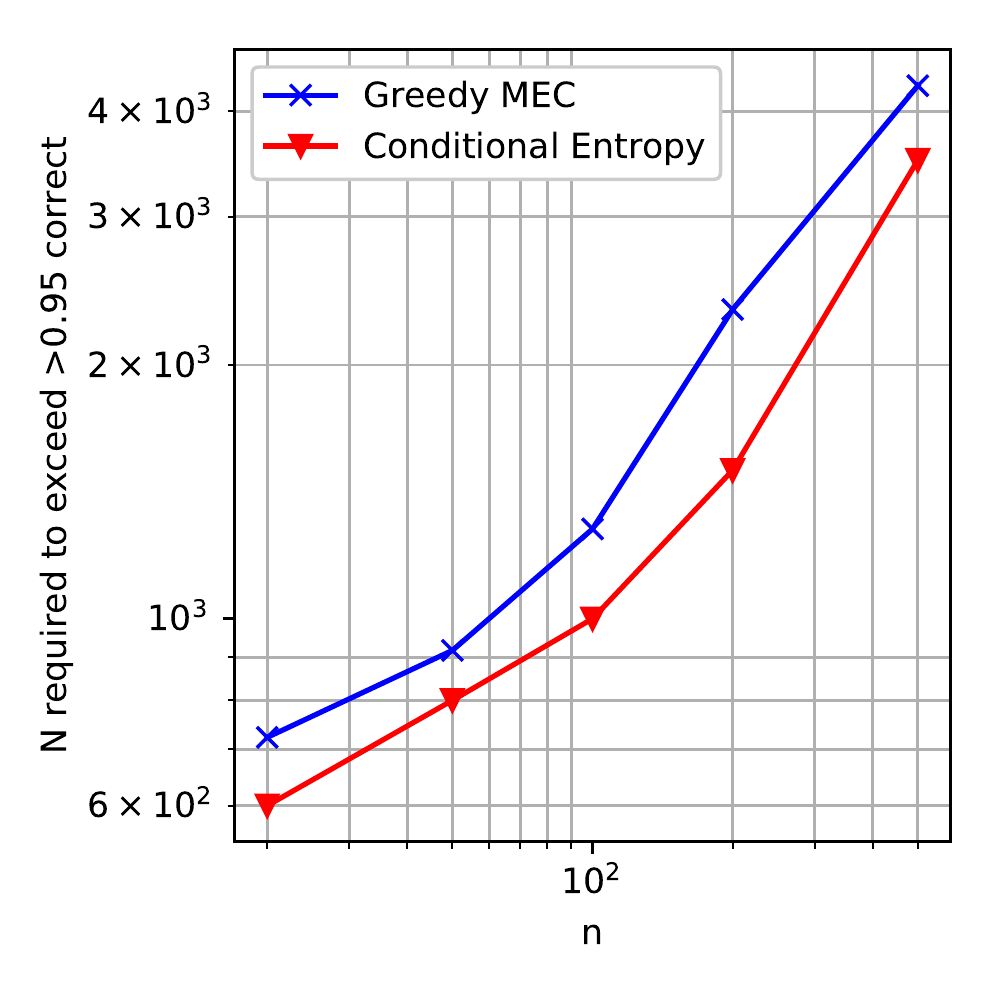}
		\caption{Number of samples vs. support size of observed variables ($H(E) = \ln(4)$).}
		\label{fig:NreqdDir1}
	\end{subfigure}
	\caption{Finite sample identifiability of the causal direction via entropic causality, where $p(x) \sim \mathrm{Dir}(1)$ (uniform on the simplex). (a) Probability of correctly discovering the causal direction $X \rightarrow Y$ as a function of $n$ and number of samples $N$, using the conditional entropies as the test. (b) Probability of correctly discovering the causal direction $X \rightarrow Y$ as a function of $n$ and number of samples $N$, using the greedy MEC algorithm to test the direction. (c) Samples $N$ required to reach 98\% correct detection as a function of $n$, derived from the plots in Figure \ref{fig:conditionalDir1} and Figure \ref{fig:conditionalMECDir1}.}
	\label{fig:Dir1}
\end{figure*}

\subsection{Additional Tuebingen Experiments}
In this section, we perform additional experiments to evaluate the stability of the method to choice of quantization on the Tuebingen dataset. Specifically, to quantize $[a,b]$ into $n$ intervals, we perturb each quantization point $\{a+\frac{(b-a)i}{n}\}_i$ with a uniform noise in $[-\frac{(b-a)}{8n},\frac{(b-a)}{8n}]$. For every pair, this is done $5$ times independently and the majority decision is taken. The results, which show similar performance to Table \ref{tab:tuebingen} are shown in Table \ref{tab:tuebingenQ}, demonstrating a degree of stability to choice of quantization. We observe that perturbed quantization demonstrates better performance for $20-$state quantization, whereas it shows somewhat worse performance for the $5$ and $10-$state quantizations. This indicates that more research is needed to determine the optimal quantization for a given dataset. 
\input{tuebingenQuantile}

%% file: table1.tex
\begin{table}[h]
\parbox{.6\linewidth}{
\centering
\begin{tabular}{| c || c | c | c | c | c | c | }
\hline
            & $\mathcal{E}$ & 1    & 2    & 3    & 4    & 5         \\\hline\hline
$\mathcal{X}$ & \backslashbox{PMF \\of $X$}{PMF \\of $E$}        & $e_1$ & $e_2$ & $e_3$ & $e_4$ & $e_5$  \\\hline
1           & $x_1$        & 2 & 3    & 2    &{\cellcolor{cyan!15} 1 }    & {\cellcolor{cyan!15} 1 }       \\\hline
2           & $x_2$        & 3    & 2   & 3    & 3    & {\cellcolor{cyan!15} 1 }      \\\hline
3           & $x_3$        & 3    & {\cellcolor{cyan!15} 1 }    & 2    & 3    & 2      \\\hline
\end{tabular}}\quad
\parbox{.4\linewidth}{
\centering
{\setlength{\extrarowheight}{7pt}
\begin{tabular}{| c | c | }
\hline
& $\mathbb{P}(X=x|Y=1)$\\\hline
$x=1$ & {\Large $\frac{x_1(e_4+e_5)}{Z}$}\\\hline
$x=2$ & {\Large$\frac{x_2e_5}{Z}$}\\\hline
$x = 3$ & {\Large$\frac{x_3e_2}{Z}$}\\\hline
\end{tabular}%
}}
\caption{Left: Balls and bins representation of function $f:\mathcal{X}\times \mathcal{E}\rightarrow \mathcal{Y}$, where $\mathcal{X}=\mathcal{Y} = [3]$ and $\mathcal{E}=[5]$. The function values for a given $X=i, E=k$ can be seen as realizations of a two dimensional balls and bins game. Right: Conditional probability values of $X$ given $Y=1$ for the given function. $Z = x_1(e_1+e_3) + x_2(e_2) + x_3(e_5) $ is the normalization constant, which also gives $\mathbb{P}(Y=1)$.}\vspace{-3mm}
\label{table:ballsandbins}
\end{table}

%% file: tuebingenQuantile.tex
\begin{table}[t!]
\footnotesize
\centering
\scalebox{0.85}
{
\parbox{.25\linewidth}{5-state quantization}\:%
\parbox{.75\linewidth}{
\begin{tabular}{|c|c|c|c|c|c|c|c|}
\hline
Threshold ($\times \log$ support) &0.6 & 0.7 & 0.8 & 0.85 & 0.9 & 1.0 & 1.2 \\\hline
\# of pairs & 10 & 13 & 32 & 42 & 53 & 69 & 85 \\\hline
Accuracy (\%) & 90.0 & 61.5 & 53.1 & 54.8 & 56.5 & 58.5 & 57.6 \\\hline 

\end{tabular}
}
}
\scalebox{0.85}
{
\parbox{.25\linewidth}{{10-state quantization}}\:%
\parbox{.75\linewidth}{
\begin{tabular}{|c|c|c|c|c|c|c|c|}
\hline
Threshold ($\times \log$ support) & 0.6 & 0.7 & 0.8 & 0.85 & 0.9 & 1.0 & 1.2 \\\hline
\# of pairs & 8 &12 & 23 & 39 & 49 & 71 & 85 \\\hline
Accuracy (\%) & 87.5 & 66.7& 60.9 & 53.8 & 51.0 & 52.1 & 57.6\\\hline 

\end{tabular}
}
}
\scalebox{0.85}
{
\parbox{.25\linewidth}{20-state quantization}\:%
\parbox{.75\linewidth}{
\begin{tabular}{|c|c|c|c|c|c|c|c|}
\hline
Threshold ($\times \log$ support) &0.6 & 0.7 & 0.8 & 0.85 & 0.9 & 1.0 & 1.2 \\\hline
\# of pairs & 5 & 10 & 15 & 31 & 54 & 78 & 85 \\\hline
Accuracy (\%) & 60.0 & 70.0 & 73.3 & 54.8 & 48.1 & 48.7 & 55.3\\\hline 

\end{tabular}
}
}

\caption{{Performance on T{\"u}bingen causal pairs with low exogenous entropy in at least one direction. Chosen based on majority voting on 5 random quantizations. }}
\label{tab:tuebingenQ}
\vspace{-5mm}
\end{table}